\documentclass{article}

\PassOptionsToPackage{numbers, compress}{natbib}
\usepackage[main, final]{neurips_2026}
\usepackage[T1]{fontenc}
\usepackage[utf8]{inputenc}
\usepackage{microtype}
\usepackage{hyperref}
\usepackage{url}
\usepackage{booktabs}
\usepackage{amsmath,amssymb,amsfonts,amsthm,mathtools,bm}
\usepackage{graphicx}
\usepackage{subcaption}
\usepackage{multirow}
\usepackage{enumitem}
\usepackage{xcolor}
\usepackage{algorithm}
\usepackage{algpseudocode}
\usepackage{multirow}
\usepackage{threeparttable}
\usepackage{subcaption}

\usepackage{float}
\usepackage{tikz}
\usepackage{placeins}
\usetikzlibrary{arrows.meta,positioning,fit}

\hypersetup{
  colorlinks=true,
  breaklinks=true,
  linkcolor=blue,
  citecolor=blue,
  urlcolor=blue
}

\newtheorem{theorem}{Theorem}
\newtheorem{proposition}{Proposition}
\newtheorem{corollary}{Corollary}
\newtheorem{lemma}{Lemma}
\newtheorem{remark}{Remark}
\newtheorem{assumption}{Assumption}

\newcommand{\E}{\mathbb{E}}

\newcommand{\KL}{\mathrm{KL}}

\newcommand{\argmin}{\mathop{\mathrm{arg\,min}}}

\usepackage[normalem]{ulem}
\newif\ifshowcrdeleted
\showcrdeletedfalse
\newcommand{\crnew}[1]{{#1}}
\newcommand{\crdel}[1]{\ifshowcrdeleted{\color{red}\sout{#1}}\fi}

\newenvironment{crblock}{\begingroup}{\endgroup}
\newif\ifcrmarkup
\crmarkuptrue
\ifcrmarkup\else\definecolor{red}{RGB}{0,0,0}\fi

\newif\ifcrGM
\crGMtrue

\usepackage{xcolor}

\title{Online Bayesian Calibration under Gradual and Abrupt System Changes}

\author{
\crnew{Yang Xu \qquad Chiwoo Park} \\
\crnew{Department of Industrial and Systems Engineering} \\
\crnew{University of Washington, Seattle, WA 98195} \\
\crnew{\texttt{yxu59@uw.edu \qquad chiwpark@uw.edu}}
}
\date{}

\begin{document}
\maketitle

\begin{abstract}
Bayesian model calibration is central to digital twins and computer experiments, as it aligns model outputs with field observations by estimating calibration parameters and correcting systematic model bias. Classical Bayesian calibration introduces latent parameters and a discrepancy function to capture bias, but suffers from parameter–discrepancy confounding and is typically formulated as an offline procedure under a stationary data-generating assumption. These limitations are restrictive in modern digital twin applications, where systems evolve over time and may exhibit both gradual drift and abrupt regime shifts. While data assimilation methods enable sequential updates, they generally do not explicitly model systematic bias and are not robust to abrupt changes. We propose Bayesian Recursive Projected Calibration (BRPC), an online Bayesian calibration framework for streaming data under simulator mismatch and nonstationarity. BRPC extends projected calibration to the online setting through a two-stage filtering procedure: first, unconditional filtering of the calibration parameters, followed by conditional filtering of the model discrepancy function given the estimated calibration parameters, \crnew{which mitigates parameter--discrepancy confounding}, while the online filtering enables bias-aware adaptation under gradual system evolution. To handle abrupt changes, we integrate BRPC with restart mechanisms that detect regime shifts and reset the calibration process, yielding a unified approach to mixed nonstationarity. We \crdel{establish theoretical guarantees}\crnew{provide theoretical analysis}, including \crdel{tracking performance under gradual evolution}\crnew{a tracking error bound of the model discrepancy estimate under gradual evolution, given accurate calibration parameter estimates} and false-alarm and detection behavior of restart mechanisms. Empirical results on synthetic and plant-simulation benchmarks demonstrate that BRPC improves calibration accuracy under gradual changes, while the restart-augmented framework further enhances robustness and predictive performance under abrupt regime shifts, compared to the state-of-the-art approaches.
\end{abstract}

\section{Introduction}
Bayesian calibration is a central problem in digital twins and computer model experiments. Computer simulators provide approximations to physical systems, typically involving unknown or artificial parameters that govern model approximations. These approximations are inherently imperfect, often resulting in systematic discrepancies between model outputs and the true system, even after the unknown parameters are optimized toward minimal discrepancies. Consequently, models must be calibrated against data to estimate unknown parameters and correct model discrepancy. This joint estimation of model parameters and discrepancy is referred to as the model calibration problem.

The classical Bayesian calibration framework relates computer model outputs to field observations by introducing latent model parameters and a model discrepancy term \citep{kennedy2001bayesian}. However, this formulation suffers from parameter–discrepancy confounding, leading to identifiability issues. Projected calibration addresses this limitation by defining a calibration target prior to learning discrepancy, thereby improving identifiability \citep{tuo2015efficient,xie2021bayesian}. Despite this, these approaches are inherently offline and assume a stationary data-generating process, where observations arise from a fixed latent system—an assumption often violated in modern digital twin applications. In such settings, data arrive sequentially and system behavior may evolve over time, with the calibration target drifting or shifting across regimes \citep{ward2021continuous,mesbah2016stochastic}. Developing online Bayesian calibration methods that preserve the identifiability advantages of projected calibration remains an important and relatively underexplored problem.

A straightforward approach to address this issue is sliding-window Bayesian calibration, which periodically applies offline calibration over a fixed data window. However, frequent updates are computationally expensive, while longer intervals can introduce substantial bias under rapid system changes. Filtering and data assimilation methods provide natural online updates \citep{evensen2003ensemble,evensen2009ensemble,ward2021continuous}, but they do not address parameter–discrepancy identifiability or explicitly model systematic simulator bias, and are less effective under abrupt regime shifts. Change-point and concept-drift methods detect nonstationarity \citep{page1954continuous,adams2007bayesian,gama2014survey}, but do not specify how calibration should be performed.

Toward this goal, we propose Bayesian Recursive Projected Calibration (BRPC), an online calibration framework for nonstationary data streams. BRPC relates simulator outputs to streaming observations through time-varying model parameters and a discrepancy term that captures systematic bias (Sec. \ref{sec:problem_setup} and \ref{sec:brpc}). It adopts a two-stage update strategy: first updating model parameters to align the simulator with evolving system dynamics, and then updating a Gaussian process–based discrepancy model conditional on the parameters. The two-stage estimation improves parameter–discrepancy identifiability by mitigating confounding between the two components.

BRPC is integrated with a restart mechanism that detects abrupt changes in data streams and reinitializes the online calibration process by resetting previously accumulated filtering states and minimizing the increase of biases due to before-change data. The restart decision is based on the predictive evidence from the online calibration process, particularly posterior update of the model discrepancy term. Consequently, the quality of restart decisions depends not only on the restart mechanism, but also on the model discrepancy learner. A highly concentrated discrepancy posterior may cause routine within-regime variation to appear anomalous, while an overly adaptive discrepancy model may absorb emerging regime mismatches before the detector can identify a clear change. \crdel{Studying the interaction between online discrepancy learning and the restart mechanism, we propose two complementary deployment strategies. We support the proposed framework with theoretical analysis of the recursive updates and the resulting predictive evidence used for control (Sec. \ref{sec:reset_control}), and evaluate it on both synthetic data and a plant simulation digital twin benchmark (Sec. \ref{sec:experiments}).}

\begin{crblock}
\noindent\textbf{Contributions.}
Our contributions are threefold.
\textbf{(i) Identifiability-aware online Bayesian calibration.}
We adopts Bayesian projected calibration idea to address the identifiability issue and extend the idea to streaming observations through a two-stage filtering procedure that combines a recursive, discrepancy-free particle update for calibration parameters with conditional Gaussian process learning of the model discrepancy. 
\textbf{(ii) Calibration-specific change detection and adaptive restarting.}
Unlike conventional change-point detection, our restart mechanisms identify when the current calibrated model becomes inconsistent with the evolving physical system. We integrate BOCPD- and wCUSUM-based restart controllers with BRPC and characterize how discrepancy adaptation influences the distinction between gradual drift and abrupt regime shifts.
\textbf{(iii) A unified framework for nonstationary calibration.}
We jointly address continuous adaptation and abrupt changes while maintaining projected calibration and uncertainty quantification. We provide theoretical analyses of discrepancy tracking error bounds and restart mechanism accuracy, along with extensive numerical validation.
\end{crblock}

\section{Related Work}

\noindent\textbf{Offline Bayesian calibration.}
Bayesian calibration with model discrepancy is commonly formulated by augmenting a simulator with a flexible bias term, often modeled as a Gaussian process discrepancy or its variants \citep{kennedy2001bayesian,rasmussen2006gaussian,gu2018scaled,xu2025deep}. While this formulation improves predictive accuracy, it introduces parameter–discrepancy confounding, whereby calibration parameters and discrepancy can jointly explain the same simulator–reality mismatch. Projected calibration and Bayesian projected calibration address this issue by defining the calibration target prior to learning the discrepancy \citep{tuo2015efficient,xie2021bayesian}. BRPC extends this principle to the online setting: the discrepancy is excluded from the parameter update, and is instead learned conditionally after anchoring the simulator through the calibrated parameters.
\crnew{
An alternative approach to parameter--discrepancy confounding is proposed by \citet{plumlee2017bayesian}, who imposes orthogonality constraints on the discrepancy prior to improve calibration identifiability. In contrast, BRPC builds on Bayesian projected calibration, which defines an explicit calibration target through simulator alignment, and extends this principle to online, nonstationary settings.
}

\noindent\textbf{Sequential calibration, filtering, and data assimilation.}
Sequential Monte Carlo, ensemble Kalman filtering, and related data-assimilation methods provide natural tools for online parameter tracking \citep{evensen2003ensemble,evensen2009ensemble,van2019particle,ruckstuhl2018parameter,kotsuki2018online}. In digital twin calibration, continuous calibration has been studied by combining particle-filter updates with sliding-window Bayesian calibration \citep{ward2021continuous}. These methods are effective under smooth system evolution, but they are ineffective under abrupt changes. Moreover, their updates are typically joint, with parameter adjustments and model mismatch absorbed through the same assimilation process. In contrast, BRPC separates a projected parameter update from a conditional discrepancy update, making their interaction explicit and more amenable to analysis.

\noindent\textbf{Change-point detection and restart control.}
Change-point methods provide principled tools for handling abrupt nonstationarity. CUSUM detects changes through accumulated score or likelihood-ratio statistics \citep{page1954continuous,xie2023window}, while Bayesian online change-point detection (BOCPD) maintains posterior uncertainty over run lengths and yields a mixture over candidate segment hypotheses \citep{adams2007bayesian,saatcci2010gaussian}. Extensions such as restarted BOCPD convert change-point evidence into explicit state-reset mechanisms \citep{alami2020restarted}. In online calibration, these detectors must be coupled with an adaptive within-regime learner, and the evidence they monitor is generated by the current calibration state. BRPC studies this coupling in the presence of simulator discrepancy and projected recursive learning, where restart decisions and discrepancy adaptation interact in nontrivial ways.

BRPC also connects to online learning, concept drift, and continual learning
through KL-regularized recursive updates and restart-based state reset;
we discuss these connections in Appendix~\ref{app:online_continual_learning}.

\section{Problem Setup}
\label{sec:problem_setup}
We start by describing a general calibration problem under nonstationary systems. For a physical system of interest, let $x$ denote a vector of input variables and $y$ the corresponding scalar system output. The input--output relationship is governed by an unknown function $y = \zeta_t(x)$, which may vary over time $t$. A computer model provides a structured approximation to this system. The simulator output, denoted by $y_s(x, \theta)$, depends on the input $x$ and latent model parameters $\theta$, and is generally subject to systematic bias relative to the physical system. Following \citet{kennedy2001bayesian}, we model the physical response as $\zeta_t(x) = y_s(x, \theta_t) + \delta_t(x),$
where $\theta_t$ represents time-varying calibration parameters and $\delta_t(x)$ is a time-varying discrepancy function that captures model bias. 

We formulate the online Bayesian calibration problem as estimating the time-varying calibration parameters $\theta_t$ and discrepancy function $\delta_t(x)$ from streaming observations of the physical system. At each time $t$, we observe a batch of data $B_t=\{(x_{t,k},y_{t,k})\}_{k=1}^{K_t}$, where $K_t$ denotes the number of observations available at time $t$. Each observation is assumed to be a noisy realization of the underlying system response $\zeta_t$, given by $
y_{t,k} = \zeta_t(x_{t,k}) + \epsilon_{t,k}
= y_s(x_{t,k},\theta_t) + \delta_t(x_{t,k}) + \epsilon_{t,k}$, 
where $\epsilon_{t,k} \sim \mathcal{N}(0,\sigma_t^2)$ represents independent observation noise. For notational simplicity, we denote by $X_t = \{x_{t,k}\}_{k=1}^{K_t}$ and $Y_t = \{y_{t,k}\}_{k=1}^{K_t}$ the collections of inputs and responses at time $t$, respectively.

For generality, we consider that the system may evolve under different modes of nonstationarity. We use the term \emph{drift} to denote gradual system evolution, under which streaming data from consecutive time points remain sufficiently similar so that information accumulated from recent past batches can help stabilize the estimation of $\theta_t$ and $\delta_t$. In contrast, we use the term \emph{regime change} to denote abrupt or structural shifts in the system, such as changes in operating conditions, demand patterns, underlying physical mechanisms, or the simulator--reality relationship. Following a regime change, information from the previous regime may become biased or misleading for the current regime, even if it was informative prior to the shift. A more formal distinction between drift and regime change is provided in Appendix~\ref{app:drift_regime_definitions}. We also consider a mixed mode of drifts and regime changes. 

\section{Bayesian Recursive Projected Calibration}
\label{sec:brpc}

Following projected calibration \cite{tuo2015efficient}, we define the calibration parameter $\theta_t$ as the minimizer of the $L_2$ discrepancy between the system response $\zeta_t(x)$ and the simulator output $y_s(x,\theta)$:
\begin{equation}
\theta_t = \Pi(\zeta_t)
\in
\argmin_{\theta \in \Theta}
\int_{\Omega}
\{\zeta_t(x) - y_s(x,\theta)\}^2 \, dF_X(x),
\label{eq:projection_operator}
\end{equation}
where $F_X$ denotes the design distribution. This definition separates the roles of $\theta_t$ and $\delta_t$ by anchoring $\theta_t$ through the simulator fit, so that the discrepancy term $\delta_t(x) = \zeta_t(x) - y_s(x,\theta_t)$ is defined conditionally on $\theta_t$, thereby mitigating parameter--discrepancy confounding. 

As the system response $\zeta_t(x)$ evolves over time, both $\theta_t$ and $\delta_t(x)$ become time-varying with unknown dynamics. Online Bayesian tracking of them requires sequentially updating their posterior distributions as new data batches $B_t$ arrive. Guided by the definition \eqref{eq:projection_operator}, we first update the posterior distribution of $\theta_t$ (Section~\ref{subsec:brpc_theta_update}), and then update the posterior of $\delta_t(x)$ conditionally on $\theta_t$ (Section~\ref{subsec:brpc_delta_update}).

\subsection{Online posterior update of $\theta_t$}
\label{subsec:brpc_theta_update}
Let $q_{t}^{\theta}$ denote the posterior distribution over the latent calibration parameter $\theta_{t}$ at time $t$. The initial posterior $q_{1}^{\theta}$ can be obtained via classical Bayesian projected calibration \cite{xie2021bayesian}. For $t \ge 2$, the posterior $q_t^{\theta}$ evolves from $q_{t-1}^{\theta}$ under unknown dynamics. To model this evolution, we introduce a proposal transition distribution $p(\theta_t \mid \theta_{t-1})$, which serves to propagate prior uncertainty forward in time before incorporating new observations.

Specifically, prior to observing the new data batch $(X_t, Y_t)$, we form a predictive prior via
\begin{equation}
\bar q_{t\mid t-1}^{\theta}(\theta_t)
=
\int
p(\theta_t \mid \theta_{t-1})
\, q_{t-1}^{\theta}(\theta_{t-1})
\, d\theta_{t-1},
\label{eq:theta_predictive_prior}
\end{equation}
which serves as the prior approximation to $q_t^{\theta}$. Upon observing $(X_t, Y_t)$, we update this prior using a likelihood induced by the projected calibration criterion in \eqref{eq:projection_operator}. Specifically, we define a discrepancy-free likelihood
\begin{equation}
p_{\mathrm{proj}}(Y_t \mid X_t, \theta_t)
=
\mathcal{N}
\!\left(
Y_t;\,
y_s(X_t,\theta_t),\,
\Sigma_{\theta,t}
\right),
\label{eq:discrepancy_free_likelihood}
\end{equation}
where $\Sigma_{\theta,t}$ denotes the covariance matrix of the observation noise.

The posterior update is then defined via a KL-regularized optimization:
\begin{equation}
q_t^\theta
=
\argmin_q
\left\{
-\eta_\theta\,
\mathbb{E}_q
\!\left[
\log p_{\mathrm{proj}}(Y_t \mid X_t,\theta_t)
\right]
+
\mathrm{KL}
\!\left(
q(\theta_t)\,\|\,\bar q_{t\mid t-1}^{\theta}(\theta_t)
\right)
\right\},
\label{eq:online_theta_kl_update}
\end{equation}
where $\eta_{\theta} > 0$ is a weighting parameter.

\begin{proposition}[Projected parameter update]
\label{prop:projected_parameter_update}
The KL-regularized update \eqref{eq:online_theta_kl_update} admits the solution
\begin{equation}
q_t^\theta(\theta_t)
\propto
\bar q_{t\mid t-1}^{\theta}(\theta_t)
\,
p_{\mathrm{proj}}(Y_t \mid X_t,\theta_t)^{\eta_\theta}.
\label{eq:theta_tempered_posterior}
\end{equation}
If $q_{t-1}^{\theta}$ is represented by particles
$\{(\theta_{t-1}^{(i)}, w_{t-1}^{(i)})\}_{i=1}^N$, then
$\bar q_{t\mid t-1}^{\theta}$ is approximated by
$\{(\theta_{t\mid t-1}^{(i)}, w_{t\mid t-1}^{(i)})\}_{i=1}^N$ with
\[
\theta_{t\mid t-1}^{(i)} \sim p(\theta_t \mid \theta_{t-1}=\theta_{t-1}^{(i)}),
\qquad
w_{t\mid t-1}^{(i)} = w_{t-1}^{(i)}.
\]
The updated posterior \eqref{eq:theta_tempered_posterior} is then approximated by
$\{(\theta_t^{(i)}, w_t^{(i)})\}_{i=1}^N$, where
\begin{equation}
\theta_t^{(i)} = \theta_{t\mid t-1}^{(i)}, 
\quad
w_t^{(i)}
\propto
w_{t\mid t-1}^{(i)}\,
p_{\mathrm{proj}}
\!\left(Y_t \mid X_t,\theta_{t\mid t-1}^{(i)}\right)^{\eta_\theta},
\qquad
\sum_{i=1}^N w_t^{(i)} = 1.
\label{eq:projected_pf_weights}
\end{equation}
\end{proposition}

The corresponding particle implementation propagates particles through
$p(\theta_t \mid \theta_{t-1})$, reweights them using
$p_{\mathrm{proj}}(Y_t \mid X_t,\theta_t)^{\eta_\theta}$, and resamples when the effective sample size becomes small. This particle filter serves as an online analogue of the pushforward uncertainty in static BPC, maintaining a distribution over the evolving projected calibration target. When $\Sigma_{\theta,t} = \sigma_t^2 I$, the negative log-likelihood is proportional to the empirical squared projected loss on the current batch, preserving consistency with the projected calibration objective. Additional implementation details are provided in Appendix~\ref{app:brpc_theta_details}, and the algorithm table is in Algorithm \ref{alg:projected_particle_update}.

\subsection{Online conditional posterior update of $\delta_t(x)$}
\label{subsec:brpc_delta_update}

\begin{crblock}
After updating the projected calibration parameters, we
estimate the discrepancy conditionally on the parameter
posterior. Within each regime, we model the temporal
evolution of the discrepancy as a Gaussian--Markov process
\begin{equation}
\delta_t(\cdot)
=
\rho_\delta \delta_{t-1}(\cdot)
+
\xi_t(\cdot),
\qquad
\xi_t \overset{\mathrm{i.i.d.}}{\sim}
\mathcal{GP}(0,q_\delta k_\phi),
\label{eq:discrepancy_dynamics}
\end{equation}
where $k_\phi$ is the spatial covariance kernel,
$\rho_\delta\in(0,1]$ controls how much of the previous
discrepancy is retained, and $q_\delta\geq 0$ controls
how much it can change between batches.
Initially, $\delta_0\sim\mathcal{GP}(0,k_\phi)$.

We represent the discrepancy on a support set $S_t$,
which can consist of fixed inducing locations or expand
with incoming data, $S_t=\bigcup_{\tau=0}^t X_\tau$.
The evolution covariance on this support is
$Q_{\delta,t}=q_\delta K_\phi(S_t,S_t)$.
Two common choices are $\rho_\delta=1$ with $q_\delta>0$,
corresponding to a random walk, and
$q_\delta=1-\rho_\delta^2$, which preserves the
marginal covariance $k_\phi$ under stationary initialization.

For parameter particle $i$, let
$u_t^{(i)}=\delta_t^{(i)}(S_t)$, and suppose its previous
posterior is
$\mathcal N(m_{t-1}^{(i)},C_{t-1}^{(i)})$.
Following standard Bayesian filtering, we propagate
this posterior through the transition model to obtain
the forecast:
\begin{equation}
\begin{aligned}
q_{t,-}^{\delta,(i)}=\mathcal N(a_t^{(i)},P_t^{(i)}),
\end{aligned}
\label{eq:delta_state_forecast}
\end{equation}
where $a_t^{(i)}=A_tm_{t-1}^{(i)}$ and $P_t^{(i)}=A_tC_{t-1}^{(i)}A_t^\top+\mathcal Q_{\delta,t}.$
On a fixed support, $A_t=\rho_\delta I$ and
$\mathcal Q_{\delta,t}=Q_{\delta,t}$.
For expanding supports, the transition also accounts
for uncertainty at newly added locations
(Appendix~\ref{app:brpc_state_maps}).

Given $\theta_t^{(i)}$, we compute the residuals
$r_t^{(i)}=Y_t-y_s(X_t,\theta_t^{(i)})$
and model them as
\begin{equation}
\label{eq:particle_residual_likelihood}
r_t^{(i)}
=G_t^{(i)}u_t^{(i)}+\varepsilon_t,
\qquad
\varepsilon_t\sim\mathcal N(0,R_t),
\end{equation}
where $G_t^{(i)}$ maps the discrepancy from the support
locations to the observed inputs, and $R_t$ accounts
for observation noise and remaining spatial
approximation uncertainty.
\end{crblock}
The posterior update with the likelihood from \eqref{eq:particle_residual_likelihood} is given as follows:

\begin{proposition}[KL-regularized Bayesian update in information form]
\label{prop:recursive_delta_update}
Consider the KL-regularized update
\begin{equation}
q_t^{\delta,(i)}
=
\argmin_{q}
\left\{
-\eta_\delta\,
\mathbb{E}_{q}
\big[
\log p(r_t^{(i)} \mid u_t^{(i)})
\big]
+
\mathrm{KL}\!\left(
q(u_t^{(i)})\,\|\,q_{t,-}^{\delta,(i)}(u_t^{(i)})
\right)
\right\},
\label{eq:particle_delta_kl_update}
\end{equation}
with $\eta_\delta > 0$. Then $q_t^{\delta,(i)}$ is Gaussian,
$q_t^{\delta,(i)}=\mathcal{N}(m_t^{(i)},C_t^{(i)})$, with information-form update
\begin{align}
J_t^{(i)}
&:=
(C_t^{(i)})^{-1}
=
(P_t^{(i)})^{-1}
+
\eta_\delta
(G_t^{(i)})^\top R_t^{-1} G_t^{(i)}, \label{eq:delta_info_update} \\
h_t^{(i)}
&:=
J_t^{(i)} m_t^{(i)}
=
(P_t^{(i)})^{-1} a_t^{(i)}
+
\eta_\delta
(G_t^{(i)})^\top R_t^{-1} r_t^{(i)}. \label{eq:delta_natparam_update}
\end{align}
where the posterior mean is the unique minimizer of
\begin{equation}
m_t^{(i)}
=
\argmin_u
\left\{
\eta_\delta
\frac{1}{2}
\|r_t^{(i)}-G_t^{(i)}u\|^2_{R_t^{-1}}
+
\frac{1}{2}
\|u-a_t^{(i)}\|^2_{(P_t^{(i)})^{-1}}
\right\}.
\label{eq:delta_mean_update_particle}
\end{equation}

\end{proposition}

\noindent\textit{Interpretation.}
Equations \eqref{eq:delta_info_update}–\eqref{eq:delta_natparam_update} coincide
with a linear–Gaussian Bayesian update (Kalman filtering in information form),
with $\eta_\delta$ acting as a tempering factor that modulates the effective
precision of the likelihood. In particular, $\eta_\delta$ controls the
bias–variance trade-off of discrepancy adaptation: larger $\eta_\delta$ increases
responsiveness to new residual evidence, while smaller $\eta_\delta$ enforces
temporal smoothing.

The particle-specific construction provides a direct Bayesian analogue of static
BPC, as each parameter particle induces its own residual batch and corresponding
conditional discrepancy posterior. For computational efficiency, we employ a
shared discrepancy approximation in the main experiments. Specifically, we form
the weighted residual batch
\begin{equation}\label{eq:res4shared}
r_t
=
\sum_{i=1}^N w_t^{(i)} r_t^{(i)}
=
Y_t - \sum_{i=1}^N w_t^{(i)} y_s(X_t,\theta_t^{(i)}),    
\end{equation}
and apply the same update as in
Proposition~\ref{prop:recursive_delta_update} with $(G_t,R_t)$ defined using the
shared representation. This corresponds to replacing the particle mixture of
likelihoods by a single moment-matched surrogate, reducing computational
complexity while stabilizing estimation under finite $N$.

Further details, including the construction of $G_t^{(i)}$ from Gaussian process
kernels and the effect of support growth on $P_t^{(i)}$, are provided in Appendix~\ref{app:brpc_state_maps} and
Appendix~\ref{app:restart_state_ablation}.

\paragraph{Tracking property of the discrepancy update.}
Proposition~\ref{prop:recursive_delta_update} admits a mean-form interpretation of the online discrepancy update. Suppressing the particle index (the result applies to both particle-specific and shared representations), define $\ell_t(u)=\tfrac{1}{2}\|r_t-G_tu\|^2_{R_t^{-1}}$ and $M_t=P_t^{-1}$. For the discrepancy representations considered here, the pre-update mean satisfies $a_t=A_tm_{t-1}$, \crnew{where $A_t$ is the Gaussian--Markov transition in \eqref{eq:delta_state_forecast}, including temporal persistence and the support map.} The update \eqref{eq:delta_mean_update_particle} then becomes $m_t=\argmin_u\{\eta_\delta \ell_t(u)+\tfrac{1}{2}\|u-A_tm_{t-1}\|^2_{M_t}\}$.

\begin{assumption}[Stable propagation of discrepancy information]
\label{ass:brpc_transport_contraction}
There exists $\gamma\in[0,1)$ such that $A_t^\top M_t A_t \preceq \gamma M_{t-1}$ for all $t$.
\end{assumption}

Assumption~\ref{ass:brpc_transport_contraction} requires that
propagating the discrepancy estimate does not amplify
past errors, with a sufficient condition provided in
Appendix~\ref{app:proof_brpc_dynamic_regret}.
Under this assumption, we obtain the following bound on
the cumulative discrepancy-fitting loss relative to a
reference sequence.

\begin{theorem}[Tracking of the BRPC discrepancy mean]
\label{thm:brpc_tracking}
Under Assumption~\ref{ass:brpc_transport_contraction}, for any reference sequence $\{v_t\}$,
\begin{equation}\label{eq:brpc_tracking_bound}
\sum_{t=1}^T \{\ell_t(m_t)-\ell_t(v_t)\}
\le
\frac{1}{2\eta_\delta}\|v_0-m_0\|^2_{M_0}
+
\frac{1}{2\eta_\delta(1-\gamma)}
\sum_{t=1}^T
\|v_t-A_tv_{t-1}\|^2_{M_t}.
\end{equation}
\end{theorem}

Theorem~\ref{thm:brpc_tracking} bounds the cumulative loss of $m_t$ relative to any reference path $\{v_t\}$. The first term reflects initialization, while the second captures variation under the same transition matrices. Gradual evolution yields a small variation term, whereas abrupt changes make it large, motivating the restart mechanisms in Section~\ref{sec:reset_control}. If an oracle sequence $\{v_t^\star\}$ satisfies $r_t=G_tv_t^\star+\xi_t$, then $\ell_t(v_t^\star)=\tfrac{1}{2}\|\xi_t\|^2_{R_t^{-1}}$, so the bound controls cumulative residual-fitting loss when the oracle path fits the data and evolves gradually. 

Further analysis of local joint parameter--discrepancy
tracking under stronger assumptions is provided in
Appendix~\ref{app:cr_local_coupled}.
The proof of Theorem~\ref{thm:brpc_tracking} is provided
in Appendix~\ref{app:proof_brpc_dynamic_regret}.

\subsection{Prediction and prequential likelihood}
\label{subsec:prediction_prequential}

BRPC prediction combines projected parameter particles with their corresponding
conditional discrepancy posteriors. After assimilating batch $t$, each particle
$i$ induces a Gaussian predictive distribution for a new input set $X_\star$:
\[
p_t^{(i)}(Y_\star \mid X_\star)
=
\mathcal{N}\!\left(
Y_\star;\,
\mu_{\star,t}^{(i)},\,
\Sigma_{\star,t}^{(i)}
\right),
\]
where $\mu_{\star,t}^{(i)}$ and $\Sigma_{\star,t}^{(i)}$ incorporate the simulator
prediction $y_s(X_\star,\theta_t^{(i)})$, the conditional Gaussian process
discrepancy mean and covariance, and the observation noise.
The overall BRPC predictive distribution is then given by the particle mixture
\begin{equation}
p_t(Y_\star \mid X_\star)
=
\sum_{i=1}^N
w_t^{(i)}\,
p_t^{(i)}(Y_\star \mid X_\star),
\label{eq:brpc_particle_mixture_predictive}
\end{equation}
which can be interpreted as a Monte Carlo approximation to the posterior
predictive distribution obtained by marginalizing over the latent calibration
parameter.
In prequential (online) use, this same predictive distribution is evaluated prior
to assimilating the incoming batch and serves as the predictive evidence for
restart decisions. Explicit forms of the Gaussian process predictive mappings for
the particle-specific, fixed-support, and shared implementations are provided in
Appendix~\ref{app:predictive_details}.

\section{Restart Mechanism for BRPC}
\label{sec:reset_control}
BRPC is designed to be stable under drift. Both the online updates of $\theta_t$ and $\delta_t(\cdot)$ are KL-regularized, promoting smooth evolution of the posterior distributions over time. Under an abrupt regime change, however, these updates may fail to track rapid shifts in the underlying system. As a result, parameter particles and discrepancy information carried over from the previous regime can introduce substantial bias in the new regime, leading to slow adaptation and persistent errors due to outdated calibration evidence. This observation motivates coupling BRPC with a change-point-driven restart
mechanism. The purpose of the restart is to determine
when the current BRPC process—which accumulates information from past data
streams—should be discarded and reinitialized. In this section, we propose three restart mechanisms paired with BRPC\crnew{, as alternative controllers on the same BRPC backbone}.

\subsection{B-BRPC: BOCPD-based restart}
\label{subsec:bbrpc}
A key latent variable in this setting is the restart time, corresponding to the
occurrence of an abrupt system change. We adopt the BOCPD framework
\citep{adams2007bayesian} to maintain a posterior over plausible restart times.
Each hypothesized restart time defines a BRPC expert $e$ with its own calibration
state. Prior to assimilating $(X_t,Y_t)$, expert $e$ forms a pre-update predictive
distribution $p^{\mathrm{pre}}_e(Y_t\mid X_t)$ (Appendix~\ref{app:predictive_details}),
with prequential loss $L^{\mathrm{pre}}_{e,t} := -\log p^{\mathrm{pre}}_e(Y_t\mid X_t)$.

Let $\mathcal{E}_{t-1}$ denote the active expert set before observing
$(X_t,Y_t)$, with weights $w_{e,t-1}$ representing posterior beliefs.
At time $t$, a new expert $e_t^{\mathrm{new}}$ is initialized from the restart
prior with predictive law $p^{\mathrm{pre}}_{\mathrm{new}}(Y_t\mid X_t)$
(Appendix~\ref{app:restart_initialization}). Posterior weights are updated via
\begin{equation}
\widetilde w_{e,t} =
\begin{cases}
(1-h_t)w_{e,t-1}p^{\mathrm{pre}}_e(Y_t\mid X_t),
& e\in\mathcal{E}_{t-1},\\
h_t p^{\mathrm{pre}}_{\mathrm{new}}(Y_t\mid X_t),
& e=e_{\mathrm{new}}.
\end{cases}
\label{eq:bbrpc_weight}
\end{equation}

followed by normalization to obtain $w_{e,t}$ and update
$\mathcal{E}_t = \mathcal{E}_{t-1} \cup \{e_t^{\mathrm{new}}\}$.

While standard BOCPD maintains a full mixture over experts, we adopt a
restarted-BOCPD rule \citep{alami2020restarted} to enable explicit resets.
Let $e_t^{\mathrm{anc}}$ denote the anchor expert from the most recent restart,
and $s_e$ the start time of expert $e$. A hard restart is triggered when
\begin{equation}
\mathrm{Restart}^{\mathrm B}_t
=
\mathbf{1}
\left\{
\max_{e:\,s_e > s_{e_t^{\mathrm{anc}}}}
w_{e,t}
>
\rho_{\mathrm B}\,w_{e_t^{\mathrm{anc}},t}
\right\},
\label{eq:bbrpc_hard_restart}
\end{equation}
with margin $\rho_{\mathrm B}\ge 1$ (default $\rho_{\mathrm B}=1$). After a
restart decision, the retained experts are updated via the BRPC parameter and
discrepancy updates, and the expert set may be pruned for computational control. We refer to the resulting method as B-BRPC. The detailed algorithm is described in Algorithm \ref{alg:bbrpc}. 

BOCPD compares restart hypotheses using pre-update predictive
likelihoods. In BRPC, these likelihoods depend on both the model-parameter estimate
and the discrepancy estimate, creating an interaction between the online
estimation procedure and the restart rule. To analyze the interaction, let us take a Gaussian approximation of the pre-update predictive likelihoods, $p^{\mathrm{pre}}_{e}(Y_t\mid X_t) = \mathcal N(\mu_c, \Sigma_c)$ for expert $e$ and
$p^{\mathrm{pre}}_{\mathrm{new}}(Y_t\mid X_t) = \mathcal N(\mu_n, \Sigma_n)$ for the newly created expert.

\begin{proposition}[Gaussian restart odds]
\label{prop:gaussian_restart_odds}
If $Y_t \sim \mathcal N(\mu_\star, \Sigma_\star)$, then
\begin{equation}
\begin{aligned}
\mathbb{E}
\left[
\log \frac{w_{\mathrm{new},t}}{w_{e,t}}
\right]
=
\log \frac{h_t}{(1-h_t)w_{e,t-1}}
+
\frac{1}{2}
\Big(
\log \frac{\det \Sigma_c}{\det \Sigma_n}
+
\mathrm{tr}\big[(\Sigma_c^{-1} - \Sigma_n^{-1}) \Sigma_\star\big]
\\
+
\|\mu_\star - \mu_c\|_{\Sigma_c^{-1}}^2
-
\|\mu_\star - \mu_n\|_{\Sigma_n^{-1}}^2
\Big).
\end{aligned}
\label{eq:expected_restart_odds}
\end{equation}
\end{proposition}

This result highlights two potential restart issues. If the old expert \(e\) is overly confident, i.e., \(\Sigma_c\) is small, ordinary within-regime variability may be mistaken for restart evidence. Conversely, if discrepancy estimation adapts too quickly reducing $\|\mu_\star - \mu_c\|_{\Sigma_c^{-1}}^2$, it can absorb changes in the data, reducing predictive contrast and delaying detection of true regime changes. Effective restart therefore requires balancing predictive sharpness and adaptation speed, motivating a variant of B-BRPC introduced in Section \ref{subsec:rra}.

\subsection{C-BRPC: score-based restart}
\label{subsec:cbrpc}
B-BRPC maintains multiple restart-time hypotheses and associated BRPC experts,
which can be computationally expensive. As a scalable alternative, we propose a
score-based restart mechanism that maintains a single active BRPC expert and
monitors its predictive fitness via a CUSUM-type statistic
\citep{xie2023window}. The fitness score is defined from the pre-update
predictive law as
\begin{equation}
s_t
:=
-\frac{1}{K_t}
\log p^{\mathrm{pre}}(Y_t\mid X_t),
\label{eq:score_def}
\end{equation}
which corresponds to the average prequential loss on the current batch.

To account for scale variation, the score is standardized as
$s_t^{\mathrm{std}} = (s_t-\widehat\mu_{t-1})/\max(\widehat\sigma_{t-1},\sigma_{\min})$,
where $\widehat\mu_{t-1}$ and $\widehat\sigma_{t-1}$ are running estimates since
the most recent restart and $\sigma_{\min}>0$ ensures numerical stability.

Following window-limited CUSUM procedures, we define
\begin{equation}
G_t^{\mathrm{score}}
=
\max_{1\le m\le W}
\sqrt{m}\,
\left(
\bar s^{\mathrm{std}}_{t-m+1:t}
-
\kappa
\right)_+,
\label{eq:wcusum_stat}
\end{equation}
where $\bar s^{\mathrm{std}}_{t-m+1:t} = \frac{1}{m}\sum_{u=t-m+1}^{t}s_u^{\mathrm{std}}$,
$W$ is the maximum window size, $\kappa$ is a drift allowance, and
$(a)_+=\max(a,0)$. A restart is triggered when
$\mathrm{Restart}^{\mathrm C}_t = \mathbf{1}\{G_t^{\mathrm{score}} > h_{\mathrm C}\}$,
with threshold $h_{\mathrm C}>0$. When a restart is triggered, the parameter posterior, discrepancy
posterior, and score statistics are reinitialized prior to assimilating $(X_t,Y_t)$. The resulting method is referred to as C-BRPC. The detailed algorithm is described in Algorithm \ref{alg:cbrpc}.

The following proposition analyzes the false-alarm and missed-detection behavior of C-BRPC as a
function of its tuning parameters, thereby providing guidance for parameter selection. 
\begin{proposition}[wCUSUM guarantees]
\label{prop:wcusum_bounds}
\crnew{Under the score conditions in Appendix~\ref{app:proof_wcusum_bounds},}
\[
\mathbb{P}\!\left(\max_{t \le T} G_t^{\mathrm{score}} > h_{\mathrm C} \right)
\le
T W e^{-h_{\mathrm C}^2/2}.
\]
\crnew{For $\delta>\kappa$ and sufficient $W$, detection takes}
$O\!\left(h_{\mathrm C}^2/(\delta-\kappa)^2\right)$ batches
\crnew{with high probability}.
\end{proposition}
The threshold $h_{\mathrm C}$
controls false alarms, $\kappa$ determines sensitivity to gradual variation,
and $W$ defines the detection window. Larger $\kappa$ reduces false alarms but
delays detection, while larger $W$ increases sensitivity at the cost of higher
false-alarm risk.

\subsection{B-BRPC-RRA: B-BRPC with $\delta_t(x)$ estimated with residual re-anchoring}
\label{subsec:rra}
The third restart mechanism is a variant of B-BRPC. It employs the same BOCPD
framework but differs in how the discrepancy $\delta_t(x)$ is estimated. Since
the estimate of $\delta_t(x)$ directly affects the predictive evidence and,
consequently, the restart decision, we treat this as a distinct mechanism. 

Specifically, while B-BRPC performs online conditional posterior updates of the
discrepancy $\delta_t(x)$, the proposed variant fits an offline Gaussian process
model for $\delta_t(x)$ using residuals computed from data accumulated since the
most recent restart. For particle $i$, the residuals are defined as
\begin{equation}\label{eq:res4RRA}
r_{t'}^{(i)} = Y_{t'} - y_s(X_{t'}, \theta_t^{(i)}), \qquad t' = t_r+1, \dots, t,    
\end{equation}
\crnew{Here $t_r$ is the last restart time. For each particle $i$, we recompute all residuals since
the last restart using its most recent estimate
$\theta_t^{(i)}$.
We then refit a GP to these residuals, instead of
updating the discrepancy model recursively as in B-BRPC.
This avoids relying on residuals computed with earlier,
potentially inaccurate parameter estimates.
The parameter updates and BOCPD restart rule remain
unchanged, but repeatedly fitting the GP increases
computational cost.}


\vspace{-10pt}
\begin{crblock}
\paragraph{Choosing a restart mechanism.}
The three restart mechanisms offer different trade-offs between detection performance and computational cost.
\textbf{C-BRPC} is our recommended default, as it maintains a single BRPC state and provides a favorable balance of calibration accuracy, restart reliability, and efficiency, particularly under mixed or uncertain nonstationarity.
\textbf{B-BRPC} maintains multiple hypotheses about possible
restart times, each with its own calibration state.
This mixture-of-experts approach allows it to account
for uncertainty about when a change occurred,
potentially improving parameter estimation and change
detection when the additional computational cost is acceptable.
\textbf{B-BRPC-RRA} is particularly effective in settings dominated by abrupt changes, where segment-wise discrepancy refitting improves changepoint localization at the extra cost of computation.
\end{crblock}

\providecommand{\BRPCrevision}[1]{\textcolor{red}{#1}}

\section{Experiments}
\label{sec:experiments}

We evaluate BRPC on two benchmark families: a synthetic streaming benchmark and a
plant-simulation benchmark. The synthetic benchmark is adapted from the static
calibration example of Gu and Wang~\cite{gu2018scaled}. The simulator is
$y_s(x,\theta)=\sin(\theta x)+5x$, and the physical response is
$\zeta_t(x)=5x\cos(\omega_t x/2)+5x$, observed with Gaussian noise. The plant
simulation benchmark uses a detailed discrete-event simulation of a factory.
\crnew{Each benchmark includes gradual drift, sudden($k$) with $k$ annotated
abrupt changes, and mixed($k$) combining both. Details can be found in
Appendix~\ref{app:benchmark_details}.}

We evaluate performance using several key metrics: the average estimation
error of the model parameters $\theta_t$ ($\theta$-RMSE), the overall
prediction error of the response (Response RMSE), and the number of restart
events. To assess restart accuracy, we measure the alignment between detected
restart times and true abrupt change points using Precision@2 and Recall@2,
{which consider detections correct only if they occur within two
batches after an annotated jump; early detections are not counted as correct.}
{Further baseline descriptions, implementation details,
and additional results, including detection delay, runtime, and ablation studies, are provided in
Appendix~\ref{app:additional_results}.}


\begin{table}[!t]
\centering
\caption{\crnew{Main benchmark results (mean $\pm$ standard deviation).
Synthetic results are aggregated across configurations, with
25 random seeds per configuration; plant-simulation results
are based on 10 runs. Detailed experimental configurations
and run counts are provided in Appendix~\ref{app:common_experimental_settings}.}}
\label{tab:brpc_main}
\begingroup
\fontsize{7.1}{8.4}\selectfont
\setlength{\tabcolsep}{2.5pt}
\renewcommand{\arraystretch}{1.02}
\begin{tabular}{@{}llccccc@{}}
\toprule
Scenario / mode & Method & $\theta$-RMSE$\downarrow$ & Response RMSE$\downarrow$ & Restarts & Precision@2$\uparrow$ & Recall@2$\uparrow$ \\
\midrule
\multicolumn{7}{@{}l}{\textit{Synthetic benchmark}} \\
\addlinespace[1pt]
Drifting & BC(80)        & $0.142\pm0.010$ & $0.622\pm0.242$ & -- & -- & -- \\
 & DA                  & $0.090\pm0.030$ & $1.909\pm0.116$ & -- & -- & -- \\
 & B-BRPC              & $0.015\pm0.003$ & $0.423\pm0.199$ & $3.192\pm1.786$ & -- & -- \\
 & C-BRPC              & $0.015\pm0.002$ & $0.465\pm0.234$ & $1.680\pm1.005$ & -- & -- \\
 & B-BRPC-RRA          & $0.015\pm0.002$ & $0.437\pm0.115$ & $8.128\pm2.345$ & -- & -- \\
\addlinespace[1pt]
Sudden(3) & BC(80)      & $0.166\pm0.079$ & $0.839\pm0.437$ & -- & -- & -- \\
 & DA                  & $0.178\pm0.162$ & $1.994\pm0.070$ & -- & -- & -- \\
 & B-BRPC              & $0.020\pm0.014$ & $0.592\pm0.281$ & $2.760\pm1.875$ & 0.657 & 0.682 \\
 & C-BRPC              & $0.027\pm0.024$ & $0.566\pm0.237$ & $3.187\pm1.392$ & 0.749 & 0.713 \\
 & B-BRPC-RRA          & $0.018\pm0.012$ & $0.512\pm0.201$ & $3.050\pm0.218$ & 0.986 & 0.999 \\
\addlinespace[1pt]
Mixed(3) & BC(80)       & $0.125\pm0.020$ & $0.647\pm0.103$ & -- & -- & -- \\
 & DA                  & $0.092\pm0.039$ & $1.873\pm0.148$ & -- & -- & -- \\
 & B-BRPC              & $0.019\pm0.012$ & $0.436\pm0.059$ & $2.547\pm1.166$ & 0.709 & 0.847 \\
 & C-BRPC              & $0.017\pm0.010$ & $0.441\pm0.079$ & $3.240\pm0.984$ & 0.468 & 0.780 \\
 & B-BRPC-RRA          & $0.021\pm0.016$ & $0.505\pm0.066$ & $4.813\pm1.565$ & 0.409 & 0.960 \\
\midrule
\multicolumn{7}{@{}l}{\textit{Plant-simulation benchmark}} \\
Drifting
& BC(80)       & 3.936$\pm$0.123 & 3.659$\pm$0.045 & --                & --    & --    \\
& DA           & 9.562$\pm$0.276 & 5.027$\pm$0.189 & --                & --    & --    \\
& B-BRPC       & 0.800$\pm$0.018 & 0.988$\pm$0.031 & 15.7$\pm$0.082    & --    & --    \\
& C-BRPC       & 0.809$\pm$0.024 & 1.001$\pm$0.027 & 7.0$\pm$0.000     & --    & --    \\
& B-BRPC-RRA   & 0.788$\pm$0.022 & 1.109$\pm$0.042 & 17.9$\pm$0.320    & --    & --    \\
Sudden(5)
& BC(80)       & 6.023$\pm$0.250 & 1.659$\pm$0.147 & --                & --    & --    \\
& DA           & 4.259$\pm$0.238 & 1.000$\pm$0.153 & --                & --    & --    \\
& B-BRPC       & 0.957$\pm$0.086 & 0.164$\pm$0.027 & 7.4$\pm$0.970     & 0.516 & 0.760 \\
& C-BRPC       & 1.056$\pm$0.238 & 0.169$\pm$0.035 & 4.3$\pm$0.670     & 0.762 & 0.660 \\
& B-BRPC-RRA   & 1.048$\pm$0.411 & 0.128$\pm$0.030 & 5.0$\pm$0.000     & 1.000 & 1.000 \\
Mixed($\approx$2--3)
& BC(80)       & 4.502$\pm$1.123 & 3.034$\pm$0.345 & --                & --    & --    \\
& DA           & 9.669$\pm$1.276 & 5.646$\pm$0.459 & --                & --    & --    \\
& B-BRPC       & 1.580$\pm$0.898 & 0.554$\pm$0.214 & 12.4$\pm$5.276    & 0.184 & 0.950 \\
& C-BRPC       & 1.582$\pm$1.111 & 0.561$\pm$0.178 & 5.7$\pm$2.500     & 0.312 & 0.747 \\
& B-BRPC-RRA   & 1.537$\pm$0.857 & 0.569$\pm$0.176 & 13.0$\pm$5.330    & 0.186 & 1.000 \\
\bottomrule
\end{tabular}
\vspace{-20pt}
\endgroup
\end{table}


We report comparisons of the three proposed methods---B-BRPC, C-BRPC, and
B-BRPC-RRA---against two \textcolor{blue}{main} baselines
\textcolor{blue}{in Table~1}:
BC(80), a sliding-window Bayesian calibration approach that
periodically fits the static \textcolor{blue}{Kennedy--O'Hagan (KOH)}
model~\cite{kennedy2001bayesian} using the most recent 80 observations,
and DA, a data-assimilation-based calibration
method~\cite{ward2021continuous}. All three BRPC variants substantially improve calibration accuracy ($\theta$-RMSE and Response RMSE) over BC(80) and DA across both synthetic and plant-simulation benchmarks. Among the proposed methods, B-BRPC and C-BRPC achieve comparable $\theta$-RMSE, but differ in restart behavior: in the plant-simulation
benchmark, B-BRPC generally attains high recall at the cost of excessive restarts, whereas C-BRPC reduces restart frequency and improves event-level precision. B-BRPC-RRA performs best in jump-dominated
settings (\textit{Sudden(3)} and \textit{Sudden(5)}), where residual refitting yields near-ideal restart behavior. Overall, C-BRPC is a robust default under heterogeneous nonstationarity, while B-BRPC-RRA is most effective when abrupt changes occur more frequently.

\noindent{\textbf{Sources of performance gains.}
Table~\ref{tab:brpc_sources} provides additional comparisons
to examine whether the improvements come from the parameter
estimation method or the restart mechanism.
Under sudden changes, BRPC achieves comparable or even better
estimation accuracy than static BPC, even when the
true changepoints are provided to BPC.
Under gradual drift, BRPC also outperforms joint EnKF and
discounted BC in parameter estimation, although both methods
are designed to adapt to nonstationarity.
These results suggest that BRPC combines the benefits of
projected calibration with online adaptation while
maintaining competitive response prediction accuracy.
Comparing BRPC with and without restarting further shows
that restarting reduces parameter estimation errors under
sudden and mixed changes.
Under gradual drift, BRPC without restarting already
performs well, indicating that frequent restarts are
not always necessary.}

\begin{table}[t]
\centering
\caption{{
Additional baseline comparisons and an ablation study are conducted on the synthetic benchmark using 25 random seeds. The baseline methods include Bayesian projected calibration (BPC)~\cite{xie2021bayesian}, applied independently to segments defined by either the oracle regime boundaries or Bayesian online change-point detection (BOCPD); ensemble Kalman filtering (EnKF)~\cite{evensen2009ensemble}, which jointly estimates calibration parameters and model bias terms; and discounted Bayesian calibration (BC), a KOH-based calibration approach that assigns exponentially decaying weights to older observations. We also compare BRPC with and without the restart mechanism to assess its contribution to the overall performance.
}}
\label{tab:brpc_sources}

{
\begingroup
\footnotesize
\setlength{\tabcolsep}{4.5pt}
\renewcommand{\arraystretch}{1.10}

\begin{tabular}{@{}lrrrrrr@{}}
\toprule
& \multicolumn{3}{c}{$\theta$-RMSE $\downarrow$}
& \multicolumn{3}{c}{Response RMSE $\downarrow$} \\
\cmidrule(lr){2-4}
\cmidrule(l){5-7}
Method
& Drift & Sudden & Mixed
& Drift & Sudden & Mixed \\
\midrule

BPC (oracle segments)
& 0.060 & 0.033 & 0.040
& 0.60 & \textbf{0.62} & 0.49 \\

BPC (BOCPD segments)
& 0.071 & 0.067 & 0.075
& 0.73 & 0.76 & 0.58 \\

Joint EnKF
& 0.109 & 0.139 & 0.373
& 0.56 & 1.07 & 0.85 \\

Discounted BC ($\gamma=0.90$)
& 0.090 & 0.077 & 0.076
& 0.39 & 0.72 & 0.51 \\

\midrule

BRPC (no restart)
& \textbf{0.014} & 0.031 & 0.036
& \textbf{0.35} & 0.67 & \textbf{0.44} \\

B-BRPC
& 0.016 & \textbf{0.018} & 0.021
& 0.37 & 0.71 & \textbf{0.44} \\

C-BRPC
& 0.015 & 0.019 & \textbf{0.019}
& 0.44 & 0.63 & \textbf{0.44} \\

\bottomrule
\end{tabular}
\endgroup
}
\vspace{-20pt}
\end{table}

\noindent{\textbf{Discrepancy dynamics and restart sensitivity.}
Hyperparameters such as the parameter random-walk step size,
$\rho_\delta$, and $q_\delta$ control how much the model relies
on past estimates and how quickly it adapts to new observations.
Under gradual drift, appropriately reducing the influence of
older discrepancy estimates can improve prediction accuracy
and reduce unnecessary restarts.
However, under mixed changes, allowing the discrepancy to
adapt too quickly may cause the model to absorb abrupt changes
before they are detected, leading to missed changepoints
(Proposition~\ref{prop:gaussian_restart_odds}); see
Appendix~\ref{app:additional_results}.}

\noindent{\textbf{Scalability and computational cost.}
We further evaluate BRPC in a higher-dimensional setting with
five calibration parameters and 20 input variables
(Appendix~\ref{app:highdim_physical_projected}), as scaling
particle filtering and change-point detection to higher
dimensions can be challenging.
BRPC without restarting achieves the lowest parameter
estimation error in this experiment, although short-window
Bayesian calibration provides better response prediction.
This difference highlights that accurately estimating the
calibration parameters does not necessarily lead to the
best response predictions.
Restarting also increases parameter estimation error in
this setting, suggesting that it is not always beneficial.
Using a fixed number of input locations reduces the
computational cost of updating the discrepancy model,
but maintaining and updating parameter particles remains
computationally expensive
(Appendix~\ref{app:runtime}).}

\noindent{\textbf{Closed-loop decision study.}
Finally, we examine whether online calibration can improve
operational decisions in the plant-simulation benchmark.
We use an upper confidence bound (UCB) rule to select
operating settings based on a digital twin calibrated by B-BRPC.
Compared with a baseline that selects its operating settings
once using a fixed nominal calibration parameter,
the calibrated UCB policy improves cumulative net revenue by
$18.3\%\pm1.5\%$ (Appendix~\ref{app:cr_closed_loop}).
Future work includes validation on physical systems and
integration with dual control~\cite{mesbah2018stochastic}
or reinforcement learning~\cite{m2023model,klenske2016dual}.}

\FloatBarrier

\section{Conclusion}

We introduced Bayesian Recursive Projected Calibration (BRPC), an online Bayesian calibration framework that preserves parameter–discrepancy identifiability under simulator mismatch and nonstationarity, including gradual drift and abrupt shifts. We established theoretical results on discrepancy tracking under gradual evolution and its impact on restart decisions, revealing trade-offs among predictive sharpness, adaptation speed, and restart sensitivity. Numerical experiments on synthetic and plant-simulation benchmarks demonstrate improved calibration accuracy over sliding-window Bayesian calibration and data assimilation baselines. Future work will extend BRPC to high-dimensional digital twins and develop restart mechanisms with stronger false-alarm control and adaptive memory-reset strategies.

\label{end:main_content}
\begin{ack}
The authors gratefully acknowledge the support from the National Science Foundation (grants \# 2420358, 2428742) and the Air Force Office of Scientific Research (grant \# FA9550-23-1-0673).
\end{ack}

\FloatBarrier
\bibliographystyle{plainnat}
\bibliography{DBC_camera_ready}

\clearpage
\appendix
\section*{Appendix}
\addcontentsline{toc}{section}{Appendix}

\begin{crblock}
\paragraph{Appendix organization.}
Appendix~\ref{appendix:algo} gives algorithms and restart initialization.
Appendix~\ref{app:brpc_construction_details} provides BRPC updates,
state representations, prediction, and their proofs.
Appendix~\ref{app:compatibility_theory} defines nonstationarity and proves
the restart guarantees. Appendix~\ref{app:optional_restart_discrepancy_theory}
studies parameter--discrepancy coupling, including local coupled tracking
and propagation stability diagnostics. Appendix~\ref{app:benchmark_details}
describes benchmarks and experiment configurations.
Appendix~\ref{app:additional_results} organizes additional experiments into
baseline comparisons, representation and state sharing, discrepancy dynamics,
restart diagnostics, scalability, and closed-loop decisions.
Appendices~\ref{app:online_continual_learning} and~\ref{sec:limit} discuss
connections to online learning and the limitations of the framework.
\end{crblock}

\paragraph{Appendix notation conventions.}
Unless otherwise stated, $t$ indexes data batches, $i$ indexes particles,
$e$ indexes restart experts, and $k$ indexes observations within a batch.
The batch is $B_t=(X_t,Y_t)$ with size $K_t$. For any finite input sets
$A$ and $B$, \crnew{$K_{AB}:=K_\phi(A,B)=k_\phi(A,B)$ denotes the
spatial kernel matrix, using a compact subscript notation for the
$K_\phi(A,B)$ matrices in Section~\ref{subsec:brpc_delta_update}.}
The symbol $\mathcal F_t$ denotes the filtration generated by the data batches
and algorithmic randomness up to time $t$. We use superscript ``pre'' for
pre-update predictive quantities, superscript ``sh'' for shared-discrepancy
quantities, and superscript ``F'' for fixed-support quantities. Symbols
introduced only for diagnostics or benchmark generation are local to the
corresponding appendix subsection.

\begin{crblock}
\paragraph{Discrepancy dynamics and covariance notation.}
The spatial kernel is $k_\phi$, and $q_\delta$ is the innovation variance
multiplier in \eqref{eq:discrepancy_dynamics}. Thus the innovation process
is $\xi_t(\cdot)\sim\mathcal{GP}(0,q_\delta k_\phi)$ and
$Q_{\delta,t}=q_\delta K_{S_tS_t}$. The restart covariance is $K_{S_tS_t}$;
the temporal innovation uses the same spatial kernel, scaled by $q_\delta$.
$H_t$ is a spatial support map and $A_t=\rho_\delta H_t$ is the full
temporal/state transition. $\Omega_t$ denotes spatial conditional covariance;
$\mathcal Q_{\delta,t}$ combines temporal innovation and any extension
uncertainty. The forecast is $\mathcal N(a_t,P_t)$, the posterior is
$\mathcal N(m_t,C_t)$, and their precisions are $M_t=P_t^{-1}$ and
$J_t=C_t^{-1}$. $R_t$ is the effective residual-likelihood covariance,
including spatial residual variance only once. All of these objects may
carry particle or expert indices. The same symbol $q_\delta$ is used for
the innovation multiplier in the theory, experimental settings, and
sensitivity tables. The local theorem uses
$\mathcal E_{\theta,t}$ and $\mathcal E_{\delta,t}$ for tracking errors
and $\chi_\delta$ for its propagation contraction factor.
\end{crblock}

\section{Algorithms and Restart Initialization}
\label{appendix:algo}

This appendix gives the algorithmic details for the BRPC procedures used in the
main text. We first state the
projected particle update and then the BOCPD-style restart procedure, and the
wCUSUM-style restart procedure. We also specify the restart-prior initialization
and the BOCPD forecaster recursion used by B-BRPC. 

\begin{algorithm}[H]
\caption{Projected particle update}
\label{alg:projected_particle_update}
\begin{algorithmic}[1]
\Require Previous particles $\{(\theta_{t-1}^{(i)},w_{t-1}^{(i)})\}_{i=1}^N$,
transition kernel $p(\theta_t\mid\theta_{t-1})$, batch $(X_t,Y_t)$, and ESS
threshold $\tau_{\mathrm{ESS}}$.
\State Propagate particles according to the predictive prior in
Eq.~\eqref{eq:theta_predictive_prior}:
$\theta_{t\mid t-1}^{(i)}\sim p(\theta_t\mid \theta_{t-1}^{(i)})$ and
$w_{t\mid t-1}^{(i)}=w_{t-1}^{(i)}$.
\State Compute the projected likelihood in
Eq.~\eqref{eq:discrepancy_free_likelihood}:
$\ell_t^{(i)}=p_{\mathrm{proj}}(Y_t\mid X_t,\theta_{t\mid t-1}^{(i)})$.
\State Reweight particles using Eq.~\eqref{eq:projected_pf_weights}:
$w_t^{(i)}\propto w_{t\mid t-1}^{(i)}(\ell_t^{(i)})^{\eta_\theta}$.
\State Normalize $\{w_t^{(i)}\}_{i=1}^N$ and resample if
$\mathrm{ESS}_t=(\sum_{i=1}^N (w_t^{(i)})^2)^{-1}<\tau_{\mathrm{ESS}}N$.
\end{algorithmic}
\end{algorithm}

\begin{algorithm}[H]
\caption{B-BRPC: BRPC with Restart-BOCPD control}
\label{alg:bbrpc}
\begin{algorithmic}[1]
\Require BRPC variant $\in\{\mathrm{E},\mathrm{P},\mathrm{F},\mathrm{RRA}\}$,
number of particles $N$, ESS threshold $\tau_{\mathrm{ESS}}$, restart hazard
$h_t$, previous expert set $\mathcal E_{t-1}$, and restart state
$\mathcal S_0$ defined in Eq.~\eqref{eq:restart_state}.
\For{$t=1,2,\dots$}
    \State Observe the incoming batch $(X_t,Y_t)$.
    \State Add a fresh expert $e_t^{\mathrm{new}}$ initialized from
$\mathcal S_0$ in Eq.~\eqref{eq:restart_state}, and set
$\mathcal E_t^{\mathrm{cand}}=\mathcal E_{t-1}\cup\{e_t^{\mathrm{new}}\}$.
    \For{each expert $e\in\mathcal E_t^{\mathrm{cand}}$}
        \State Form the pre-update predictive law
        $p_e^{\mathrm{pre}}(Y_t\mid X_t)$ in Eq.~\eqref{eq:particle_mixture_predictive}.
        \State \crnew{The predictive law uses the parameter forecast and
        \eqref{eq:delta_state_forecast} for continuing recursive experts;
        fresh experts use their restart prior without an extra transition.}
    \EndFor
    \State Update BOCPD expert weights using
    Eqs.~\eqref{eq:bbrpc_weight}.
    \State Apply the hard-restart rule in Eq.~\eqref{eq:bbrpc_hard_restart}
    and the pruning rule in Eq.~\eqref{eq:expert_pruning}
    to obtain $\mathcal E_t$.
    \For{each expert $e\in\mathcal E_t$}
        \State Update parameter particles using
        Algorithm~\ref{alg:projected_particle_update}, equivalently
        Eqs.~\eqref{eq:theta_predictive_prior}--\eqref{eq:projected_pf_weights}.
        \State Construct the residual object for the selected BRPC variant:
        Eq.~\eqref{eq:res4particle} for the particle-specific variant,
        Eq.~\eqref{eq:res4shared} for the shared-discrepancy variant, and
        Eq.~\eqref{eq:res4RRA} for RRA.
        \State Update the discrepancy posterior using
        Eqs.~\eqref{eq:delta_info_update}--\eqref{eq:delta_natparam_update}
        \crnew{using the same discrepancy forecast; after a reset, use the fresh prior.
        RRA instead refits on segment-local residuals}.
    \EndFor
\EndFor
\end{algorithmic}
\end{algorithm}

\begin{algorithm}[H]
\caption{C-BRPC: BRPC with wCUSUM control}
\label{alg:cbrpc}
\begin{algorithmic}[1]
\Require BRPC variant $\in\{\mathrm{E},\mathrm{P},\mathrm{F}, \mathrm{RRA}\}$,
number of particles $N$, ESS threshold $\tau_{\mathrm{ESS}}$, window limit $W$,
drift allowance $\kappa$, threshold $h_C$, score floor $\sigma_{\min}$,
active BRPC state, and restart state $\mathcal S_0$ defined in
Eq.~\eqref{eq:restart_state}.
\For{$t=1,2,\dots$}
    \State Observe the incoming batch $(X_t,Y_t)$.
    \State Form the active state's pre-update predictive law
    $p^{\mathrm{pre}}(Y_t\mid X_t)$ using Eq.~\eqref{eq:particle_mixture_predictive}
    \crnew{with the parameter forecast and Gaussian--Markov discrepancy forecast
    \eqref{eq:delta_state_forecast}; a fresh state uses its restart prior}.
    \State Compute the batch-normalized pre-update score $s_t$ using
    Eq.~\eqref{eq:score_def}.
    \State Standardize $s_t$ using the current-segment score history and update
    $G_t^{\mathrm{score}}$ using Eq.~\eqref{eq:wcusum_stat}.
    \If{$G_t^{\mathrm{score}}>h_C$}
        \State Reset the parameter state, discrepancy state and score-standardization history from
        $\mathcal S_0$ in Eq.~\eqref{eq:restart_state}.
    \EndIf
    \State Update parameter particles using
    Algorithm~\ref{alg:projected_particle_update}, equivalently
    Eqs.~\eqref{eq:theta_predictive_prior}--\eqref{eq:projected_pf_weights}.
    \State Construct the residual object for the selected BRPC variant:
    Eq.~\eqref{eq:res4particle} for the particle-specific variant,
        Eq.~\eqref{eq:res4shared} for the shared-discrepancy variant, and
        Eq.~\eqref{eq:res4RRA} for RRA.
    \State Update the discrepancy posterior using
    Eqs.~\eqref{eq:delta_info_update}--\eqref{eq:delta_natparam_update}
        \crnew{using the same discrepancy forecast; after a reset, use the fresh prior.
        RRA instead refits on segment-local residuals}.
\EndFor
\end{algorithmic}
\end{algorithm}

\subsection{Restart-prior initialization}
\label{app:restart_initialization}

This subsection specifies the initial state used whenever B-BRPC creates a fresh
BOCPD branch or C-BRPC triggers a hard reset. Let 
\begin{equation}\label{eq:restart_state}
\mathcal S_0
=
\left(
\{(\theta^{(i)},w^{(i)},
q_0^{\delta,(i)})\}_{i=1}^N
\right)    
\end{equation}

denote the restart state, where the components are the parameter particle
cloud and the discrepancy prior.
\crnew{The index $0$ is local to each expert: the first batch uses this prior
directly. Subsequent batches apply \eqref{eq:discrepancy_dynamics} once
before their pre-update likelihood is evaluated.}

For the parameter state, particles are sampled independently from a user-specified
restart prior
\[
\theta^{(i)} \sim \pi_0^\theta,\qquad
w^{(i)}=\frac1N,\qquad i=1,\ldots,N .
\]
In the experiments, unless otherwise stated, \(\pi_0^\theta\) is the uniform
distribution over the admissible calibration range \(\Theta\). More informative
restart priors can be used when reliable physical or engineering prior knowledge
is available.

For the discrepancy state, the default restart prior is the zero-mean Gaussian
process prior,
\begin{equation} \label{eq:delta_prior}
\delta \sim \mathrm{GP}(0,k_\phi).
\end{equation}
On a finite sample set \(S\), the Gaussian process means a multivariate normal distribution,
\begin{equation} \label{eq:init_delta_prior}
q_0^\delta(\delta(S))
=
\mathcal N(0,K_{SS}),
\end{equation}
where $K_{SS}$ is the covariance matrix evaluated among pairs in $S$.
For instance, using the common squared-exponential (RBF) kernel, the entries of the covariance matrix $K_{SS}$ are given by:
\[
[K_{SS}]_{ij}=\sigma_\delta^2 \exp(-\frac{||S_i-S_j||^2}{2\ell^2}).
\]
where $\phi = (\sigma_\delta^2,\ell)$ are the kernel hyperparameters.

Consequently, the pre-update predictive distribution for a freshly initialized expert (prior to observing any data) explicitly follows:
\begin{equation}\label{eq:pre_new}
p_{new}^{pre}(Y_t | X_t) = \frac{1}{N} \sum_{i=1}^N \mathcal{N}\left(Y_t; y_s(X_t, \theta^{(i)}), K_{X_t X_t} + \Sigma_{\epsilon, t}\right),    
\end{equation}
where $K_{X_t X_t}$ is the prior covariance matrix evaluated at the current batch inputs $X_t$, and $\Sigma_{\epsilon, t}$ is the observation noise covariance.

We also denote an active BRPC state as the combination of parameter particles and discrepancy posterior
$\{(\theta_t^{(i)},w_t^{(i)}, q_t^{\delta,(i)})\}_{i=1}^N$.
For B-BRPC, each expert $e\in\mathcal E_t$ carries its own state
$\{(\theta_t^{(i)},w_t^{(i)}, q_t^{\delta,(i)})\}_{i=1}^N$, start time $s_e$, and BOCPD weight $w_{e,t}$.

\subsection{BOCPD-style forecaster recursion with hard restart}
\label{app:bbrpc_forecaster}

BOCPD maintains a posterior mixture over possible run lengths by updating each
candidate segment expert with its pre-update predictive likelihood. In our
setting, a run-length expert is a candidate BRPC state, and the one-step likelihood
is replaced by the batch pre-update predictive evidence
$p_e^{\mathrm{pre}}(Y_t\mid X_t)$. This term evaluates the probability density of the new observations $(X_t, Y_t)$ by marginalizing over the expert's current parameter particles and discrepancy posterior (Eq.~\eqref{eq:particle_mixture_predictive}) before any state updates occur. Following the hard-restart principle of
restarted BOCPD, we use this BOCPD forecaster recursion to compare continuation
experts with a newly initialized expert, and declare a restart when a post-anchor
candidate overtakes the anchor expert. The only extension from the pointwise
setting is that all likelihood terms below are batch predictive likelihoods.

Let $\mathcal F_t$ denote the history of observations up to time $t$ and $\mathcal{E}_{t-1}$ denote the active expert set before observing batch $t$,
with BOCPD weights $\{w_{e,t-1}\}_{e\in\mathcal{E}_{t-1}}$. The corresponding
mixture predictive law is
\begin{equation}
p^{\mathrm{mix}}(Y_t\mid X_t,\mathcal{F}_{t-1})
=
\sum_{e\in\mathcal{E}_{t-1}}
w_{e,t-1}\,
p^{\mathrm{pre}}_e(Y_t\mid X_t).
\label{eq:app_mixture_predictive}
\end{equation}
With restart hazard $h_t\in(0,1)$, continuation experts and the fresh restart
branch receive unnormalized weights
\begin{equation}
\widetilde w_{e,t}
=
\begin{cases}
(1-h_t)\,w_{e,t-1}\,p^{\mathrm{pre}}_e(Y_t\mid X_t),
& e\in\mathcal{E}_{t-1},\\[4pt]
h_t\,p^{\mathrm{pre}}_{\mathrm{new}}(Y_t\mid X_t),
& e=e_t^{\mathrm{new}},
\end{cases}
\qquad
w_{e,t}
=
\frac{\widetilde w_{e,t}}
{\sum_{e'}\widetilde w_{e',t}} .
\label{eq:app_bbrpc_weight_update}
\end{equation}
Here $p^{\mathrm{pre}}_{\mathrm{new}}(Y_t\mid X_t)$ is the batch predictive
law of a freshly initialized BRPC expert, as in Eq.~\eqref{eq:pre_new}.

Let $e_t^{\mathrm{anc}}$ denote the anchor expert launched at the most recent
hard restart. For each expert $e\in\mathcal E_t$, let $s_e$ denote the segment
start time represented by that expert. Following the restarted-BOCPD hard
decision rule, we declare a restart when some post-anchor candidate has larger
posterior mass than the anchor\crnew{ by the margin $\rho_{\mathrm B}\ge1$
in \eqref{eq:bbrpc_hard_restart} (default $\rho_{\mathrm B}=1$)}:
\begin{equation}
\mathrm{Restart}_t
=
\mathbf{1}
\left\{
\max_{e\in\mathcal E_t:\,s_e>s_{e_t^{\mathrm{anc}}}}
w_{e,t}
>
\crnew{\rho_{\mathrm B}}\,w_{e_t^{\mathrm{anc}},t}
\right\}.
\label{eq:app_bbrpc_restart_rule}
\end{equation}
When $\mathrm{Restart}_t=1$, the selected post-anchor candidate becomes the new
anchor and the BRPC state is reinitialized for the new segment. After the hard-restart decision, we retain at most $M_{\max}$ experts with the
largest BOCPD weights:
\begin{equation}
\mathcal E_t
=
\operatorname{Top}_{M_{\max}}
\left(
\mathcal E_t^{\mathrm{cand}};\{w_{e,t}:e\in\mathcal E_t^{\mathrm{cand}}\}
\right).
\label{eq:expert_pruning}    
\end{equation}

We use the hard-restart decision in the main experiments because it makes
restart quality directly measurable through event-level metrics such as
Precision@2, Recall@2, and Delay@2. \crnew{Applications that favor predictive
averaging can instead retain the BOCPD mixture over experts and use
\eqref{eq:app_mixture_predictive}.}

\section{BRPC Updates, State Representations, and Prediction}
\label{app:brpc_construction_details}

This appendix provides the construction details supporting Section~\ref{sec:brpc}.
Appendix~\ref{app:brpc_theta_details} proves the projected parameter update in
Proposition~\ref{prop:projected_parameter_update}. 
Appendix~\ref{app:recursive_delta_update_proof} proves the recursive Gaussian
discrepancy update in Proposition~\ref{prop:recursive_delta_update}. 
Appendix~\ref{app:brpc_state_maps} describes the finite-dimensional discrepancy
state representations used by BRPC-E, BRPC-P, and BRPC-F. 
Appendix~\ref{app:proof_brpc_dynamic_regret} proves the tracking bound for the
recursive discrepancy mean. Appendix~\ref{app:predictive_details} gives the
predictive laws used for forecasting and restart decisions.

\subsection{Proof of Proposition~\ref{prop:projected_parameter_update}}
\label{app:brpc_theta_details}

\begin{proof}[Proof of Proposition~\ref{prop:projected_parameter_update}]
The objective in \eqref{eq:online_theta_kl_update} is
\[
\mathcal J(q)
=
-\eta_\theta
\int q(\theta_t)
\log p_{\rm proj}(Y_t\mid X_t,\theta_t)\,d\theta_t
+
\int q(\theta_t)
\log
\frac{q(\theta_t)}
{\bar q_{t\mid t-1}^{\theta}(\theta_t)}
\,d\theta_t .
\]
Using a Lagrange multiplier for the constraint $\int q(\theta_t)d\theta_t=1$,
the first-order condition gives
\[
-\eta_\theta \log p_{\rm proj}(Y_t\mid X_t,\theta_t)
+
\log q(\theta_t)
-
\log \bar q_{t\mid t-1}^{\theta}(\theta_t)
+
1+\lambda
=
0.
\]
Hence
\[
q_t^\theta(\theta_t)
\propto
\bar q_{t\mid t-1}^{\theta}(\theta_t)
p_{\rm proj}(Y_t\mid X_t,\theta_t)^{\eta_\theta}.
\]

To see the particle-filter approximation explicitly, suppose the predictive prior is represented by
particles $\{(\theta^{(i)}_{t|t-1}, w^{(i)}_{t|t-1})\}_{i=1}^N$. We restrict the candidate
posterior to the same particle support and optimize only over the updated
weights $\omega=(\omega_1,\ldots,\omega_N)\in\Delta_N$ where $\Delta_N$ denotes the $(N-1)-$dimensional probability simplex. The finite-particle KL
update is
\[
\min_{\omega\in\Delta_N}
\left\{
-\eta_\theta
\sum_{i=1}^N
\omega_i
\log p_{\rm proj}(Y_t\mid X_t,\theta_{t|t-1}^{(i)})
+
\sum_{i=1}^N
\omega_i
\log
\frac{\omega_i}{w_{t|t-1}^{(i)}}
\right\}.
\]
The first-order condition over the simplex gives
\[
\omega_i
\propto
w_{t|t-1}^{(i)}
p_{\rm proj}(Y_t\mid X_t,\theta_{t|t-1}^{(i)})^{\eta_\theta},
\qquad
\sum_{i=1}^N\omega_i=1,
\]
which is exactly the projected particle-filter weight update
\eqref{eq:projected_pf_weights}. Propagation through
$p(\theta_t\mid\theta_{t-1})$ produces the predictive particle cloud, and
resampling is the standard SMC step used to control weight degeneracy.
\end{proof}

When $\Sigma_{\theta,t}=\sigma_t^2I$, the projected likelihood satisfies
\[
-\log p_{\rm proj}(Y_t\mid X_t,\theta)
=
\frac{1}{2\sigma_t^2}
\sum_{k=1}^{K_t}
\{y_{t,k}-y_s(x_{t,k},\theta)\}^2
+
c_t,
\]
where $c_t$ does not depend on $\theta$. Therefore minimizing the negative
projected log likelihood is equivalent to minimizing the empirical projected
squared-loss criterion on the current batch. If the inputs are sampled from the
local design distribution $F_X$, this empirical loss is an unbiased batch
analogue of the projected calibration objective up to observation-noise
variance.

\subsection{Proof of Proposition~\ref{prop:recursive_delta_update}}
\label{app:recursive_delta_update_proof}

Before proving Proposition~\ref{prop:recursive_delta_update}, we spell out the
objects used in the conditional discrepancy update. For a parameter particle
$\theta_t^{(i)}$, define the residual batch
\begin{equation}\label{eq:res4particle}
r_t^{(i)}
=
Y_t-y_s(X_t,\theta_t^{(i)}).    
\end{equation}

Let $u_t^{(i)}$ denote the finite-dimensional discrepancy state associated with
particle $i$. For example, in the expanding-support construction, if the
previous discrepancy state is carried on support $S_{t-1}$ and the current
batch inputs enlarge the support to $S_t=S_{t-1}\cup X_t$, then
\[
u_t^{(i)}=\crnew{\delta_t^{(i)}(S_t)}.
\]
The pre-update discrepancy distribution is
\[
q_{t,-}^{\delta,(i)}(u_t^{(i)})
=
\mathcal N(a_t^{(i)},P_t^{(i)}),
\]
where $a_t^{(i)}$ and $P_t^{(i)}$ are obtained by propagating the previous
Gaussian discrepancy posterior to the current support
\crnew{using \eqref{eq:delta_state_forecast}, with $A_t$ and
$\mathcal Q_{\delta,t}$ specified in Appendix~\ref{app:brpc_state_maps}}.
The batch residual model
is
\[
r_t^{(i)}\mid u_t^{(i)}
\sim
\mathcal N(G_t^{(i)}u_t^{(i)},R_t).
\]
Here $G_t^{(i)}$ maps the discrepancy state to discrepancy values at the current
inputs. In the expanding-support case, it is the selection matrix that extracts
the coordinates of $u_t^{(i)}$ corresponding to $X_t$. The covariance $R_t$ is
the residual-likelihood covariance. It is not the GP kernel; it may include
observation noise, nugget regularization, and approximation variance from the
chosen finite-dimensional discrepancy representation.

\begin{proof}[Proof of Proposition~\ref{prop:recursive_delta_update}]
Fix a particle $i$. The KL-regularized discrepancy update is
\[
q_t^{\delta,(i)}
=
\argmin_q
\left\{
-\eta_\delta
\mathbb E_q[\log p(r_t^{(i)}\mid u_t^{(i)})]
+
\KL\!\left(
q(u_t^{(i)})\,\|\,q_{t,-}^{\delta,(i)}(u_t^{(i)})
\right)
\right\}.
\]
Equivalently, it has the tilted Bayesian form
\[
q_t^{\delta,(i)}(u_t^{(i)})
\propto
q_{t,-}^{\delta,(i)}(u_t^{(i)})
p(r_t^{(i)}\mid u_t^{(i)})^{\eta_\delta}.
\]
Using
\[
q_{t,-}^{\delta,(i)}(u_t^{(i)})
\propto
\exp\left\{
-\frac12
\|u_t^{(i)}-a_t^{(i)}\|^2_{(P_t^{(i)})^{-1}}
\right\}
\]
and
\[
p(r_t^{(i)}\mid u_t^{(i)})^{\eta_\delta}
\propto
\exp\left\{
-\frac{\eta_\delta}{2}
\|r_t^{(i)}-G_t^{(i)}u_t^{(i)}\|^2_{R_t^{-1}}
\right\},
\]
we obtain
\[
q_t^{\delta,(i)}(u_t^{(i)})
\propto
\exp\left\{
-\frac12
\|u_t^{(i)}-a_t^{(i)}\|^2_{(P_t^{(i)})^{-1}}
-
\frac{\eta_\delta}{2}
\|r_t^{(i)}-G_t^{(i)}u_t^{(i)}\|^2_{R_t^{-1}}
\right\}.
\]
Therefore $q_t^{\delta,(i)}$ is Gaussian. Writing
\[
q_t^{\delta,(i)}
=
\mathcal N(m_t^{(i)},C_t^{(i)}),
\]
its information-form parameters are
\[
J_t^{(i)}
:=
(C_t^{(i)})^{-1}
=
(P_t^{(i)})^{-1}
+
\eta_\delta
(G_t^{(i)})^\top R_t^{-1}G_t^{(i)}
\]
and
\[
h_t^{(i)}
:=
J_t^{(i)}m_t^{(i)}
=
(P_t^{(i)})^{-1}a_t^{(i)}
+
\eta_\delta
(G_t^{(i)})^\top R_t^{-1}r_t^{(i)}.
\]
Equivalently,
\[
m_t^{(i)}
=
\argmin_u
\left\{
\frac{\eta_\delta}{2}
\|r_t^{(i)}-G_t^{(i)}u\|^2_{R_t^{-1}}
+
\frac12
\|u-a_t^{(i)}\|^2_{(P_t^{(i)})^{-1}}
\right\}.
\]
This is exactly the Gaussian update and proximal characterization stated in
Proposition~\ref{prop:recursive_delta_update}. The shared-discrepancy version
used in the main experiments is obtained by replacing $r_t^{(i)}$ with the
shared residual batch $r_t$ and suppressing the particle index.
\end{proof}

\begin{crblock}
\subsection{Discrepancy state representations for BRPC-E, BRPC-F, and BRPC-P}
\label{app:brpc_state_maps}

All three recursive representations use the Gaussian--Markov model
\eqref{eq:discrepancy_dynamics} and the forecast/update sequence
\eqref{eq:delta_state_forecast}--\eqref{eq:delta_natparam_update}.
They differ in the support and in how Gaussian information is stored.
$Q_{\delta,t}$ always denotes temporal innovation evaluated on $S_t$;
$\Omega$ denotes spatial conditional covariance; and
$\mathcal Q_{\delta,t}$ denotes total covariance in a finite-state transition.
These covariances have different roles and are not interchangeable.

\paragraph{Spatial covariance used for support extension.}
Let $K^0_{\delta,t}(A,B)$ denote the marginal GP covariance before any
residual observations in an expert's segment. Starting from the restart
kernel, it obeys
\begin{equation}
K^0_{\delta,0}(A,B)=K_{AB},\qquad
K^0_{\delta,t}(A,B)=\rho_\delta^2K^0_{\delta,t-1}(A,B)+q_\delta K_{AB}.
\label{eq:delta_marginal_kernel}
\end{equation}
Thus $K_{AB}$ specifies the reference spatial covariance. Under the unified model \eqref{eq:discrepancy_dynamics},
\[
K^0_{\delta,t}(A,B)=s_{\delta,t}K_{AB},\qquad
s_{\delta,0}=1,\quad s_{\delta,t}=\rho_\delta^2s_{\delta,t-1}+q_\delta.
\]
The random-walk special case has $s_{\delta,t}=1+tq_\delta$;
stationary forgetting with $q_\delta=1-\rho_\delta^2$ has
$s_{\delta,t}=1$. In these covariance formulas, $t$ counts temporal
transitions since initialization; a fresh expert uses $s_{\delta,0}=1$.
One may also use the reference conditional $s_{\delta,t}=1$ as a finite-state
approximation with a fixed spatial scale. The functional random-walk model
instead uses the growing scale $s_{\delta,t}=1+tq_\delta$. Both versions
propagate the state with $\rho_\delta$ and $Q_{\delta,t}$.

\paragraph{BRPC-E: expanding support.}
Suppose
\begin{equation}
q_{t-1}^{\delta,(i)}\{\delta_{t-1}(S_{t-1})\}
=\mathcal N(m_{t-1}^{(i)},C_{t-1}^{(i)}),
\label{eq:brpce_prev_posterior}
\end{equation}
and enlarge the support to
\begin{equation}
S_t=S_{t-1}\cup X_t,\label{eq:brpce_expanding_support}
\end{equation}
with
\begin{equation}
u_t^{(i)}=\delta_t(S_t).\label{eq:brpce_state}
\end{equation}
First extend the previous-time state spatially, then propagate it temporally.
The spatial mean map and conditional covariance are
\begin{equation}
H_t=K_{S_tS_{t-1}}K_{S_{t-1}S_{t-1}}^{-1},\qquad A_t=\rho_\delta H_t,
\label{eq:brpce_mean_map}
\end{equation}
\begin{equation}
\Omega_t^{\mathrm E}
=s_{\delta,t-1}\bigl(K_{S_tS_t}
-K_{S_tS_{t-1}}K_{S_{t-1}S_{t-1}}^{-1}K_{S_{t-1}S_t}\bigr).
\label{eq:brpce_conditional_covariance}
\end{equation}
Marginalizing over the previous Gaussian posterior gives
\begin{equation}
q_{t,-}^{\delta,(i)}(u_t^{(i)})=\mathcal N(a_t^{(i)},P_t^{(i)}),
\label{eq:brpce_preupdate_distribution}
\end{equation}
\begin{equation}
a_t^{(i)}=\rho_\delta H_tm_{t-1}^{(i)}=A_tm_{t-1}^{(i)},
\label{eq:brpce_preupdate_mean}
\end{equation}
\begin{equation}
\begin{aligned}
Q_{\delta,t}&=q_\delta K_{S_tS_t},\qquad
\mathcal Q_{\delta,t}=\rho_\delta^2\Omega_t^{\mathrm E}+Q_{\delta,t},\\
P_t^{(i)}&=\rho_\delta^2H_tC_{t-1}^{(i)}H_t^\top
+\rho_\delta^2\Omega_t^{\mathrm E}+Q_{\delta,t}.
\end{aligned}
\label{eq:brpce_preupdate_covariance}
\end{equation}
$\Omega_t^{\mathrm E}$ represents uncertainty at added locations at the
previous time; $Q_{\delta,t}$ represents fresh temporal innovation at the
current time. On retained locations the extension covariance is zero, but
temporal innovation need not be. The carry-over ablation is the limiting
configuration $\rho_\delta=1$, $q_\delta=0$, which retains spatial extension
uncertainty and has no temporal innovation.

Since $S_t$ contains $X_t$, let $G_t^{\mathrm E}$ select the current batch
coordinates. The residual model is
\begin{equation}
r_t^{(i)}\mid u_t^{(i)}\sim\mathcal N(G_t^{\mathrm E}u_t^{(i)},R_t).
\label{eq:brpce_residual_model}
\end{equation}
Within the stated Gaussian model, E retains the full support accumulated in
the segment. Its covariance dimension therefore grows with segment length;
numerical regularization or support reduction may still be needed.

\paragraph{BRPC-F: fixed support.}
Choose
\begin{equation}
S_t\equiv Z=\{z_1,\ldots,z_M\},\label{eq:brpcf_fixed_support}
\end{equation}
and define
\begin{equation}
u_t^{(i)}=\delta_t(Z).\label{eq:brpcf_state}
\end{equation}
There is no support extension: $H_t=I$, $\Omega_t^{\mathrm E}=0$, and
$A_t=\rho_\delta I$. Consequently,
\begin{equation}
\begin{aligned}
q_{t,-}^{\delta,(i)}&=\mathcal N(a_t^{(i)},P_t^{(i)}),\\
a_t^{(i)}&=\rho_\delta m_{t-1}^{(i)},\qquad
P_t^{(i)}=\rho_\delta^2C_{t-1}^{(i)}+Q_{\delta,t},\quad
Q_{\delta,t}=q_\delta K_{ZZ}.
\end{aligned}
\label{eq:brpcf_preupdate_distribution}
\end{equation}
The identity transition $A_t=I$ applies only when $\rho_\delta=1$.
For current batch inputs, the spatial map is
\begin{equation}
G_t^{\mathrm F}=K_{X_tZ}K_{ZZ}^{-1},\label{eq:brpcf_residual_map}
\end{equation}
so the finite-support GP representation uses
\begin{equation}
\mathbb E\{\delta_t(X_t)\mid u_t^{(i)}\}=G_t^{\mathrm F}u_t^{(i)},
\label{eq:brpcf_conditional_mean}
\end{equation}
with spatial residual covariance
\begin{equation}
\Omega_t^{\mathrm F}
=s_{\delta,t}\bigl(K_{X_tX_t}-K_{X_tZ}K_{ZZ}^{-1}K_{ZX_t}\bigr).
\label{eq:brpcf_conditional_covariance}
\end{equation}
Writing the observation covariance as $\Sigma_{\epsilon,t}$, the model with
explicit spatial residual is
\begin{equation}
r_t^{(i)}\mid u_t^{(i)}
\sim\mathcal N(G_t^{\mathrm F}u_t^{(i)},\Sigma_{\epsilon,t}+\Omega_t^{\mathrm F}).
\label{eq:brpcf_residual_model_explicit}
\end{equation}
Set $R_t=\Sigma_{\epsilon,t}+\Omega_t^{\mathrm F}$, including any stated
numerical nugget, to recover the generic likelihood
\begin{equation}
r_t^{(i)}\mid u_t^{(i)}\sim\mathcal N(G_t^{\mathrm F}u_t^{(i)},R_t).
\label{eq:brpcf_residual_model}
\end{equation}
This is the finite-support approximation to the functional GP. A deterministic
interpolant omits $\Omega_t^{\mathrm F}$ and must be distinguished from the
version retaining this variance. In both versions, $Q_{\delta,t}$ enters the
state forecast and is not an extra likelihood-noise term.

\paragraph{BRPC-P: proxy observations.}
P stores posterior information as pseudo-data on a chosen support.
Forecasting reconstructs that posterior and applies the temporal transition.
Suppose
\begin{equation}
q_t^{\delta,(i)}\{\delta_t(S_t)\}=\mathcal N(m_t^{(i)},C_t^{(i)}).
\label{eq:brpcp_carried_posterior}
\end{equation}
Use the current marginal reference covariance
$\mathcal K_{\delta,t}^{0}:=K^0_{\delta,t}(S_t,S_t)$, so that
\begin{equation}
p_t^0\{\delta_t(S_t)\}=\mathcal N(0,\mathcal K_{\delta,t}^{0}).
\label{eq:brpcp_prior_on_support}
\end{equation}
This reference need not equal the restart covariance $K_{S_tS_t}$.
An ideal Gaussian proxy likelihood is
\begin{equation}
\widetilde y_t^{(i)}\mid\delta_t(S_t)
\sim\mathcal N(\delta_t(S_t),\Lambda_t^{(i)}),
\label{eq:brpcp_proxy_likelihood}
\end{equation}
with natural parameters
\begin{equation}
(\Lambda_t^{(i)})^{-1}=(C_t^{(i)})^{-1}-(\mathcal K_{\delta,t}^{0})^{-1},
\label{eq:brpcp_proxy_precision}
\end{equation}
\begin{equation}
(\Lambda_t^{(i)})^{-1}\widetilde y_t^{(i)}
=(C_t^{(i)})^{-1}m_t^{(i)}.
\label{eq:brpcp_proxy_information}
\end{equation}
An ordinary nonsingular Gaussian proxy requires
\begin{equation}
(C_t^{(i)})^{-1}-(\mathcal K_{\delta,t}^{0})^{-1}\succ0.
\label{eq:brpcp_positive_proxy_precision}
\end{equation}
Positive semidefinite differences give information factors on their supported
subspace. An indefinite difference requires changing the reference or
approximating the proxy precision, for example by eigenvalue clipping.
For the next batch, reconstruct the stored Gaussian posterior and apply
\eqref{eq:delta_state_forecast} once.

\paragraph{Relationship among the representations.}
\begin{proposition}[Relationship among BRPC-E, BRPC-P, and BRPC-F]
\label{prop:brpc_variant_relationship}
On a common support, E and an admissible ideal P representation produce the
same Gaussian update if they use the same reference covariance, temporal
transition, spatial extension, and residual likelihood. F uses the same
forecast and KL update on fixed support, with the additional finite-support
approximation described above.
\end{proposition}
\begin{proof}
Fix a particle and support $S$ at time $t$, and write $u_S=\delta_t(S)$.
Combining the proxy with $\mathcal N(0,\mathcal K_{\delta,t}^{0})$ gives precision and information
vector
\[
(\mathcal K_{\delta,t}^{0})^{-1}+(\Lambda_t^{(i)})^{-1}=(C_t^{(i)})^{-1},\qquad
(\Lambda_t^{(i)})^{-1}\widetilde y_t^{(i)}=(C_t^{(i)})^{-1}m_t^{(i)}.
\]
Thus it reconstructs the same $\mathcal N(m_t^{(i)},C_t^{(i)})$.
Applying the same Gaussian--Markov transition gives identical $(a_{t+1},P_{t+1})$,
and Proposition~\ref{prop:recursive_delta_update} then gives identical updated
means and covariances. F follows that proposition with
$A_t=\rho_\delta I$, $Q_{\delta,t}=q_\delta K_{ZZ}$, and
$G_t=K_{X_tZ}K_{ZZ}^{-1}$. Changing the proxy support, reference
hyperparameters, or precision approximation need not preserve equivalence.
\end{proof}
\end{crblock}

\subsection{Proof of Theorem~\ref{thm:brpc_tracking}}
\label{app:proof_brpc_dynamic_regret}

This subsection proves the tracking bound for the recursive discrepancy mean
and records a sufficient condition for
Assumption~\ref{ass:brpc_transport_contraction}. As in the main text, we
suppress the particle index. The same argument applies to either a
particle-specific discrepancy posterior or the shared discrepancy posterior used
in the main experiments.

Let $u_t$ denote the finite-dimensional discrepancy state used by the chosen
representation, and let
\[
q_{t,-}^{\delta}(u_t)
=
\mathcal N(a_t,P_t),
\qquad
M_t=P_t^{-1},
\]
be the pre-update discrepancy distribution before assimilating batch $t$.
The residual model is
\[
r_t\mid u_t
\sim
\mathcal N(G_tu_t,R_t),
\]
and we define the batch residual loss
\[
\ell_t(u)
=
\frac12
\|r_t-G_tu\|^2_{R_t^{-1}}.
\]
The mean-form BRPC discrepancy update is therefore
\begin{equation}
m_t
=
\argmin_u
\left\{
\eta_\delta \ell_t(u)
+
\frac12
\|u-a_t\|^2_{M_t}
\right\}.
\label{eq:theory_brpc_mean_update}
\end{equation}
For the discrepancy representations considered in this paper, the pre-update
mean is obtained by propagating the previous posterior mean,
\[
a_t=A_tm_{t-1},
\]
\crnew{where $A_t=\rho_\delta H_t$ includes the temporal and spatial maps in \eqref{eq:delta_state_forecast}.}

\begin{crblock}
\paragraph{Fixed-support example.}
For BRPC-F, $u_t=\delta_t(Z)\in\mathbb R^M$. The fixed-support forecast is
\[
A_t=\rho_\delta I,\qquad a_t=\rho_\delta m_{t-1},\qquad
P_t=\rho_\delta^2C_{t-1}+Q_{\delta,t},\quad Q_{\delta,t}=q_\delta K_{ZZ}.
\]
The residual map is $G_t=K_{X_tZ}K_{ZZ}^{-1}$, with any retained spatial
variance absorbed into $R_t$. The reference-path term is therefore
$\sum_t\|v_t-\rho_\delta v_{t-1}\|_{M_t}^2$.
Only the special case $\rho_\delta=1$ gives ordinary increments
$v_t-v_{t-1}$; only $\rho_\delta=1$, $Q_{\delta,t}=0$ gives carry-over.
For changing supports, $A_t=\rho_\delta H_t$ also transports the spatial state.
\end{crblock}

\begin{proof}
The first-order optimality condition for
\eqref{eq:theory_brpc_mean_update} is
\begin{equation}
\eta_\delta\nabla\ell_t(m_t)+M_t(m_t-a_t)=0.
\label{eq:app_foc}
\end{equation}
Since $\ell_t$ is convex,
\begin{equation}
\ell_t(m_t)-\ell_t(v_t)
\le
\left\langle
\nabla\ell_t(m_t),m_t-v_t
\right\rangle .
\label{eq:app_convexity}
\end{equation}
Multiplying by $\eta_\delta$ and using \eqref{eq:app_foc} gives
\begin{equation}
\eta_\delta
\{\ell_t(m_t)-\ell_t(v_t)\}
\le
\left\langle
M_t(m_t-a_t),v_t-m_t
\right\rangle .
\label{eq:app_start}
\end{equation}
Using the weighted three-point identity
\[
2\langle a-b,c-a\rangle_M
=
\|c-b\|_M^2
-
\|c-a\|_M^2
-
\|a-b\|_M^2,
\]
with $a=m_t$, $b=a_t$, $c=v_t$, and $M=M_t$, we obtain
\begin{equation}
\eta_\delta
\{\ell_t(m_t)-\ell_t(v_t)\}
\le
\frac12\|v_t-a_t\|^2_{M_t}
-
\frac12\|v_t-m_t\|^2_{M_t}
-
\frac12\|m_t-a_t\|^2_{M_t}.
\label{eq:app_threepoint}
\end{equation}
Since $a_t=A_tm_{t-1}$,
\[
v_t-a_t
=
(v_t-A_tv_{t-1})
+
A_t(v_{t-1}-m_{t-1}).
\]
For any $\beta>0$,
\begin{align}
\|v_t-a_t\|^2_{M_t}
&\le
(1+\beta)
\|v_t-A_tv_{t-1}\|^2_{M_t}
+
\left(1+\frac1\beta\right)
\|A_t(v_{t-1}-m_{t-1})\|^2_{M_t}
\nonumber\\
&\le
(1+\beta)
\|v_t-A_tv_{t-1}\|^2_{M_t}
+
\left(1+\frac1\beta\right)\gamma
\|v_{t-1}-m_{t-1}\|^2_{M_{t-1}},
\label{eq:app_young}
\end{align}
where the second inequality uses
Assumption~\ref{ass:brpc_transport_contraction}. \crnew{If $\gamma=0$, positive definiteness of $M_t$ implies $A_t=0$,
and \eqref{eq:app_onestep} follows directly. For $0<\gamma<1$, choose}
\[
\beta=\frac{\gamma}{1-\gamma}.
\]
Then
\[
1+\beta=\frac{1}{1-\gamma},
\qquad
\left(1+\frac1\beta\right)\gamma=1.
\]
Substituting this choice into \eqref{eq:app_threepoint} yields
\begin{equation}
\eta_\delta
\{\ell_t(m_t)-\ell_t(v_t)\}
\le
\frac{1}{2(1-\gamma)}
\|v_t-A_tv_{t-1}\|^2_{M_t}
+
\frac12
\|v_{t-1}-m_{t-1}\|^2_{M_{t-1}}
-
\frac12
\|v_t-m_t\|^2_{M_t}.
\label{eq:app_onestep}
\end{equation}
Summing \eqref{eq:app_onestep} over $t=1,\ldots,T$ telescopes and gives
\[
\eta_\delta
\sum_{t=1}^T
\{\ell_t(m_t)-\ell_t(v_t)\}
\le
\frac12
\|v_0-m_0\|^2_{M_0}
+
\frac{1}{2(1-\gamma)}
\sum_{t=1}^T
\|v_t-A_tv_{t-1}\|^2_{M_t}.
\]
Dividing both sides by $\eta_\delta$ proves
\eqref{eq:brpc_tracking_bound}.
\end{proof}

We next discuss Assumption~\ref{ass:brpc_transport_contraction} from the
perspective of streaming industrial calibration. \crnew{Operating variability, finite-batch noise, changing input coverage,
and degradation all affect the usefulness of past discrepancy information.
Assumption~\ref{ass:brpc_transport_contraction} requires its propagated
precision to contract relative to the previous pre-update precision.}

\crnew{This is a condition on the algorithm's covariance recursion.
It can be checked for the chosen transition, spatial extension, and innovation
covariance. Numerical nuggets regularize matrix solves; temporal innovation
models uncertainty accumulated between batches.}
\begin{crblock}
\begin{proposition}[Sufficient condition for
Assumption~\ref{ass:brpc_transport_contraction}]
\label{prop:app_transport_sufficient}
Suppose the pre-update covariance is propagated as
\[
P_t=A_tC_{t-1}A_t^\top+\mathcal Q_{\delta,t},
\qquad
\mathcal Q_{\delta,t}\succeq q_{\min}I,
\]
and assume
\[
A_t^\top A_t\preceq a_{\max}I,
\qquad
P_{t-1}\preceq p_{\max}I.
\]
If
\[
\gamma=\frac{a_{\max}p_{\max}}{q_{\min}}<1,
\]
then
\[
A_t^\top M_tA_t\preceq \gamma M_{t-1},
\qquad
M_t=P_t^{-1}.
\]
Hence Assumption~\ref{ass:brpc_transport_contraction} holds.
\end{proposition}

\begin{proof}
Since $P_t\succeq\mathcal Q_{\delta,t}\succeq q_{\min}I$,
\[
M_t=P_t^{-1}\preceq q_{\min}^{-1}I.
\]
Therefore
\[
A_t^\top M_tA_t
\preceq
q_{\min}^{-1}A_t^\top A_t
\preceq
\frac{a_{\max}}{q_{\min}}I.
\]
Since $P_{t-1}\preceq p_{\max}I$,
\[
M_{t-1}=P_{t-1}^{-1}\succeq p_{\max}^{-1}I.
\]
Thus
\[
\frac{a_{\max}}{q_{\min}}I
\preceq
\frac{a_{\max}p_{\max}}{q_{\min}}M_{t-1}
=
\gamma M_{t-1}.
\]
Combining the inequalities gives
\[
A_t^\top M_tA_t\preceq \gamma M_{t-1}.
\]
\end{proof}
\end{crblock}

\begin{crblock}
\paragraph{Interpretation.}
In this proposition, $\mathcal Q_{\delta,t}$ is the total transition covariance
from \eqref{eq:delta_state_forecast}. On fixed support it is $Q_{\delta,t}$;
on expanding support it also includes $\rho_\delta^2\Omega_t^{\mathrm E}$.
For the kernel-scaled choices, fixed-support propagation is
\[
P_t=\rho_\delta^2C_{t-1}+q_\delta K_{ZZ}.
\]
The proposition requires the quantitative bounds stated above, not merely
$q_\delta>0$. In particular, the posterior-to-prior contraction in
Lemma~\ref{lem:cr_gm_contraction} uses $J_{t-1}=C_{t-1}^{-1}$;
Assumption~\ref{ass:brpc_transport_contraction} uses
$M_{t-1}=P_{t-1}^{-1}$. These are different stability statements. The Euclidean eigenvalue bounds
in Proposition~\ref{prop:app_transport_sufficient} are conservative,
especially for expanding supports. The direct metric diagnostic in
Appendix~\ref{app:transport_replay} evaluates the required inequality for
the propagated covariances.
\end{crblock}

\begin{crblock}
\subsection{Predictive-law details}
\label{app:predictive_details}

Prediction at time $t$ after assimilating $Y_t$ uses $(m_t,C_t)$.
Evidence for the incoming batch $t$ uses $(a_t,P_t)$ from
\eqref{eq:delta_state_forecast}, without conditioning on $Y_t$.
We distinguish the two explicitly below. Expert indices may be omitted for a
single BRPC run.

\paragraph{Particle-specific posterior prediction.}
Let $u_{e,t}^{(i)}=\delta_t(S_{e,t}^{(i)})$ have posterior
$\mathcal N(m_{e,t}^{(i)},C_{e,t}^{(i)})$. For new inputs $X_\star$ at the
same time, define
\[
G_{\star,e,t}^{(i)}=K_{X_\star S_{e,t}^{(i)}}
K_{S_{e,t}^{(i)}S_{e,t}^{(i)}}^{-1},
\]
\[
\Omega_{\star,e,t}^{(i)}=s_{\delta,e,t}\bigl(
K_{X_\star X_\star}-K_{X_\star S_{e,t}^{(i)}}
K_{S_{e,t}^{(i)}S_{e,t}^{(i)}}^{-1}K_{S_{e,t}^{(i)}X_\star}\bigr).
\]
Here $s_{\delta,e,t}$ follows \eqref{eq:delta_marginal_kernel} from the
expert's start. Under the finite-support GP representation,
\begin{equation}
\begin{aligned}
\mu_{\star,e,t}^{(i)}&=y_s(X_\star,\theta_{e,t}^{(i)})
+G_{\star,e,t}^{(i)}m_{e,t}^{(i)},\\
\Sigma_{\star,e,t}^{(i)}&=\Sigma_s(X_\star,\theta_{e,t}^{(i)})
+\Omega_{\star,e,t}^{(i)}
+G_{\star,e,t}^{(i)}C_{e,t}^{(i)}(G_{\star,e,t}^{(i)})^\top
+\Sigma_{\epsilon,\star,t}.
\end{aligned}
\label{eq:particle_predictive_y}
\end{equation}
$\Sigma_s$ is zero for a deterministic simulator.
$\Omega_{\star,e,t}^{(i)}$ is spatial conditional uncertainty; it is not
temporal innovation $Q_{\delta,t}$. E retains the current inputs in its
support; F uses a fixed-support approximation and may retain or omit the
spatial residual as specified in Appendix~\ref{app:brpc_state_maps}.
The posterior predictive mixture is
\[
p_{e,t}^{\mathrm{post}}(Y_\star\mid X_\star)
\approx\sum_iw_{e,t}^{(i)}
\mathcal N(Y_\star;\mu_{\star,e,t}^{(i)},\Sigma_{\star,e,t}^{(i)}).
\]

\paragraph{Pre-update evidence for a continuing expert.}
Propagate parameter particles to $\theta_{e,t\mid t-1}^{(i)}$ with weights
$w_{e,t\mid t-1}^{(i)}$, and form the discrepancy forecast
$\mathcal N(a_{e,t}^{(i)},P_{e,t}^{(i)})$ on the support used at batch $t$.
For $X_t$, let $G_{e,t}^{(i)}$ and $\Omega_{e,t}^{\mathrm{pre},(i)}$ be
the corresponding spatial map and residual covariance. Then
\begin{equation}
\begin{aligned}
p_e^{\mathrm{pre}}(Y_t\mid X_t)
&\approx\sum_iw_{e,t\mid t-1}^{(i)}
\mathcal N(Y_t;\mu_{e,t}^{\mathrm{pre},(i)},\Sigma_{e,t}^{\mathrm{pre},(i)}),\\
\mu_{e,t}^{\mathrm{pre},(i)}
&=y_s(X_t,\theta_{e,t\mid t-1}^{(i)})+G_{e,t}^{(i)}a_{e,t}^{(i)},\\
\Sigma_{e,t}^{\mathrm{pre},(i)}
&=\Sigma_s(X_t,\theta_{e,t\mid t-1}^{(i)})+\Sigma_{\epsilon,t}
+\Omega_{e,t}^{\mathrm{pre},(i)}
+G_{e,t}^{(i)}P_{e,t}^{(i)}(G_{e,t}^{(i)})^\top.
\end{aligned}
\label{eq:particle_mixture_predictive}
\end{equation}
For E, the support contains $X_t$, so the spatial residual at $X_t$ is zero;
its extension uncertainty and temporal innovation are already in $P_{e,t}^{(i)}$.
For F, the spatial residual is \eqref{eq:brpcf_conditional_covariance} if
retained. Neither covariance is added twice. The parameter draws and discrepancy
forecast may be reused for assimilation; scoring does not add a second
transition or condition on the incoming responses.

\paragraph{Shared discrepancy prediction.}
The shared approximation has one posterior $(m_{e,t}^{\mathrm{sh}},
C_{e,t}^{\mathrm{sh}})$ and one forecast $(a_{e,t}^{\mathrm{sh}},
P_{e,t}^{\mathrm{sh}})$ per expert. For pre-update evidence,
\begin{equation}
\begin{aligned}
p_e^{\mathrm{pre}}(Y_t\mid X_t)
\approx\sum_iw_{e,t\mid t-1}^{(i)}\mathcal N\bigl(
Y_t;&\ y_s(X_t,\theta_{e,t\mid t-1}^{(i)})+G_{e,t}^{\mathrm{sh}}a_{e,t}^{\mathrm{sh}},\\
&\Sigma_s(X_t,\theta_{e,t\mid t-1}^{(i)})+\Sigma_{\epsilon,t}
+\Omega_{e,t}^{\mathrm{pre,sh}}
+G_{e,t}^{\mathrm{sh}}P_{e,t}^{\mathrm{sh}}(G_{e,t}^{\mathrm{sh}})^\top\bigr).
\end{aligned}
\label{eq:shared_predictive}
\end{equation}
Same-time posterior prediction replaces the forecast quantities by
$(m_{e,t}^{\mathrm{sh}},C_{e,t}^{\mathrm{sh}})$ and the parameter forecast
by the updated particles. A fresh expert uses its restart prior directly, as
in \eqref{eq:pre_new}. RRA instead uses the segment refit based on batches
through $t-1$ when scoring batch $t$.
\end{crblock}

\section{Nonstationarity Definitions and Restart Proofs}
\label{app:compatibility_theory}

This appendix supports the restart discussion in Section~\ref{sec:reset_control}.
Appendix~\ref{app:drift_regime_definitions} clarifies the operational distinction
between gradual drift and abrupt regime change. 
Appendix~\ref{app:proof_gaussian_restart_odds} proves the Gaussian restart-odds
calculation in Proposition~\ref{prop:gaussian_restart_odds}. 
Appendix~\ref{app:proof_wcusum_bounds} proves the wCUSUM false-alarm and detection
bound in Proposition~\ref{prop:wcusum_bounds}. 
Appendix~\ref{app:bocpd_posterior_odds} gives an optional posterior-odds
interpretation of the BOCPD-style restart rule.

\subsection{Operational definitions of drift and regime change}
\label{app:drift_regime_definitions}

This subsection clarifies the distinction between gradual drift and abrupt
regime change used in the problem setup, restart discussion, and experimental
metrics. The distinction is operational: it describes whether information from
recent past batches remains useful for the current local calibration problem.

Let $g_t$ denote the local calibration object at batch $t$. In the simplest
case, $g_t$ is the projected calibration target $\theta_t^\dagger$. More
generally, it may include both the projected calibration target and the residual
target used by the discrepancy update. Let $\mathcal A_t$ be the propagation map used to
carry the previous local object to batch $t$, and let $\|\cdot\|_{\mathcal G,t}$
be a problem-dependent local norm, such as an $M_t$-weighted state norm or an
$L^2(F_X)$ response norm.

\paragraph{Gradual drift.}
An interval $I=\{\tau,\ldots,T\}$ has gradual drift if the propagated one-step
variation
\[
\Delta^{\rm drift}_t
:=
\|g_t-\mathcal A_tg_{t-1}\|_{\mathcal G,t}
\]
remains small enough that recent past batches are still informative for the
current local calibration problem. \crnew{Over such an interval, the}
cumulative variation
\[
\sum_{t\in I}(\Delta^{\rm drift}_t)^2
\]
does not dominate the regularization and tracking terms in the recursive
updates. In this case, carrying forward recent parameter and discrepancy
information is appropriate.

\paragraph{Abrupt regime change.}
A time $\tau$ is an abrupt regime change if the local calibration object changes
sharply relative to the continuation model. This may occur through a jump in the
projected calibration target, a change in the residual target, a shift in
operating conditions, or a change in the simulator--system relationship. In the
same notation, this corresponds to a propagated variation
\[
\Delta^{\rm jump}_\tau
:=
\|g_\tau-\mathcal A_\tau g_{\tau-1}\|_{\mathcal G,\tau}
\]
that is large relative to the local tracking scale of the continuing BRPC
posterior. After such a change, information accumulated before $\tau$ may bias
the current calibration and discrepancy updates.

For interpretation, one may introduce a local tolerance sequence
$\{b_t\}$ and call a change at time $\tau$ restart-relevant when
\[
\Delta_\tau
:=
\|g_\tau-\mathcal A_\tau g_{\tau-1}\|_{\mathcal G,\tau}
>
b_\tau .
\]
\crnew{The tolerance $b_\tau$ represents the local variation scale that the
continuing BRPC posterior can absorb with limited bias.} For example, if the one-step drift fluctuations over a
gradually drifting segment are conditionally sub-Gaussian with scale
$\sigma_\Delta$, then a high-probability tolerance over a horizon of length $T$
may be chosen as
\[
b_t
=
\mu_t
+
\sigma_\Delta\sqrt{2\log(T/\alpha)} ,
\]
where $\mu_t$ denotes a typical within-segment drift size, $\sigma_\Delta$ is the sub-Gaussian scale of one-step drift fluctuations and $\alpha$ is a target false-alarm
level. \crnew{Alternatively, if the $d$-dimensional increment is centered Gaussian
with covariance $M_t^{-1}$, its squared $M_t$-norm has a chi-square
distribution and one may use}
\[
b_t^2=\chi^2_{d,1-\alpha}.
\]
\crnew{Here $d$ is the dimension of the local calibration object. These
tolerances describe the scale of restart-relevant changes. The implemented
controllers use predictive evidence, and the experimental event metrics use
the annotated changes in the generated streams.}

\paragraph{Mixed nonstationarity.}
A stream has mixed nonstationarity if it contains both gradually drifting
segments and a finite number of abrupt regime changes:
\[
\{1,\ldots,T\}
=
\bigcup_{s=1}^S I_s,
\]
where each segment $I_s$ is treated as gradually drifting and adjacent segments
are separated by abrupt regime changes.

This distinction also determines how restart accuracy is evaluated in the
experiments. Precision, recall, and delay are computed only with respect to
annotated abrupt regime changes. In purely drifting settings, there are no
annotated abrupt changes. In mixed settings, a restart caused by gradual drift
may improve local prediction, but it is not counted as a true positive for
abrupt-change detection.

\subsection{Proof of Proposition~\ref{prop:gaussian_restart_odds}}
\label{app:proof_gaussian_restart_odds}

We prove the Gaussian restart-odds calculation used in
Section~\ref{subsec:bbrpc}. Consider a continuation expert $e$ and a newly
initialized expert $e_{\mathrm{new}}$. Before assimilating batch $(X_t,Y_t)$,
\crnew{condition on the past and current inputs, and suppose their
pre-update predictive laws are}
\[
p_e^{\mathrm{pre}}(Y_t\mid X_t)
=
\mathcal N(Y_t;\mu_c,\Sigma_c),
\]
and
\[
p_{\mathrm{new}}^{\mathrm{pre}}(Y_t\mid X_t)
=
\mathcal N(Y_t;\mu_n,\Sigma_n).
\]
The BOCPD weight update gives
\[
\widetilde w_{e,t}
=
(1-h_t)w_{e,t-1}p_e^{\mathrm{pre}}(Y_t\mid X_t),
\]
and
\[
\widetilde w_{\mathrm{new},t}
=
h_t p_{\mathrm{new}}^{\mathrm{pre}}(Y_t\mid X_t).
\]
The one-step restart
log-odds is then
\[
\log
\frac{w_{\mathrm{new},t}}
     {w_{e,t}}
=
\log
\frac{h_t}{(1-h_t)w_{e,t-1}}
+
\log
\frac{
\mathcal N(Y_t;\mu_n,\Sigma_n)
}{
\mathcal N(Y_t;\mu_c,\Sigma_c)
}.
\]
If $Y_t\sim\mathcal N(\mu_\star,\Sigma_\star)$, then
\[
\begin{aligned}
\E\left[
\log
\frac{\widetilde w_{\mathrm{new},t}}
     {\widetilde w_{e,t}}
\right]
&=
\log\frac{h_t}{(1-h_t)w_{e,t-1}}
+
\frac12
\log\frac{\det\Sigma_c}{\det\Sigma_n} \\
&\quad
+
\frac12
\E\left[
(Y_t-\mu_c)^\top\Sigma_c^{-1}(Y_t-\mu_c)
-
(Y_t-\mu_n)^\top\Sigma_n^{-1}(Y_t-\mu_n)
\right].
\end{aligned}
\]
For any Gaussian random vector $Y\sim\mathcal N(\mu_\star,\Sigma_\star)$,
\[
\E\left[
(Y-\mu)^\top\Sigma^{-1}(Y-\mu)
\right]
=
\operatorname{tr}(\Sigma^{-1}\Sigma_\star)
+
\|\mu_\star-\mu\|_{\Sigma^{-1}}^2.
\]
Applying this identity to the continuation and newly initialized experts gives
\[
\begin{aligned}
\E\left[
\log
\frac{\widetilde w_{\mathrm{new},t}}
     {\widetilde w_{e,t}}
\right]
&=
\log\frac{h_t}{(1-h_t)w_{e,t-1}}
+
\frac12
\Bigg(
\log\frac{\det\Sigma_c}{\det\Sigma_n}
+
\operatorname{tr}
\{(\Sigma_c^{-1}-\Sigma_n^{-1})\Sigma_\star\} \\
&\qquad\qquad
+
\|\mu_\star-\mu_c\|_{\Sigma_c^{-1}}^2
-
\|\mu_\star-\mu_n\|_{\Sigma_n^{-1}}^2
\Bigg).
\end{aligned}
\]
This proves Proposition~\ref{prop:gaussian_restart_odds}.

\paragraph{Interpretation.}
The expectation separates the prior-odds term from the Gaussian predictive
contrast. A small continuation covariance $\Sigma_c$ can penalize ordinary
batch-to-batch variation and increase restart odds, producing false restarts.
Conversely, if the discrepancy update makes the continuation mean $\mu_c$ move
quickly toward the post-change response, the contrast
\(\|\mu_\star-\mu_c\|_{\Sigma_c^{-1}}^2
-
\|\mu_\star-\mu_n\|_{\Sigma_n^{-1}}^2\)
may shrink, delaying restart after a true regime change. 

\begin{crblock}
\subsection{Proof of Proposition~\ref{prop:wcusum_bounds}}
\label{app:proof_wcusum_bounds}

We analyze the first restart within a segment. Extend its score sequence
without resetting for this calculation; the first threshold crossing agrees
with the implemented rule. At time $t$, windows have lengths
$1\le m\le\min(W,t)$, measured from the start of monitoring after warm-up.
Write $\bar s^{\mathrm{std}}_{t-m+1:t}=m^{-1}\sum_{u=t-m+1}^t s_u^{\mathrm{std}}$.

\begin{assumption}[Null score concentration]
\label{ass:wcusum_null_concentration}
Within a segment with no abrupt change, for every available window and $a>0$,
\[
\mathbb P\left(\sqrt m\,\bar s^{\mathrm{std}}_{t-m+1:t}>a\right)
\le e^{-a^2/2}.
\]
\end{assumption}

\begin{assumption}[Post-change score concentration]
\label{ass:wcusum_postchange_concentration}
For a change at batch $\tau$, the first $m_\star\le W$ post-change scores
have average $\bar s=m_\star^{-1}\sum_{u=\tau}^{\tau+m_\star-1}s_u^{\mathrm{std}}$
with mean at least $\delta>\kappa$. For every $u>0$,
\[
\mathbb P\left(\sqrt{m_\star}(\bar s-\delta)\le-u\right)
\le e^{-u^2/2}.
\]
\end{assumption}

These assumptions concern the adaptive score process itself. Running
standardization controls scale; the displayed bounds additionally require
calibrated window tails and a sustained post-change shift. Gaussian log scores
are quadratic and generally have sub-exponential tails. The unit sub-Gaussian
form here is a conditional idealization; it requires a score construction
with the stated tail control.

\begin{proof}
For $h_{\mathrm C}>0$ and $\kappa\ge0$, a threshold crossing implies that some
available window satisfies
\[
\sqrt m\,\bar s^{\mathrm{std}}_{t-m+1:t}
>h_{\mathrm C}+\sqrt m\,\kappa.
\]
The null concentration bound and a union bound therefore give
\begin{align*}
\mathbb P\left(\max_{t\le T}G_t^{\mathrm{score}}>h_{\mathrm C}\right)
&\le\sum_{t=1}^T\sum_{m=1}^{\min(W,t)}
\exp\left\{-\tfrac12(h_{\mathrm C}+\sqrt m\,\kappa)^2\right\}\\
&\le TW e^{-h_{\mathrm C}^2/2}.
\end{align*}
Thus $h_{\mathrm C}\ge\sqrt{2\log(TW/\alpha)}$ controls the probability
of a false alarm within this monitoring horizon at level $\alpha$.

For detection, let $t_\star=\tau+m_\star-1$. If an alarm has not occurred
by $t_\star$, the entirely post-change window of length $m_\star$ satisfies
\[
\sqrt{m_\star}(\bar s-\kappa)_+
\le G_{t_\star}^{\mathrm{score}}\le h_{\mathrm C}.
\]
Consequently,
\[
\sqrt{m_\star}(\bar s-\delta)
\le h_{\mathrm C}-\sqrt{m_\star}(\delta-\kappa).
\]
For $\beta\in(0,1)$, the post-change concentration assumption bounds this
event by $\beta$ whenever
\[
m_\star\ge
\frac{(h_{\mathrm C}+\sqrt{2\log(1/\beta)})^2}{(\delta-\kappa)^2},
\qquad m_\star\le W.
\]
This proves both statements. The detection result concerns a shift that
persists over the required window; the empirical running standardization and
within-regime adaptation determine whether this condition holds.
\end{proof}
\end{crblock}

\begin{crblock}
\subsection{Posterior-odds interpretation of B-BRPC}
\label{app:bocpd_posterior_odds}

Fix an anchor expert $a$ with start time $s_a$ and a later candidate expert
born at batch $s>s_a$. Let $w_{s,t}$ and $w_{a,t}$ be their normalized
weights at time $t$, while both experts are retained. The hard-restart rule is
\[
\max_{s>s_a}\frac{w_{s,t}}{w_{a,t}}>\rho_{\mathrm B}.
\]
Write $q_{s,u}=p_s^{\mathrm{pre}}(Y_u\mid X_u)$ and
$q_{a,u}=p_a^{\mathrm{pre}}(Y_u\mid X_u)$. These densities are formed
before observing $Y_u$, as in Appendix~\ref{app:predictive_details}.
Here $\mathcal H_{u-1}$ denotes the pre-response history, including $X_u$
and forecast random draws; $Y_u$ is included before the next response is scored.

\begin{lemma}[Posterior odds as a windowed log Bayes factor]
\label{lem:posterior_odds_windowed_bf}
Under the common hazard $h_u$ in \eqref{eq:bbrpc_weight}, define the
candidate-birth log odds
\[
b_s:=\log\frac{h_s}{(1-h_s)w_{a,s-1}}.
\]
While candidate $s$ and anchor $a$ remain active,
\[
\log\frac{w_{s,t}}{w_{a,t}}
=b_s+\sum_{u=s}^t\log\frac{q_{s,u}}{q_{a,u}}.
\]
Thus candidate $s$ overtakes the anchor when its cumulative log predictive
advantage exceeds $\log\rho_{\mathrm B}-b_s$.
\end{lemma}

\begin{proof}
At candidate birth, the weight updates give
\[
\frac{w_{s,s}}{w_{a,s}}
=\frac{h_s q_{s,s}}{(1-h_s)w_{a,s-1}q_{a,s}}.
\]
At each later batch $u>s$, both experts are continuation branches, so their
common factor $1-h_u$ and the weight-normalization constant cancel:
\[
\frac{w_{s,u}}{w_{a,u}}
=\frac{w_{s,u-1}}{w_{a,u-1}}\frac{q_{s,u}}{q_{a,u}}.
\]
Iterating this identity proves the result. Pruning preserves the ratio when
both experts survive and their weights are renormalized together.
\end{proof}

The term $b_s$ retains the anchor's evidence accumulated before the candidate
was born. For expert-specific continuation hazards, the same calculation adds
$\sum_{u=s+1}^t\log\{(1-h_{s,u})/(1-h_{a,u})\}$, with the corresponding
birth-branch mass in $b_s$.

The following bounds concern a fixed candidate--anchor comparison, with time
reindexed from candidate birth. Conditioning on that birth history makes $b_s$
fixed. They apply to a positive threshold
$h=\log\rho_{\mathrm B}-b_s$; simultaneous testing across candidates requires
combining the corresponding crossing events.

Define the one-step log-likelihood-ratio increment
\[
\Delta\ell_t
:=
\log
\frac{
q_{\mathrm{new}}(B_t\mid\mathcal H_{t-1})
}{
q_{\mathrm{old}}(B_t\mid\mathcal H_{t-1})
},
\qquad
S_t:=\sum_{u=1}^t\Delta\ell_u,
\]
and let
\[
\tau_h:=\inf\{t\ge 1:S_t\ge h\}.
\]
Write
\[
\mu_t
:=
\E[\Delta\ell_t\mid\mathcal H_{t-1}],
\qquad
Z_t
:=
\Delta\ell_t-\mu_t.
\]

\begin{assumption}[Conditional sub-Gaussian fluctuations]
\label{ass:llr_subgaussian}
There exists $0<\iota_\ell^2<\infty$ such that, for all
$\lambda\in\mathbb R$,
\[
\E\!\left[
\exp(\lambda Z_t)\mid\mathcal H_{t-1}
\right]
\le
\exp\left(\frac{\lambda^2\iota_\ell^2}{2}\right).
\]
\end{assumption}

For conditionally Gaussian responses and Gaussian predictive laws with a common
covariance, the log-likelihood ratio is affine in the response, and a uniform
conditional variance bound provides this concentration condition. Unequal predictive covariances add a
quadratic term, for which sub-exponential concentration is the usual form.
Here the conditional sub-Gaussian condition specifies the class of increments
covered by the proposition.

\begin{proposition}[Finite-horizon restart testing bounds]
\label{prop:finite_horizon_restart_bounds}
Fix $h>0$ and let
\[
\Delta \ell_t = \mu_t + Z_t,
\qquad
\mu_t = \mathbb E[\Delta \ell_t \mid \mathcal H_{t-1}],
\qquad
\mathbb E[Z_t\mid \mathcal H_{t-1}]=0,
\]
and let
\[
S_t=\sum_{u=1}^t \Delta \ell_u .
\]
Suppose Assumption~\ref{ass:llr_subgaussian} holds; that is, for all
$\lambda\in\mathbb R$,
\[
\mathbb E\!\left[
\exp(\lambda Z_t)\mid \mathcal H_{t-1}
\right]
\le
\exp\!\left(\frac{\lambda^2\iota_\ell^2}{2}\right).
\]

\noindent
(a) If $\mu_t\le \bar\mu_0$ for all $t\le T$, with $\bar\mu_0\ge0$, then for any
$h>T\bar\mu_0$,
\[
\mathbb P\!\left(
\sup_{1\le t\le T}S_t\ge h
\right)
\le
\exp\left\{
-\frac{(h-T\bar\mu_0)^2}{2\iota_\ell^2T}
\right\}.
\]

\noindent
(b) Suppose a change occurs at time $\nu$ and $\mu_t\ge\mu_1>0$
on the first $m_{\mathrm{post}}$ post-change batches. Define the hitting delay
\[
D_h
:=
\inf\{m\ge 1:S_{\nu+m}-S_\nu\ge h\}.
\]
For $h/\mu_1<m\le m_{\mathrm{post}}$,
\[
\mathbb P(D_h>m\mid\mathcal H_\nu)
\le\exp\left\{-\frac{(m\mu_1-h)^2}{2\iota_\ell^2m}\right\}.
\]
If the mean lower bound persists for all $t>\nu$, then
\[
\mathbb E[D_h\mid \mathcal H_\nu]
\le
\left\lceil \frac{2h}{\mu_1}\right\rceil
+
\frac{1}{1-\exp\{-\mu_1^2/(8\iota_\ell^2)\}}.
\]
In particular,
\[
\mathbb E[D_h]
=
O\!\left(
1+\frac{h}{\mu_1}
+
\frac{\iota_\ell^2}{\mu_1^2}
\right).
\]
If the restart statistic is reset at time $\nu$, then $D_h$ coincides with
the usual detection delay $(\tau_h-\nu)_+$ for the reset statistic.
\end{proposition}

\begin{proof}
We first prove part (a). Define the martingale
\[
M_t := \sum_{u=1}^t Z_u .
\]
By Assumption~\ref{ass:llr_subgaussian}, for every fixed
$\lambda>0$,
\[
L_t(\lambda)
:=
\exp\left\{
\lambda M_t
-
\frac{\lambda^2\iota_\ell^2 t}{2}
\right\}
\]
is a nonnegative supermartingale with respect to
$\{\mathcal H_t\}_{t\ge 0}$. Indeed,
\[
\begin{aligned}
\mathbb E[L_t(\lambda)\mid \mathcal H_{t-1}]
&=
\exp\left\{
\lambda M_{t-1}
-
\frac{\lambda^2\iota_\ell^2 t}{2}
\right\}
\mathbb E\!\left[
\exp(\lambda Z_t)\mid \mathcal H_{t-1}
\right]
\\
&\le
\exp\left\{
\lambda M_{t-1}
-
\frac{\lambda^2\iota_\ell^2 t}{2}
\right\}
\exp\left\{
\frac{\lambda^2\iota_\ell^2}{2}
\right\}
\\
&=
\exp\left\{
\lambda M_{t-1}
-
\frac{\lambda^2\iota_\ell^2 (t-1)}{2}
\right\}
=
L_{t-1}(\lambda).
\end{aligned}
\]
Since $\mu_u\le \bar\mu_0$ and $\bar\mu_0\ge0$, for every $1\le t\le T$,
\[
S_t
=
\sum_{u=1}^t \mu_u
+
M_t
\le
T\bar\mu_0+M_t .
\]
Therefore, with $a:=h-T\bar\mu_0>0$,
\[
\left\{\sup_{1\le t\le T}S_t\ge h\right\}
\subseteq
\left\{
\sup_{1\le t\le T}M_t\ge a
\right\}.
\]
On the event $\{\sup_{1\le t\le T}M_t\ge a\}$, there exists
$t\le T$ such that
\[
L_t(\lambda)
=
\exp\left\{
\lambda M_t
-
\frac{\lambda^2\iota_\ell^2 t}{2}
\right\}
\ge
\exp\left\{
\lambda a
-
\frac{\lambda^2\iota_\ell^2 T}{2}
\right\}.
\]
Hence, by Ville's inequality for nonnegative supermartingales,
\[
\begin{aligned}
\mathbb P\!\left(
\sup_{1\le t\le T}M_t\ge a
\right)
&\le
\mathbb P\!\left(
\sup_{1\le t\le T}L_t(\lambda)
\ge
\exp\left\{
\lambda a
-
\frac{\lambda^2\iota_\ell^2 T}{2}
\right\}
\right)
\\
&\le
\exp\left\{
-\lambda a
+
\frac{\lambda^2\iota_\ell^2 T}{2}
\right\}.
\end{aligned}
\]
Optimizing the right-hand side over $\lambda>0$ gives
\[
\lambda^\star=\frac{a}{\iota_\ell^2 T}.
\]
Substituting this value yields
\[
\mathbb P\!\left(
\sup_{1\le t\le T}M_t\ge a
\right)
\le
\exp\left\{
-\frac{a^2}{2\iota_\ell^2T}
\right\}.
\]
Since $a=h-T\bar\mu_0$, part (a) follows.

We next prove part (b). Conditional on $\mathcal H_\nu$, define the
post-change centered martingale
\[
M_{\nu,m}
:=
\sum_{j=1}^m Z_{\nu+j},
\qquad
m\ge 1.
\]
For $1\le m\le m_{\mathrm{post}}$,
\[
S_{\nu+m}-S_\nu
=
\sum_{j=1}^m \mu_{\nu+j}
+
M_{\nu,m}
\ge
m\mu_1+M_{\nu,m}.
\]
Therefore, if $D_h>m$, then
\[
S_{\nu+m}-S_\nu<h,
\]
and hence
\[
M_{\nu,m}
<
h-m\mu_1.
\]
Thus, whenever $m\mu_1>h$,
\[
\mathbb P(D_h>m\mid \mathcal H_\nu)
\le
\mathbb P\!\left(
M_{\nu,m}
\le
-(m\mu_1-h)
\,\middle|\,
\mathcal H_\nu
\right).
\]
We now bound this lower tail. For any $\lambda>0$, the conditional
sub-Gaussian assumption applied to $-Z_t$ gives
\[
\mathbb E\!\left[
\exp(-\lambda Z_t)\mid \mathcal H_{t-1}
\right]
\le
\exp\left\{
\frac{\lambda^2\iota_\ell^2}{2}
\right\}.
\]
Iterating this bound over $j=1,\ldots,m$ gives
\[
\mathbb E\!\left[
\exp(-\lambda M_{\nu,m})
\,\middle|\,
\mathcal H_\nu
\right]
\le
\exp\left\{
\frac{\lambda^2\iota_\ell^2 m}{2}
\right\}.
\]
By Chernoff's inequality, for every $a>0$,
\[
\begin{aligned}
\mathbb P\!\left(
M_{\nu,m}\le -a
\,\middle|\,
\mathcal H_\nu
\right)
&=
\mathbb P\!\left(
\exp(-\lambda M_{\nu,m})\ge \exp(\lambda a)
\,\middle|\,
\mathcal H_\nu
\right)
\\
&\le
\exp(-\lambda a)
\mathbb E\!\left[
\exp(-\lambda M_{\nu,m})
\,\middle|\,
\mathcal H_\nu
\right]
\\
&\le
\exp\left\{
-\lambda a
+
\frac{\lambda^2\iota_\ell^2 m}{2}
\right\}.
\end{aligned}
\]
Optimizing over $\lambda>0$ gives
\[
\lambda^\star=\frac{a}{\iota_\ell^2 m},
\]
and therefore
\[
\mathbb P\!\left(
M_{\nu,m}\le -a
\,\middle|\,
\mathcal H_\nu
\right)
\le
\exp\left\{
-\frac{a^2}{2\iota_\ell^2 m}
\right\}.
\]
Taking $a=m\mu_1-h$ gives, for $h/\mu_1<m\le m_{\mathrm{post}}$,
\[
\mathbb P(D_h>m\mid \mathcal H_\nu)
\le
\exp\left\{
-\frac{(m\mu_1-h)^2}{2\iota_\ell^2 m}
\right\}.
\]
For the expected-delay bound, assume the mean lower bound persists
throughout the post-change sequence. Let
\[
m_0:=\left\lceil \frac{2h}{\mu_1}\right\rceil .
\]
For all $m\ge m_0$, we have
\[
m\mu_1-h\ge \frac{m\mu_1}{2}.
\]
Consequently,
\[
\mathbb P(D_h>m\mid \mathcal H_\nu)
\le
\exp\left\{
-\frac{\mu_1^2 m}{8\iota_\ell^2}
\right\},
\qquad m\ge m_0.
\]
Using the tail-sum formula for nonnegative integer-valued random variables,
\[
\mathbb E[D_h\mid \mathcal H_\nu]
=
\sum_{m=0}^\infty
\mathbb P(D_h>m\mid \mathcal H_\nu),
\]
we obtain
\[
\begin{aligned}
\mathbb E[D_h\mid \mathcal H_\nu]
&\le
m_0
+
\sum_{m=m_0}^\infty
\exp\left\{
-\frac{\mu_1^2 m}{8\iota_\ell^2}
\right\}
\\
&\le
m_0
+
\sum_{m=0}^\infty
\exp\left\{
-\frac{\mu_1^2 m}{8\iota_\ell^2}
\right\}
\\
&=
m_0
+
\frac{1}{
1-\exp\{-\mu_1^2/(8\iota_\ell^2)\}
}.
\end{aligned}
\]
Since $m_0=\lceil 2h/\mu_1\rceil$, this proves
\[
\mathbb E[D_h\mid \mathcal H_\nu]
\le
\left\lceil \frac{2h}{\mu_1}\right\rceil
+
\frac{1}{
1-\exp\{-\mu_1^2/(8\iota_\ell^2)\}
}.
\]
Finally, using the elementary bound
\[
\frac{1}{1-e^{-x}}
\le
1+\frac{1}{x},
\qquad x>0,
\]
with $x=\mu_1^2/(8\iota_\ell^2)$ gives
\[
\mathbb E[D_h\mid \mathcal H_\nu]
\le
\left\lceil \frac{2h}{\mu_1}\right\rceil
+
1
+
\frac{8\iota_\ell^2}{\mu_1^2}.
\]
Therefore
\[
\mathbb E[D_h]
=
O\!\left(
1+\frac{h}{\mu_1}
+
\frac{\iota_\ell^2}{\mu_1^2}
\right),
\]
which proves part (b).
\end{proof}

\paragraph{Interpretation.}
The conditional drift
\[
\mu_t
=
\E[\Delta\ell_t\mid\mathcal H_{t-1}]
\]
determines the behavior of the restart statistic. During gradual drift, restart
remains unlikely when candidate experts do not accumulate systematic predictive
advantage over the anchor, so $\mu_t$ is nonpositive or small. After an abrupt
change, restart is fast when a newly initialized expert has positive predictive
advantage before the continuation expert adapts to the new data. This explains
the evidence interaction discussed in the main text: a highly concentrated
continuation predictive distribution may produce false restarts, while a highly
adaptive discrepancy update may reduce the candidate--anchor contrast after a
true change.
\end{crblock}

\begin{crblock}
\section{Supplementary Theory of Parameter--Discrepancy Coupling}
\label{app:optional_restart_discrepancy_theory}

This appendix examines the interaction between projected parameter estimation
and discrepancy learning. We first describe confounding under joint assimilation
and quantify how parameter-anchor error perturbs the discrepancy residuals.
We then establish a local coupled tracking bound under Gaussian--Markov
propagation and report diagnostics of the propagation geometry used in
Theorem~\ref{thm:brpc_tracking}. The local coupled result uses stronger
assumptions than the residual-conditional discrepancy bound in that theorem.
\end{crblock}

\subsection{Identifiability issue in joint assimilation}
\label{app:joint_assimilation_identifiability}

This subsection gives a static calculation explaining why joint assimilation of
calibration parameters and discrepancy-like bias can blur attribution.

Consider the Gaussian regression model
\[
y
=
y_s(x,\theta)+\Phi_m(x)^\top\beta+\epsilon,
\qquad
\epsilon\sim\mathcal N(0,\sigma^2),
\]
where \(\Phi_m(x)=(\phi_1(x),\ldots,\phi_m(x))^\top\) spans a finite-dimensional
discrepancy class. Let \(\psi=(\theta,\beta)\) denote the joint parameter vector.
The Fisher information matrix for \(\psi\) is denoted by \(I(\psi)\). We
partition it according to the calibration parameter \(\theta\) and the
discrepancy coefficients \(\beta\):
\[
I(\psi)
=
\begin{pmatrix}
I_{\theta\theta} & I_{\theta\beta}\\
I_{\beta\theta} & I_{\beta\beta}
\end{pmatrix}.
\]
Here \(I_{\theta\theta}\) is the information for the calibration parameter
alone, \(I_{\beta\beta}\) is the information for the discrepancy coefficients,
and \(I_{\theta\beta}=I_{\beta\theta}^\top\) is the cross-information between
simulator-parameter sensitivity and the discrepancy basis.

\begin{lemma}[Structural confounding]
\label{lem:structural_confounding_joint}
If $I_{\beta\beta}\succ 0$, then the profiled Fisher information for $\theta$
after eliminating $\beta$ is
\[
I_m(\theta)
=
I_{\theta\theta}
-
I_{\theta\beta}I_{\beta\beta}^{-1}I_{\beta\theta}
\preceq
I_{\theta\theta}.
\]
Moreover, there is information loss in at least one direction whenever
$I_{\theta\beta}\ne 0$, equivalently whenever some simulator sensitivity
direction has nonzero projection onto the discrepancy span under the design
distribution.
\end{lemma}

\begin{proof}
Under the Gaussian model, the Fisher information is the Gram matrix of the
score functions for $(\theta,\beta)$. Profiling out $\beta$ gives the Schur
complement. Since
\[
I_{\theta\beta}I_{\beta\beta}^{-1}I_{\beta\theta}\succeq 0,
\]
the profiled information is no larger than $I_{\theta\theta}$ in Loewner order.
If $I_{\theta\beta}\ne 0$, then the positive-semidefinite correction is nonzero,
so at least one direction of $\theta$ loses information.
\end{proof}

\emph{Connection to joint filtering baselines.}
A joint filtering baseline updates calibration parameters and discrepancy-like
bias through the same innovation. Under simulator misspecification, the same
residual can be explained by parameter movement, discrepancy bias, observation
noise, or regime change. Lemma~\ref{lem:structural_confounding_joint} formalizes
the static version of this attribution problem. In the online setting, this can
appear as inaccurate parameter tracking together with weak hard-restart evidence:
the local filter may smooth part of an abrupt change into parameter motion before
the restart mechanism sees a clean predictive contrast.

\subsection{Effect of projected-parameter anchor error on discrepancy learning}
\label{app:theta_contamination}

The tracking bound in Theorem~\ref{thm:brpc_tracking} is conditional on the
residual sequence supplied to the discrepancy update. This subsection records
how error in the projected-parameter update perturbs that residual sequence.
For simplicity, we present the shared-discrepancy version; the particle-specific discrepancy
version follows by replacing the shared simulator anchor with the
particle-specific simulator output.

Let \(\theta_t^\star\) denote the oracle projected calibration target at batch
\(t\), and define the corresponding oracle simulator anchor
\[
\mu_t^\star(X_t)
:=
y_s(X_t,\theta_t^\star).
\]
The shared projected particle filter instead uses the posterior-weighted
simulator anchor
\[
\widehat\mu_t(X_t)
:=
\sum_{i=1}^N
w_t^{(i)}y_s(X_t,\theta_t^{(i)}).
\]
Define the anchor error
\[
\varepsilon_t^\theta
:=
\widehat\mu_t(X_t)-\mu_t^\star(X_t).
\]
The oracle residual batch and the BRPC shared residual batch are
\[
r_t^\star
:=
Y_t-\mu_t^\star(X_t),
\]
and
\[
r_t
:=
Y_t-\widehat\mu_t(X_t).
\]
Therefore,
\[
r_t
=
r_t^\star-\varepsilon_t^\theta .
\]

For any discrepancy state \(u\), define the residual loss
\[
\ell_t(u;r)
:=
\frac12\|r-G_tu\|^2_{R_t^{-1}},
\]
where \(G_t\) maps the discrepancy state to the current batch inputs, \(R_t\)
is the residual-likelihood covariance, and
\[
\|v\|_{R_t^{-1}}^2:=v^\top R_t^{-1}v.
\]
When the \(M_t\)-weighted norm appears below, \(M_t=P_t^{-1}\) denotes the
pre-update discrepancy precision used in Theorem~\ref{thm:brpc_tracking}.

\begin{proposition}[Residual-contamination identity]
\label{prop:exact_contamination_identity}
For every \(t\) and every discrepancy state \(u\),
\[
\ell_t(u;r_t)
=
\ell_t(u;r_t^\star)
-
\left\langle
\varepsilon_t^\theta,
r_t^\star-G_tu
\right\rangle_{R_t^{-1}}
+
\frac12
\|\varepsilon_t^\theta\|^2_{R_t^{-1}} .
\]
\end{proposition}

\begin{proof}
Substitute \(r_t=r_t^\star-\varepsilon_t^\theta\) into
\[
\ell_t(u;r_t)
=
\frac12\|r_t-G_tu\|^2_{R_t^{-1}}.
\]
Then
\[
\begin{aligned}
\ell_t(u;r_t)
&=
\frac12
\|r_t^\star-\varepsilon_t^\theta-G_tu\|^2_{R_t^{-1}}
\\
&=
\frac12
\|(r_t^\star-G_tu)-\varepsilon_t^\theta\|^2_{R_t^{-1}}
\\
&=
\frac12
\|r_t^\star-G_tu\|^2_{R_t^{-1}}
-
\left\langle
\varepsilon_t^\theta,
r_t^\star-G_tu
\right\rangle_{R_t^{-1}}
+
\frac12
\|\varepsilon_t^\theta\|^2_{R_t^{-1}}
\\
&=
\ell_t(u;r_t^\star)
-
\left\langle
\varepsilon_t^\theta,
r_t^\star-G_tu
\right\rangle_{R_t^{-1}}
+
\frac12
\|\varepsilon_t^\theta\|^2_{R_t^{-1}} .
\end{aligned}
\]
This proves the identity.
\end{proof}

\begin{corollary}[Pointwise perturbation bound]
\label{cor:pointwise_contamination_bound}
Suppose the projected-anchor error is bounded as
\[
\|\varepsilon_t^\theta\|_{R_t^{-1}}\le B_t^\theta
\]
for some \(B_t^\theta\ge 0\).
Then for every discrepancy state \(u\),
\[
\left|
\ell_t(u;r_t)-\ell_t(u;r_t^\star)
\right|
\le
B_t^\theta
\|r_t^\star-G_tu\|_{R_t^{-1}}
+
\frac12(B_t^\theta)^2 .
\]
\end{corollary}

\begin{proof}
By Proposition~\ref{prop:exact_contamination_identity},
\[
\ell_t(u;r_t)-\ell_t(u;r_t^\star)
=
-
\left\langle
\varepsilon_t^\theta,
r_t^\star-G_tu
\right\rangle_{R_t^{-1}}
+
\frac12
\|\varepsilon_t^\theta\|^2_{R_t^{-1}}.
\]
Taking absolute values and applying Cauchy--Schwarz gives
\[
\left|
\ell_t(u;r_t)-\ell_t(u;r_t^\star)
\right|
\le
\|\varepsilon_t^\theta\|_{R_t^{-1}}
\|r_t^\star-G_tu\|_{R_t^{-1}}
+
\frac12
\|\varepsilon_t^\theta\|^2_{R_t^{-1}}.
\]
Using \(\|\varepsilon_t^\theta\|_{R_t^{-1}}\le B_t^\theta\) proves the claim.
\end{proof}

\begin{proposition}[Regret perturbation identity]
\label{prop:regret_perturbation_identity}
For every pair of discrepancy states \((m,v)\),
\[
\bigl[
\ell_t(m;r_t)-\ell_t(v;r_t)
\bigr]
-
\bigl[
\ell_t(m;r_t^\star)-\ell_t(v;r_t^\star)
\bigr]
=
\left\langle
\varepsilon_t^\theta,
G_t(m-v)
\right\rangle_{R_t^{-1}} .
\]
\end{proposition}

\begin{proof}
Apply Proposition~\ref{prop:exact_contamination_identity} first with \(u=m\)
and then with \(u=v\). Subtracting the two identities cancels the quadratic
term \(\frac12\|\varepsilon_t^\theta\|^2_{R_t^{-1}}\) and gives
\[
\begin{aligned}
&
\bigl[
\ell_t(m;r_t)-\ell_t(v;r_t)
\bigr]
-
\bigl[
\ell_t(m;r_t^\star)-\ell_t(v;r_t^\star)
\bigr]
\\
&\qquad =
-
\left\langle
\varepsilon_t^\theta,
r_t^\star-G_tm
\right\rangle_{R_t^{-1}}
+
\left\langle
\varepsilon_t^\theta,
r_t^\star-G_tv
\right\rangle_{R_t^{-1}}
\\
&\qquad =
\left\langle
\varepsilon_t^\theta,
G_t(m-v)
\right\rangle_{R_t^{-1}} .
\end{aligned}
\]
This proves the identity.
\end{proof}

\begin{corollary}[Quantitative contamination penalty]
\label{cor:quantitative_contamination_penalty}
Assume
\[
\|\varepsilon_t^\theta\|_{R_t^{-1}}\le B_t^\theta,
\]
and suppose the residual map is compatible with the \(M_t\)-weighted
discrepancy norm in the sense that
\[
G_t^\top R_t^{-1}G_t\preceq \alpha_t M_t .
\]
Then for every pair \((m,v)\),
\[
\ell_t(m;r_t)-\ell_t(v;r_t)
\le
\ell_t(m;r_t^\star)-\ell_t(v;r_t^\star)
+
\frac{\alpha_t}{2}\|m-v\|_{M_t}^2
+
\frac12(B_t^\theta)^2 .
\]
\end{corollary}

\begin{proof}
By Proposition~\ref{prop:regret_perturbation_identity},
\[
\begin{aligned}
&
\bigl[
\ell_t(m;r_t)-\ell_t(v;r_t)
\bigr]
-
\bigl[
\ell_t(m;r_t^\star)-\ell_t(v;r_t^\star)
\bigr]
\\
&\qquad =
\left\langle
\varepsilon_t^\theta,
G_t(m-v)
\right\rangle_{R_t^{-1}} .
\end{aligned}
\]
Using Cauchy--Schwarz gives
\[
\left\langle
\varepsilon_t^\theta,
G_t(m-v)
\right\rangle_{R_t^{-1}}
\le
\|\varepsilon_t^\theta\|_{R_t^{-1}}
\|G_t(m-v)\|_{R_t^{-1}} .
\]
The anchor-error bound and the compatibility condition imply
\[
\|\varepsilon_t^\theta\|_{R_t^{-1}}
\|G_t(m-v)\|_{R_t^{-1}}
\le
B_t^\theta
\sqrt{\alpha_t}
\|m-v\|_{M_t}.
\]
By Young's inequality,
\[
B_t^\theta
\sqrt{\alpha_t}
\|m-v\|_{M_t}
\le
\frac12(B_t^\theta)^2
+
\frac{\alpha_t}{2}
\|m-v\|_{M_t}^2 .
\]
Combining the preceding inequalities proves the claim.
\end{proof}

\paragraph{Interpretation.}
The projected particle filter affects the discrepancy update through the
residual anchor. If the shared simulator anchor
\(\widehat\mu_t(X_t)\) is close to the oracle projected anchor
\(\mu_t^\star(X_t)\), then \(\varepsilon_t^\theta\) is small and the BRPC
residual loss is close to the oracle residual loss. If the projected parameter
posterior is biased or poorly covered by the particle cloud, then the shared
residual \(r_t\) is contaminated before the discrepancy update is performed.
This is the sense in which projected-parameter tracking error and discrepancy
learning interact in BRPC. The particle-specific discrepancy construction
satisfies the same identities with
\(\widehat\mu_t(X_t)\) replaced by \(y_s(X_t,\theta_t^{(i)})\) for each
particle \(i\).

\begin{crblock}
\subsection{Local coupled parameter--discrepancy tracking}
\label{app:cr_local_coupled}

Theorem~\ref{thm:brpc_tracking} controls the discrepancy update conditional on
the residual sequence. This subsection gives an additional \emph{local} result
under stronger assumptions: it bounds the parameter particles and the
discrepancy means jointly on a segment without restarts, for the fixed-support
representation with Gaussian--Markov propagation. The result is conditional
on local contraction of the projected parameter update and describes its
interaction with discrepancy learning. Global parameter stability remains
outside the scope of this local analysis.

\paragraph{Setting.}
Fix a segment without restarts and index its batches by $t=0,1,\dots,T$. Let
$\theta_t^\star=\Pi(\zeta_t)$ be the projected target in
\eqref{eq:projection_operator}, $\delta_t^\star=\zeta_t-y_s(\cdot,\theta_t^\star)$
the projected discrepancy, $Z=\{z_1,\dots,z_M\}$ a fixed support, and
$u_t^\star=\delta_t^\star(Z)$. For inputs $X_t$ let
$G_t=K_{X_tZ}K_{ZZ}^{-1}$ and, for a single input $x$, $g(x)=K_{ZZ}^{-1}K_{Zx}$.
The parameter particles are propagated by
$\theta_{t\mid t-1}^{(i)}=\theta_{t-1}^{(i)}+\nu_t^{(i)}$ with
$\nu_t^{(i)}\sim\mathcal N(0,\sigma_\theta^2I_{d_\theta})$ independent of the
past and of the data, reweighted as in \eqref{eq:projected_pf_weights}, and
optionally resampled with an unbiased scheme (multinomial, residual, or
systematic). Each particle carries a discrepancy mean $m_t^{(i)}$ that is
propagated and updated by
\begin{equation}
a_t^{(i)}=\rho_\delta m_{t-1}^{(i)},\qquad
P_t=\rho_\delta^2C_{t-1}+Q_{\delta,t},\qquad
J_t:=C_t^{-1}=P_t^{-1}+\eta_\delta G_t^\top R_t^{-1}G_t,
\label{eq:cr_gm_update}
\end{equation}
\[
m_t^{(i)}=C_t\big(P_t^{-1}a_t^{(i)}+\eta_\delta G_t^\top R_t^{-1}r_t^{(i)}\big),
\qquad
r_t^{(i)}=Y_t-y_s(X_t,\theta_t^{(i)}),
\]
with $Q_{\delta,t}=q_\delta K_{ZZ}$, $\rho_\delta\in(0,1]$, $q_\delta>0$, and
$M_t:=P_t^{-1}$. Because the support is fixed and all particles share the same
initial covariance $C_0$, the matrices $P_t$, $C_t$, $J_t$ do not depend on $i$.
In the shared approximation of Section~\ref{subsec:brpc_delta_update}, a single
mean $m_t$ is updated with the weighted residual \eqref{eq:res4shared}.

Treat the physical target path as fixed. Kernel hyperparameters are fixed
within the segment, and $G_t,R_t,\Sigma_{\epsilon,t}$ are chosen from the
past and $X_t$ before observing $Y_t$. Resampling copies each parameter
particle together with its discrepancy mean. Weighted sums below are evaluated
before optional unbiased resampling, which preserves their expectations.

Define the error functionals, with expectations over observation noise,
input designs, and the particle system,
\[
\mathcal E_{\theta,t}:=\mathbb E\Big[\sum_{i=1}^N w_t^{(i)}\|\theta_t^{(i)}-\theta_t^\star\|^2\Big],
\qquad
\mathcal E_{\delta,t}:=\mathbb E\Big[\sum_{i=1}^N w_t^{(i)}\|m_t^{(i)}-u_t^\star\|_{J_t}^2\Big],
\]
and, for the shared approximation,
$\mathcal E_{\delta,t}^{\mathrm{sh}}:=\mathbb E\|m_t-u_t^\star\|_{J_t}^2$. Since $M_t\preceq J_t$,
bounds on $\mathcal E_{\delta,t}$ also bound the same errors measured in the pre-update norm
$\|\cdot\|_{M_t}$. Finally let
\[
V_\theta(T)=\sum_{t=1}^T\|\theta_t^\star-\theta_{t-1}^\star\|^2,\qquad
V_\delta(T)=\sum_{t=1}^T\mathbb E\|u_t^\star-\rho_\delta u_{t-1}^\star\|_{M_t}^2,
\]
\[
\begin{aligned}
\alpha_Z&:=\sup_{\substack{0\le t\le T\\x\in\mathcal X}}
\big|\delta_t^\star(x)-g(x)^\top u_t^\star\big|,\\
\mathcal P_Z(x)^2&:=k_\phi(x,x)-K_{xZ}K_{ZZ}^{-1}K_{Zx}.
\end{aligned}
\]
\[
r_{\delta,t}:=\eta_\delta^2\,\mathbb E\mathrm{tr}\big(C_tG_t^\top R_t^{-1}\Sigma_{\epsilon,t}R_t^{-1}G_t\big).
\]

\begin{assumption}[Local conditions on a no-restart segment]
\label{ass:cr_local}
\begin{enumerate}[label=(\roman*),leftmargin=*]
\item \emph{Local contraction of the projected update.} There are
$\gamma_\theta'\in(0,1)$, $r_{\theta,t}\ge0$, and $c_N\ge0$ such that
\[
\mathcal E_{\theta,t}\le\gamma_\theta'\,
\mathbb E\Big[\sum_{i=1}^N w_{t\mid t-1}^{(i)}\|\theta_{t\mid t-1}^{(i)}-\theta_t^\star\|^2\Big]
+r_{\theta,t}+c_NN^{-1/2}.
\]
\item \emph{Local Lipschitz simulator.}
$|y_s(x,\theta)-y_s(x,\theta')|\le L_\theta\|\theta-\theta'\|$ for all
$x\in\mathcal X$ and all $\theta,\theta'$ in a neighbourhood of
$\{\theta_t^\star\}$ that contains the particles.
\item \emph{Bounded posterior covariance.} $C_t\preceq\bar c\,K_{ZZ}$ for all
$t$ and some $\bar c\ge1$.
\item \emph{Finite approximation error and predictable noise.} $\alpha_Z<\infty$;
$R_t\succeq r_{\min}I$; $K_t\le\bar K$; and the observation noise $\epsilon_t$
satisfies $\mathbb E[\epsilon_t\mid\mathcal F_{t-1},X_t]=0$ and
$\mathbb E[\epsilon_t\epsilon_t^\top\mid\mathcal F_{t-1},X_t]=\Sigma_{\epsilon,t}$.
\item \emph{Bounded reweighting} (particle-specific discrepancy only).
$w_t^{(i)}\le\varpi\,w_{t\mid t-1}^{(i)}$ for all $i,t$ and some $\varpi\ge1$.
\end{enumerate}
\end{assumption}

Condition (iii) holds with $\bar c=1$ for the stationary choice
$Q_{\delta,t}=(1-\rho_\delta^2)K_{ZZ}$ with $\rho_\delta<1$ and $C_0\preceq K_{ZZ}$:
by induction, $C_{t-1}\preceq K_{ZZ}$ gives
$P_t\preceq\rho_\delta^2K_{ZZ}+(1-\rho_\delta^2)K_{ZZ}=K_{ZZ}$, and $C_t\preceq P_t$.
Condition (i) is the key stability assumption and encodes most of the parameter
tracking; see Remark~\ref{rem:cr_contraction}.

\begin{lemma}[Contraction of Gaussian--Markov propagation]
\label{lem:cr_gm_contraction}
Under Assumption~\ref{ass:cr_local}(iii),
$\rho_\delta^2M_t\preceq\chi_\delta J_{t-1}$ with
$\chi_\delta:=\rho_\delta^2/(\rho_\delta^2+q_\delta/\bar c)<1$.
\end{lemma}

\begin{proof}
Since $K_{ZZ}\succeq\bar c^{-1}C_{t-1}$,
$P_t\succeq\rho_\delta^2C_{t-1}+q_\delta K_{ZZ}\succeq(\rho_\delta^2+q_\delta/\bar c)C_{t-1}$.
Inverting reverses the order: $M_t\preceq(\rho_\delta^2+q_\delta/\bar c)^{-1}J_{t-1}$.
\end{proof}

For the stationary choice with $\bar c=1$ and $q_\delta=1-\rho_\delta^2$, $\chi_\delta=\rho_\delta^2$. For the Gaussian--Markov random-walk special
case $\rho_\delta=1$, $Q_{\delta,t}=q_\delta K_{ZZ}$ used in
Appendix~\ref{app:markov_discrepancy_k_sensitivity}, $\chi_\delta=1/(1+q_\delta/\bar c)$.
At $q_\delta=0$, the same bound gives $\chi_\delta=1$, the nonexpansive case.

\begin{lemma}[One-step discrepancy error]
\label{lem:cr_one_step}
For any particle $i$, any $\varepsilon_1>0$, and
$n_t^{(i)}:=r_t^{(i)}-G_tu_t^\star$,
\[
\|m_t^{(i)}-u_t^\star\|_{J_t}^2
\le
(1+\varepsilon_1)\|a_t^{(i)}-u_t^\star\|_{M_t}^2
+(1+\varepsilon_1^{-1})\,\eta_\delta^2\,
n_t^{(i)\top}R_t^{-1}G_tC_tG_t^\top R_t^{-1}n_t^{(i)}.
\]
\end{lemma}

\begin{proof}
Because $C_t(P_t^{-1}+\eta_\delta G_t^\top R_t^{-1}G_t)u_t^\star=u_t^\star$,
\[
m_t^{(i)}-u_t^\star=C_t\big(P_t^{-1}(a_t^{(i)}-u_t^\star)+\eta_\delta G_t^\top R_t^{-1}n_t^{(i)}\big).
\]
For any vector $b$, $\|C_tb\|_{J_t}^2=b^\top C_tb$. Young's inequality
$(b_1+b_2)^\top C_t(b_1+b_2)\le(1+\varepsilon_1)b_1^\top C_tb_1+(1+\varepsilon_1^{-1})b_2^\top C_tb_2$
gives the second term directly. For the first, with $v=a_t^{(i)}-u_t^\star$,
$v^\top P_t^{-1}C_tP_t^{-1}v\le v^\top P_t^{-1}v$ because
$P_t^{-1/2}C_tP_t^{-1/2}\preceq I$, which follows from $J_t\succeq P_t^{-1}$.
\end{proof}

\begin{lemma}[Residual decomposition]
\label{lem:cr_residual}
Let $\Delta_t^{(i)}:=y_s(X_t,\theta_t^\star)-y_s(X_t,\theta_t^{(i)})$ and
$e_{Z,t}:=\delta_t^\star(X_t)-G_tu_t^\star$. Then
$n_t^{(i)}=\Delta_t^{(i)}+e_{Z,t}+\epsilon_t$, and under
Assumption~\ref{ass:cr_local}(ii),(iv),
\[
\eta_\delta^2\,n_t^{(i)\top}R_t^{-1}G_tC_tG_t^\top R_t^{-1}n_t^{(i)}
\le
\frac{4\eta_\delta\bar K}{r_{\min}}
\Big(L_\theta^2\|\theta_t^{(i)}-\theta_t^\star\|^2+\alpha_Z^2\Big)
+2\eta_\delta^2\,\epsilon_t^\top R_t^{-1}G_tC_tG_t^\top R_t^{-1}\epsilon_t .
\]
\end{lemma}

\begin{proof}
The decomposition follows from
$Y_t=y_s(X_t,\theta_t^\star)+\delta_t^\star(X_t)+\epsilon_t$. Write
$S_t:=\eta_\delta R_t^{-1/2}G_tC_tG_t^\top R_t^{-1/2}
=I-(I+\eta_\delta R_t^{-1/2}G_tP_tG_t^\top R_t^{-1/2})^{-1}\preceq I$
(Woodbury). Splitting $n=(\Delta+e_Z)+\epsilon$ with Young's inequality and
using $S_t\preceq I$,
$\eta_\delta^2(\Delta+e_Z)^\top R_t^{-1}G_tC_tG_t^\top R_t^{-1}(\Delta+e_Z)
\le\eta_\delta\|\Delta+e_Z\|_{R_t^{-1}}^2
\le 2\eta_\delta r_{\min}^{-1}(\|\Delta\|^2+\|e_Z\|^2)$.
Each entry of $\Delta_t^{(i)}$ is at most $L_\theta\|\theta_t^{(i)}-\theta_t^\star\|$
in absolute value. Each entry of $e_{Z,t}$ is an interpolation error bounded
by $\alpha_Z$ by definition. Summing over the $K_t\le\bar K$ entries
gives the claim.
\end{proof}

\begin{theorem}[Local coupled tracking on a no-restart segment]
\label{thm:cr_local_coupled}
Let Assumption~\ref{ass:cr_local}(i)--(iv) hold and fix
$\varepsilon_1,\varepsilon_2,\varepsilon_3>0$. Set
$\gamma_\theta:=(1+\varepsilon_3)\gamma_\theta'$,
$\gamma_\delta^{\mathrm{sh}}:=(1+\varepsilon_1)(1+\varepsilon_2)\chi_\delta$,
and, under (v) as well, $\gamma_\delta:=\varpi\,\gamma_\delta^{\mathrm{sh}}$.
Then for $t\ge1$,
\begin{align}
\mathcal E_{\theta,t}&\le\gamma_\theta \mathcal E_{\theta,t-1}
+\gamma_\theta'(1+\varepsilon_3^{-1})\|\theta_t^\star-\theta_{t-1}^\star\|^2
+\gamma_\theta'd_\theta\sigma_\theta^2+r_{\theta,t}+c_NN^{-1/2},
\label{eq:cr_A_rec}\\
\mathcal E_{\delta,t}&\le\gamma_\delta \mathcal E_{\delta,t-1}+c_{\delta\theta}\mathcal E_{\theta,t}
+(1+\varepsilon_1)(1+\varepsilon_2^{-1})\mathbb E\|u_t^\star-\rho_\delta u_{t-1}^\star\|_{M_t}^2
+c_\delta\alpha_Z^2
+2(1+\varepsilon_1^{-1})r_{\delta,t},
\label{eq:cr_D_rec}
\end{align}
with $c_\delta:=4(1+\varepsilon_1^{-1})\eta_\delta\bar K/r_{\min}$
and $c_{\delta\theta}:=c_\delta L_\theta^2$,
and \eqref{eq:cr_D_rec} holds for $\mathcal E_{\delta,t}^{\mathrm{sh}}$ with $\gamma_\delta$
replaced by $\gamma_\delta^{\mathrm{sh}}$ without assumption (v). If
$\gamma_\theta<1$ and $\gamma_\delta<1$ (respectively $\gamma_\delta^{\mathrm{sh}}<1$), then
\begin{equation}
\begin{split}
\frac{1}{T+1}\sum_{t=0}^T(\mathcal E_{\theta,t}+\mathcal E_{\delta,t})
=
O\Big(&
\frac{\mathcal E_{\theta,0}+\mathcal E_{\delta,0}+V_\theta(T)+V_\delta(T)}{T+1}
+d_\theta\sigma_\theta^2\\
&+\frac{1}{T+1}\sum_{t=1}^T(r_{\theta,t}+r_{\delta,t})
+\alpha_Z^2
+N^{-1/2}
\Big),
\end{split}
\label{eq:cr_avg}
\end{equation}
where the constant depends only on $\gamma_\theta,\gamma_\delta,\gamma_\theta',c_\delta,L_\theta,\varepsilon_1,\varepsilon_2,\varepsilon_3,c_N$.
\end{theorem}

\begin{proof}
\emph{Parameter recursion.} Propagation leaves the weights unchanged,
$w_{t\mid t-1}^{(i)}=w_{t-1}^{(i)}$ (after resampling, if any), and $\nu_t^{(i)}$
is independent of everything else with mean zero, so
$\mathbb E\|\theta_{t-1}^{(i)}+\nu_t^{(i)}-\theta_t^\star\|^2
=\mathbb E\|\theta_{t-1}^{(i)}-\theta_t^\star\|^2+d_\theta\sigma_\theta^2$.
By Young's inequality,
$\|\theta_{t-1}^{(i)}-\theta_t^\star\|^2\le(1+\varepsilon_3)\|\theta_{t-1}^{(i)}-\theta_{t-1}^\star\|^2
+(1+\varepsilon_3^{-1})\|\theta_t^\star-\theta_{t-1}^\star\|^2$.
Unbiased resampling preserves the expectation of weighted sums, so the
pre-update weighted error is at most
$(1+\varepsilon_3)\mathcal E_{\theta,t-1}+(1+\varepsilon_3^{-1})\|\theta_t^\star-\theta_{t-1}^\star\|^2+d_\theta\sigma_\theta^2$.
Assumption~\ref{ass:cr_local}(i) gives \eqref{eq:cr_A_rec}.

\emph{Discrepancy recursion.} By Lemma~\ref{lem:cr_one_step}, Young's inequality
with $\varepsilon_2$, and Lemma~\ref{lem:cr_gm_contraction},
\begin{align*}
\|a_t^{(i)}-u_t^\star\|_{M_t}^2
&\le(1+\varepsilon_2)\rho_\delta^2\|m_{t-1}^{(i)}-u_{t-1}^\star\|_{M_t}^2
+(1+\varepsilon_2^{-1})\|u_t^\star-\rho_\delta u_{t-1}^\star\|_{M_t}^2\\
&\le(1+\varepsilon_2)\chi_\delta\|m_{t-1}^{(i)}-u_{t-1}^\star\|_{J_{t-1}}^2
+(1+\varepsilon_2^{-1})\|u_t^\star-\rho_\delta u_{t-1}^\star\|_{M_t}^2 .
\end{align*}
Combine with Lemma~\ref{lem:cr_residual}, multiply by $w_t^{(i)}$, and sum
over $i$. The parameter term gives
$\sum_iw_t^{(i)}\|\theta_t^{(i)}-\theta_t^\star\|^2$, whose expectation is $\mathcal E_{\theta,t}$.
The noise quadratic form is shared across particles. Conditional on
$\mathcal F_{t-1}$ and $X_t$, its matrices are fixed and the conditional
noise covariance is $\Sigma_{\epsilon,t}$. The tower property therefore
gives the contribution $r_{\delta,t}$. For the propagated term, assumption (v) gives
$\sum_iw_t^{(i)}\|m_{t-1}^{(i)}-u_{t-1}^\star\|_{J_{t-1}}^2
\le\varpi\sum_iw_{t\mid t-1}^{(i)}\|m_{t-1}^{(i)}-u_{t-1}^\star\|_{J_{t-1}}^2$,
whose expectation is at most $\varpi \mathcal E_{\delta,t-1}$ by unbiased resampling. This
proves \eqref{eq:cr_D_rec}. In the shared approximation there is one mean, so
no reweighting factor arises; the residual perturbation is
$\Delta_t=\sum_iw_t^{(i)}\Delta_t^{(i)}$ and Jensen's inequality gives
$\|\Delta_t\|^2\le\sum_iw_t^{(i)}\|\Delta_t^{(i)}\|^2$, so the same bound holds
with $\gamma_\delta^{\mathrm{sh}}$.

\emph{Averaging.} Write \eqref{eq:cr_A_rec} as $\mathcal E_{\theta,t}\le\gamma_\theta \mathcal E_{\theta,t-1}+b_t$
and \eqref{eq:cr_D_rec} as $\mathcal E_{\delta,t}\le\gamma_\delta \mathcal E_{\delta,t-1}+c_{\delta\theta}\mathcal E_{\theta,t}+f_t$.
Summing over $t=1,\dots,T$ and writing $S_\theta=\sum_{t=0}^T\mathcal E_{\theta,t}$,
$S_\theta-\mathcal E_{\theta,0}\le\gamma_\theta S_\theta+\sum_tb_t$, so
$S_\theta\le(\mathcal E_{\theta,0}+\sum_tb_t)/(1-\gamma_\theta)$. Likewise
$\sum_{t=0}^T\mathcal E_{\delta,t}\le(\mathcal E_{\delta,0}+c_{\delta\theta}S_\theta+\sum_tf_t)/(1-\gamma_\delta)$.
Substituting $\sum_tb_t\lesssim V_\theta(T)+Td_\theta\sigma_\theta^2+\sum_tr_{\theta,t}+TN^{-1/2}$
and $\sum_tf_t\lesssim V_\delta(T)+T\alpha_Z^2+\sum_tr_{\delta,t}$ and dividing
by $T+1$ yields \eqref{eq:cr_avg}.
\end{proof}

\begin{corollary}[Response error]
\label{cor:cr_response}
Let $\widehat\zeta_t(x)=\sum_iw_t^{(i)}\{y_s(x,\theta_t^{(i)})+g(x)^\top m_t^{(i)}\}$
and define the deterministic bound
$G_Z^2:=\bar c\sup_{x\in\mathcal X}g(x)^\top K_{ZZ}g(x)<\infty$.
Under
Assumption~\ref{ass:cr_local}(ii)--(iv),
\[
\mathbb E\Big[\sup_{x\in\mathcal X}|\widehat\zeta_t(x)-\zeta_t(x)|^2\Big]
\le 3L_\theta^2\mathcal E_{\theta,t}+3G_Z^2\mathcal E_{\delta,t}+3\alpha_Z^2 .
\]
\end{corollary}

\begin{proof}
Since $\zeta_t=y_s(\cdot,\theta_t^\star)+\delta_t^\star$,
$\widehat\zeta_t(x)-\zeta_t(x)=\sum_iw_t^{(i)}\{y_s(x,\theta_t^{(i)})-y_s(x,\theta_t^\star)+g(x)^\top(m_t^{(i)}-u_t^\star)\}
+\{g(x)^\top u_t^\star-\delta_t^\star(x)\}$. Apply $(a+b+c)^2\le3(a^2+b^2+c^2)$,
Jensen's inequality over the weights,
$|g(x)^\top v|\le\|g(x)\|_{C_t}\|v\|_{J_t}\le G_Z\|v\|_{J_t}$
by Assumption~\ref{ass:cr_local}(iii), and the interpolation bound
$|g(x)^\top u_t^\star-\delta_t^\star(x)|\le\alpha_Z$.
\end{proof}

\begin{remark}[Interpretation and scope]
\label{rem:cr_contraction}
(a) Assumption~\ref{ass:cr_local}(i) supplies the local stability of the
projected parameter update. An ideal Gaussian calculation illustrates the
condition. For forecast $\mathcal N(m,pI)$ and quadratic projected loss
$\mu_\theta\|\theta-\widehat\theta\|^2/2$, let
$a=(1+\eta_\theta\mu_\theta p)^{-1}$. The tempered Gaussian posterior
has covariance $apI$ and mean $\widehat\theta+a(m-\widehat\theta)$, so
\[
\mathbb E_{\mathrm{post}}\|\theta-\widehat\theta\|^2
=a^2\|m-\widehat\theta\|^2+ad_\theta p
\le a\,\mathbb E_{\mathrm{forecast}}\|\theta-\widehat\theta\|^2.
\]
Assumption (i) requires analogous control for the finite-particle update
relative to $\theta_t^\star$; its remainder accounts for batch-target error
and particle approximation. The resulting bound is informative when that
remainder and the contraction factor are controlled. A projected random-walk move on a convex
parameter domain containing $\theta_t^\star$ also satisfies the propagation
bound used above, by nonexpansiveness of Euclidean projection.

(b) Lemma~\ref{lem:cr_gm_contraction} compares forecast precision with the
previous posterior precision. Theorem~\ref{thm:brpc_tracking} instead uses
successive forecast precisions. The shared-discrepancy recursion uses
$\chi_\delta$ directly. Its particle-specific counterpart additionally
requires sufficiently mild reweighting, $\varpi\chi_\delta<1$, to obtain
a strict contraction after choosing $\varepsilon_1,\varepsilon_2$.

(c) The bound has an error floor determined by propagation noise, update
noise, support approximation, and finite particle count. When
$R_t=\Sigma_{\epsilon,t}$,
\[
r_{\delta,t}=\eta_\delta\,\mathbb E\mathrm{tr}
\big(I-P_t^{-1/2}C_tP_t^{-1/2}\big)
\le\eta_\delta\min(M,\bar K),
\]
which measures the effective number of directions informed by one batch.

(d) If the physical residual additionally satisfies
$\sup_t\|\delta_t^\star\|_{\mathcal H_{k_\phi}}\le B_\delta$, the kernel
interpolation bound gives $\alpha_Z\le B_\delta\sup_x\mathcal P_Z(x)$.
For a Mat\'ern kernel of smoothness $\nu$ on a bounded Lipschitz
domain in $\mathbb R^{d_x}$ and quasi-uniform $Z$, the power function obeys
$\sup_x\mathcal P_Z(x)=O(M^{-\nu/d_x})$. Thus
$\alpha_Z^2=O(B_\delta^2M^{-2\nu/d_x})$
\citep[Ch.~11]{wendland2004scattered}. This RKHS condition describes
smoothness of the physical residual; the GP supplies the inference prior.

(e) The result applies within a segment. Restarts set new initialization
errors, while their detection behavior is analyzed in
Appendix~\ref{app:compatibility_theory}.
\end{remark}
\end{crblock}

\FloatBarrier
\begin{crblock}
\subsection{Propagation Stability Diagnostics}
\label{app:transport_replay}

Assumption~\ref{ass:brpc_transport_contraction} is a sufficient pre-update propagation condition used in the
tracking analysis. To inspect the corresponding stability
geometry in representative streams, we compute two diagnostics.

Here $A_t$ includes temporal persistence as in \eqref{eq:delta_state_forecast}, $P_t$ is the propagated pre-update covariance, $C_t$ is the
post-update covariance, and $M_t=(P_t)^{-1}$ is the
pre-update precision. The diagnostic uses the same discrepancy tempering
parameter $\eta_\delta$ as in Proposition~\ref{prop:recursive_delta_update}.

Let $\lambda_{\max}(\cdot)$ denote the maximum eigenvalue. We compute
\[
\gamma_{\mathrm{prior},t}
=
\lambda_{\max}
\left(
\bigl(M_{t-1}\bigr)^{-1/2}
A_t^\top M_t A_t
\bigl(M_{t-1}\bigr)^{-1/2}
\right),
\]
which is the direct empirical analogue of
Assumption~\ref{ass:brpc_transport_contraction} because it uses the propagated pre-update precision. We also
compute
\[
\gamma_{\mathrm{post},t}
=
\lambda_{\max}
\left(
\bigl(M_{t-1}\bigr)^{-1/2}
A_t^\top C_t^{-1} A_t
\bigl(M_{t-1}\bigr)^{-1/2}
\right).
\]
The second quantity measures posterior sharpening after the current batch. The
first quantity evaluates Assumption~\ref{ass:brpc_transport_contraction}. Since
\[
C_t^{-1}
=
(P_t)^{-1}
+
\eta_\delta G_t^\top R_t^{-1}G_t,
\]
\(\gamma_{\mathrm{post},t}\) can exceed one even when the pre-update propagation
is stable.

\begin{table}[t]
\centering
\small
\setlength{\tabcolsep}{5pt}
\caption{
Propagation-metric replay diagnostic on representative mixed streams.
\(\gamma_{\mathrm{prior}}\) is the pre-update propagation diagnostic closest to
Assumption~\ref{ass:brpc_transport_contraction}. \(\gamma_{\mathrm{post}}\) uses the post-update precision and
therefore measures posterior sharpening after assimilation.
}
\label{tab:transport_replay}
\resizebox{\linewidth}{!}{%
\begin{tabular}{llrrrrrr}
\toprule
Dataset & Scenario
& \(n\)
& Median \(\gamma_{\mathrm{prior}}\)
& Max \(\gamma_{\mathrm{prior}}\)
& Frac. \(\gamma_{\mathrm{prior}}\le 1\)
& Median \(\gamma_{\mathrm{post}}\)
& Max \(\gamma_{\mathrm{post}}\) \\
\midrule
Synthetic & Mixed
& 145
& 0.9999
& 1.0000
& 0.9517
& 1.0666
& 1.9844 \\
PlantSim & Mixed
& 1060
& 0.1532
& 0.9877
& 1.0000
& 0.5846
& 2.1394 \\
\bottomrule
\end{tabular}%
}
\end{table}

Table~\ref{tab:transport_replay} reports the diagnostic on the mixed
synthetic and plant-simulation streams, which contain both gradual drift and
abrupt changes. The pre-update diagnostic \(\gamma_{\mathrm{prior}}\) is nearly
nonexpansive on the synthetic stream and strictly contractive on the
plant-simulation stream. The reported fraction of synthetic steps above one shows occasional
expansion, so uniform contraction remains an additional requirement of
Theorem~\ref{thm:brpc_tracking}.

The post-update diagnostic behaves differently. In both benchmarks,
\(\gamma_{\mathrm{post}}\) can be larger than one, especially in the upper tail.
The Gaussian update adds $\eta_\delta G_t^\top R_t^{-1}G_t$ to the
precision, explaining this posterior sharpening. The relevant propagation
criterion remains $\gamma_{\mathrm{prior},t}$.

The fixed Gaussian replay evaluates the covariance recursion from input
batches, kernel hyperparameters, residual noise, and likelihood tempering.
These settings determine its propagation geometry. It summarizes the replay
configuration rather than certifying stability for every experimental suite.
\end{crblock}

\FloatBarrier

\section{Benchmark Construction and Experimental Settings}
\label{app:benchmark_details}

This appendix describes the benchmark streams and experimental settings used in
Section~\ref{sec:experiments}. Appendix~\ref{app:synthetic_benchmark} defines
the synthetic benchmark and its drifting, sudden-jump, and mixed nonstationary
trajectories. Appendix~\ref{app:plant_benchmark} describes the plant-simulation
benchmark and the surrogate model used to generate streaming responses.
Appendix~\ref{app:common_experimental_settings} collects the common settings for
particles, discrepancy kernels, restart rules, baselines, event metrics, and
compute resources.

\begin{crblock}
\subsection{Synthetic benchmark}
\label{app:synthetic_benchmark}

The synthetic benchmark adapts the static calibration example of
\citet{gu2018scaled} to an online nonstationary stream. The simulator is
\[
y_s(x,\theta)=\sin(\theta x)+5x,
\]
and the physical response is
\[
\zeta_b(x)=5x\cos(\omega_b x/2)+5x,
\]
observed with independent Gaussian noise,
\[
Y_b(x)=\zeta_b(x)+\epsilon_b(x),
\qquad
\epsilon_b(x)\sim\mathcal N(0,\sigma_y^2).
\]
In the original static example, $\omega_b=15$ is fixed. We turn this example
into an online benchmark by allowing the physical-response parameter to vary
across batches. The batch index is denoted by $b=0,\ldots,B-1$, where
$B=T/K_t$ is the number of batches. We use \(b\) here to emphasize that this is the batch-level index used in
benchmark generation. In the main text, the generic online update index is
denoted by \(t\). Thus \(b\) in this appendix plays the same role as \(t\) in
the online algorithms, but is reserved here for describing the data-generating
trajectory. All observations within batch \(b\) are generated from the same
physical-response parameter \(\omega_b\). Hence each batch is treated as locally
homogeneous, and nonstationarity is modeled as variation of \(\omega_b\) across
batches rather than within a batch.

For each batch \(b\), the projected calibration target is defined by applying
the projection operator in Eq.~\eqref{eq:projection_operator} to the physical
response \(\zeta_b\):
\[
\theta_b^\star
=
\Pi(\zeta_b)
\in
\arg\min_{\theta\in\Theta}
\int_{\Omega_x}
\{\zeta_b(x)-y_s(x,\theta)\}^2\,dF_X(x).
\]
In the synthetic benchmark, this projection is computed numerically on a dense
grid. Specifically, for each candidate physical-response parameter \(\omega_b\),
we evaluate \(\zeta_b(x)=5x\cos(\omega_b x/2)+5x\), compute the empirical
projection loss over a fixed grid of \(x\)-values and \(\theta\)-values, and
take the minimizer as \(\theta_b^\star\). This gives a deterministic numerical
mapping from \(\omega_b\) to the projected calibration target \(\theta_b^\star\). When constructing a stream with a prescribed target trajectory
\(\{\theta_b^\star\}\), we first precompute this grid-based mapping
\(\omega\mapsto \Pi(\zeta_\omega)\), and then choose \(\omega_b\) whose projected
target is closest to \(\theta_b^\star\). The resulting \(\theta_b^\star\) is then
used as the ground-truth projected calibration target for evaluation.
The projected target \(\theta_b^\star\) is used to evaluate tracking error;
the algorithms receive the generated observations.

We consider three trajectory families for $\theta_b^\star$.

\paragraph{Drifting family.}
In the drifting family, the projected target evolves gradually across batches.
We use slope-style trajectories of the form
\[
\theta_b^\star
=
\theta^\star_{b-1}
+
s
+
\xi_b^{\mathrm{gen}},
\]
where $s$ controls the drift strength and $\xi_b^{\mathrm{gen}}$ denotes a small smooth
perturbation. The drift slopes used in the experiments are
\[
s\in
\{5{\times}10^{-4},\,10^{-3},\,1.5{\times}10^{-3},\,
2{\times}10^{-3},\,2.5{\times}10^{-3}\}.
\]
This setting tests whether the online updates can track gradual movement of the
projected calibration target without requiring explicit segmentation.

\paragraph{Sudden-jump family.}
In the sudden-jump family, the projected target is piecewise constant. For a
segment length $L$, annotated changepoints occur every $L$ observations,
equivalently every $L/K_t$ batches. At a changepoint batch $c_j$, the projected
target is shifted by a jump increment $J_j$:
\[
\theta_b^\star
=
\theta^\star_0
+
\sum_{j:c_j\le b}J_j .
\]
The jump magnitudes used in the experiments are
\[
|J_j|\in\{0.5,1.0,2.0,3.0\},
\]
with signs chosen so that the resulting trajectory remains in the admissible
calibration range. This setting tests whether restart mechanisms can respond to
abrupt changes without producing excessive false restarts.

\paragraph{Mixed family.}
In the mixed family, gradual drift and abrupt jumps are combined. The projected
target follows
\[
\theta_b^\star
=
\theta^\star_{b-1}
+
d_{s(b)}
+
\xi_b^{\mathrm{gen}}
+
\sum_{j=1}^2 J_j\mathbf 1\{b=c_j\},
\]
where $s(b)\in\{1,2,3\}$ is the segment index, $d_{s(b)}$ is the
segment-specific drift slope, $\xi_b^{\mathrm{gen}}$ is a smooth perturbation, and $J_j$ is the
jump increment at changepoint batch $c_j$. The two changepoints are placed near
one-third and two-thirds of the stream, with small random jitter:
$c_1\approx 0.33B,
c_2\approx 0.70B.$

\crnew{This data-generation perturbation is distinct from the inference innovation $\xi_t(\cdot)$ and from the parameter-particle move. It follows an AR(1)-type recursion}
\[
\xi_b^{\mathrm{gen}}=\rho_{\mathrm{pert}}\xi_{b-1}^{\mathrm{gen}}+\sigma_{\mathrm{pert}} z_b,
\qquad
z_b\sim\mathcal N(0,1),
\]
with $\rho_{\mathrm{pert}}=0.65$ and $\sigma_{\mathrm{pert}}=0.015$ in the main mixed experiments. The
segment slopes are random multiples of the drift-scale parameter, and the jump
increments are random multiples of the jump-scale parameter. In the main experiments, we use $d_{s(b)}=0.009$ and $J_j\in\{0.28,0.38,0.58\}.$

After generation, the trajectory is centered, rescaled, and clipped to the
admissible range $[1.0,2.5]$. A minimum realized jump size is enforced at the
annotated changepoints so that the mixed stream contains identifiable abrupt
changes.

Across the three families, the benchmark produces simulator--reality mismatch
while retaining a known nonstationary structure. Event-level restart metrics such
as Precision@2, Recall@2, and Delay@2 are computed only for settings with
annotated abrupt changes, namely the sudden-jump and mixed families. In purely
drifting settings, these event-level metrics are not reported.

Unless otherwise stated, the number of particles is fixed at $1024$ for
particle-based methods, and the maximum number of BOCPD experts maintained
online is fixed at $5$. The configuration ranges used in the synthetic study are
summarized in Table~\ref{tab:app_synthetic_configs}.

\begin{table}[t]
\centering
\small
\caption{
Synthetic benchmark configurations. Segment length denotes the number of
observations between consecutive annotated changepoints in sudden-jump streams.
The mixed family combines segmentwise drift, smoothed perturbations, and two
annotated abrupt jumps.
}
\label{tab:app_synthetic_configs}
\begin{tabular}{ll}
\toprule
\textbf{Setting} & \textbf{Configuration} \\
\midrule
\multicolumn{2}{l}{\textbf{Sudden-jump family}} \\
\quad Segment length $(L)$ & $\{80,120,200\}$ \\
\quad Jump magnitude & $\{0.5,1.0,2.0,3.0\}$ \\
\quad Batch size $(K_t)$ & $\{10,20,40\}$ \\
\quad Total observations & $4L$ \\
\quad Particles & $1024$ \\
\quad Experts maintained & $5$ \\
\midrule
\multicolumn{2}{l}{\textbf{Drifting family}} \\
\quad Drift slope $(s)$ &
$\{5{\times}10^{-4},\,10^{-3},\,1.5{\times}10^{-3},\,2{\times}10^{-3},\,2.5{\times}10^{-3}\}$ \\
\quad Batch size $(K_t)$ & $\{10,20,40\}$ \\
\quad Total observations & $600$ \\
\quad Particles & $1024$ \\
\quad Experts maintained & $5$ \\
\midrule
\multicolumn{2}{l}{\textbf{Mixed family}} \\
\quad Drift scale ($d_{s(b)}$) & $0.009$  \\
\quad Jump scale ($J_j$) & $\{0.28,0.38\}$  \\
\quad Smooth perturbation s.d. ($\sigma_{\mathrm{pert}}$) & $0.015$ \\
\quad Perturbation AR coefficient ($\rho_{\mathrm{pert}}$)& $0.65$ \\
\quad Changepoint locations & Near $0.33B$ and $0.70B$, with random jitter \\
\quad Admissible target range & $[1.0,2.5]$ \\
\quad Batch size $(K_t)$ & $20$  \\
\quad Total observations & $600$ \\
\quad Particles & $1024$ \\
\quad Experts maintained & $5$ \\
\bottomrule
\end{tabular}
\end{table}
\end{crblock}

\subsection{Plant-simulation benchmark}
\label{app:plant_benchmark}

The plant-simulation benchmark complements the stylized synthetic study with a
more realistic digital-twin synchronization task. The physical system is a
bicycle-production line implemented in Plant Simulation, with stochastic
arrivals and processing times, work-in-progress and inventory constraints, and
regime-dependent demand. The latent parameter to be tracked is an arrival-rate
variable $\theta_b$ that determines the effective system load. Figure~\ref{fig:ps} shows
the digital-twin environment used for this case study.

\begin{figure}[htbp]
    \centering
    \includegraphics[width=\textwidth]{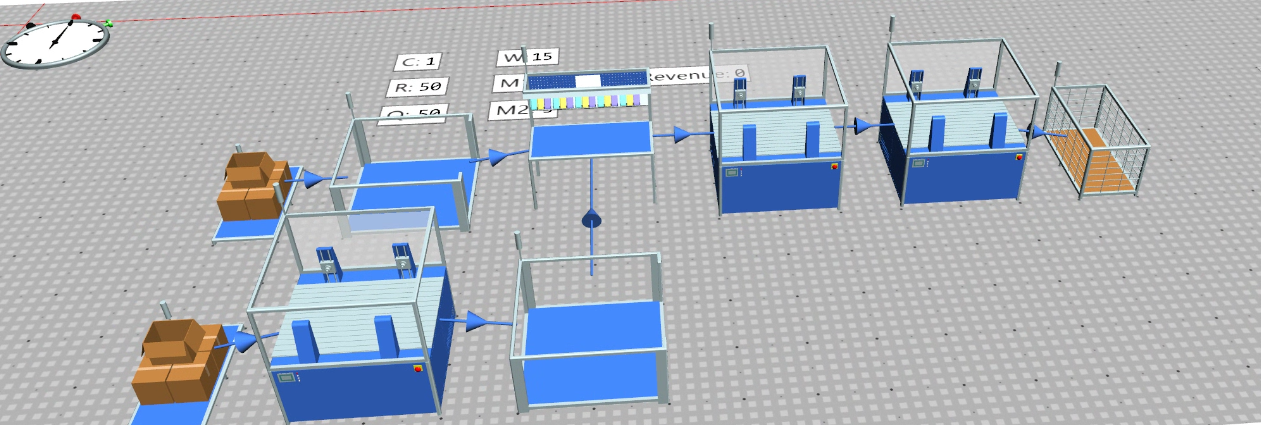}
    \caption{Digital-twin environment for the bicycle-production
    plant-simulation benchmark. The online calibration stream is generated from
    this Plant Simulation model under multiple operating modes.}
    \label{fig:ps}
\end{figure}

The simulated line contains three main stations (Assembly, Cleaning, and
Inspection), two buffers (CustomerBuffer and KitBuffer), and an $(R,Q)$
inventory policy. Each batch also carries an exogenous input vector
\[
X=(M_1,M_2,W,R,Q),
\]
where $M_1$ and $M_2$ are station-level capacity settings, $W$ is a workload
factor, and $(R,Q)$ are inventory-policy factors. The input ranges are
summarized in Table~\ref{tab:app_bicycle_controls}.

\begin{table}[t]
\centering
\small
\caption{Input factors used in the bicycle-plant simulation benchmark.}
\label{tab:app_bicycle_controls}
\begin{tabular}{lcc}
\toprule
Factor & Levels / Range & Description \\
\midrule
$M_1$ & $\{1,2,3\}$ & Number of parallel machines at the cleaning station \\
$M_2$ & $\{1,2,3\}$ & Number of parallel machines at the inspection station \\
$W$   & $[10,15]$   & Workload factor \\
$R$   & $[50,100]$  & Reorder point \\
$Q$   & $[50,100]$  & Order quantity \\
\bottomrule
\end{tabular}
\end{table}

We generate $2{,}000$ Plant Simulation runs by Latin hypercube sampling over the
control space and train a surrogate model for 10-day net revenue, which becomes
the streaming response used by the online calibrator. The surrogate input is
\[
(W,R,M_1,M_2,Q,\theta),
\]
where $\theta$ is the arrival-rate parameter. Preliminary experiments found that
a standard Gaussian-process surrogate underfit this response surface, so the
benchmark uses a multilayer perceptron surrogate instead.

The MLP has three hidden layers with widths $(256,256,128)$ and ReLU activations,
followed by a scalar output layer. Thus the network has four linear layers in
total. Before training, the response is transformed using a signed log-one-plus
transform and both the inputs and transformed responses are standardized using
the training split. The model is trained with a $90/10$ train-validation split,
AdamW optimization, mean-squared-error loss, batch size $128$, learning rate
$10^{-3}$, and weight decay $10^{-6}$. Training is run for at most $400$ epochs
with early stopping based on validation loss.

For the physical data stream, we sequentially sample control inputs and generate
observations under three representative operating modes:
\begin{enumerate}
    \item \textbf{Drifting mode:} the latent arrival-rate parameter varies
    smoothly. In the sinusoidal configuration,
    \[
    \theta_b = 11.5 + 8.5\sin\!\left(\frac{2\pi b}{40}\right).
    \]
    \item \textbf{Sudden-change mode:} the latent arrival-rate parameter is
    piecewise constant with abrupt regime switches.
    \item \textbf{Mixed mode:} the latent arrival-rate parameter combines
    continuous variation with abrupt structural changes.
\end{enumerate}

The goal is to infer the latent arrival-rate parameter online from the revenue
stream while maintaining synchronization between the physical system and the
digital twin. This benchmark is intended to test whether the qualitative
conclusions from the synthetic study persist when parameter mismatch, surrogate
error, and operational stochasticity are entangled in a realistic simulator.
Reported relative $y$-metrics are normalized by the mode-specific response scale
so that predictive performance is comparable across operating modes.

\subsection{Common Experimental Settings}
\label{app:common_experimental_settings}

All methods are evaluated on common generated streams and common random seeds
within each benchmark. Observations are processed in mini-batches; the main synthetic and
plant-simulation CPD runs use batch size $K_t=20$. \crnew{Scalar metrics are
reported as mean $\pm$ standard deviation across runs. The main synthetic
results pool 25 seeds per configuration: five drift slopes (125 runs),
four jump magnitudes and three segment lengths (300 runs), and three mixed
jump scales (75 runs). Their standard deviations reflect variation across
both configurations and seeds.}

\paragraph{Synthetic data and projected targets.}
Inputs satisfy $x\in[0,1]$ and are sampled within each batch by a randomized
stratified grid. The observation noise in the generated synthetic data is
\[
\epsilon\sim\mathcal N(0,0.2^2).
\]
The projected target \(\theta_t^\dagger\) is approximated offline by minimizing
the empirical \(L_2\) projection loss over a uniform grid of 600 candidate
calibration values on \([0,3]\) and a uniform grid of 400 input locations on
\([0,1]\).
The main synthetic suite uses $N=1024$ particles, batch size $K_t=20$, and
discrepancy KL weight $\eta_\delta=1.0$.

\paragraph{BRPC implementation.}
Particle-based BRPC methods use $N=1024$ particles, systematic resampling, and
ESS threshold ratio $0.5$. The local parameter evolution prior uses a
\crnew{Gaussian random-walk move for $\theta$}, with default scale $0.1$ in the main synthetic
runs. The likelihood noise used inside the BRPC predictive model is
$\sigma_{\epsilon}=0.05,$
which is a modeling/score noise parameter and is distinct from the synthetic
data-generation noise $0.2$. Unless otherwise stated, the discrepancy GP uses an
RBF kernel with lengthscale $1.0$ and variance $0.01$.
\crnew{In the synthetic portion of Table~\ref{tab:brpc_main}, B-BRPC and
C-BRPC use expanding support (E) with $\rho_\delta=1$ and
$q_\delta=10^{-2}$. Thus $Q_{\delta,t}=10^{-2}K_{S_tS_t}$, and each
innovation is a GP draw with covariance $k_\phi$ scaled by $0.1$.
Table~\ref{tab:discrepancy_innovation} varies $q_\delta$ around this setting;
the carry-over and forgetting ablations use the configurations specified
in their respective subsections.}

\paragraph{\crnew{Code availability.}}
\crnew{Code is available at \url{https://github.com/crushonyfg/BRPC}.}

\paragraph{Restart mechanisms.}
For B-BRPC, the default BOCPD-style hazard is
\[
h(r)=\frac{1}{200+r},
\]
where $r$ is the current run length. The maximum number of active experts is
$5$. The main implementation uses a restart cooldown of $10$ batches and a
restart margin of $1$ for the hard-restart decision. Sensitivity to these
restart-rule settings is reported in
Appendix~\ref{app:hyper_sensitivity}.

For C-BRPC, the default wCUSUM configuration is
\[
W=4,\qquad h_{\mathrm C}=0.25,\qquad \kappa=0.25,\qquad \sigma_{\min}=0.25,
\]
with $3$ warm-up batches and the batch-normalized log-surprise score. Sensitivity
checks over alternative wCUSUM configurations are reported in
Appendix~\ref{app:hyper_sensitivity}.

\paragraph{Baselines.}
BC(80) denotes sliding-window Kennedy--O'Hagan recalibration using the most
recent $80$ observations. 
In the synthetic benchmark, BC(80) searches over a uniform grid of 200
candidate calibration values on \([0,3]\) and uses $\sigma_{\rm obs}=0.2$
and a GP discrepancy with lengthscale $0.3$ and signal variance $1.0$. The DA
baseline uses $N=1024$ particles, $\theta\in[0,3]$, observation noise
$\sigma_{\rm obs}=0.2$, parameter random-walk standard deviation $0.05$, and a
residual-GP window of $80$ observations with the same GP lengthscale and signal
variance as BC(80).

\paragraph{Event metrics.}
Precision@2, Recall@2, F1@2, and Delay@2 are computed only in settings with
annotated abrupt changes. A detected restart is counted as correct if it occurs
within two batches after an annotated changepoint; early detections are not
rewarded. In mixed settings, drift-triggered restarts may therefore reduce
Precision@2 even if they improve local adaptation.

\paragraph{Compute resources.}
Experiments were run on a local workstation with an NVIDIA T1000 GPU with 4GB
memory, driver version 573.44, and CUDA 12.8 available. \crnew{Computation is dominated by}
particle updates, Gaussian-process discrepancy updates, restart-mechanism
bookkeeping, and repeated benchmark runs across random seeds. The
plant-simulation surrogate was trained once and then reused for the online
calibration experiments.

\FloatBarrier

\begin{crblock}
\section{Additional Experiments and Ablation Studies}
\label{app:additional_results}

The experiments below follow the questions in Section~\ref{sec:experiments}:
which baselines explain the gains, how the discrepancy representation and
dynamics affect adaptation, how restart controllers behave, and how the method
scales and supports downstream decisions. Each subsection specifies its own
experimental configuration; carry-over diagnostics and discrepancy-dynamics
experiments are interpreted within their respective suites.

\subsection{Extended Baseline Comparisons}
\label{app:extended_baselines}

\subsubsection{Windowed, Discounted, and Segment-Based Calibration}
\label{app:cr_additional_baselines}

Table~\ref{tab:cr_additional_baselines} expands the comparison in
Table~\ref{tab:brpc_sources} using a separate synthetic suite with 25 seeds
per scenario. BC($w$) refits Kennedy--O'Hagan calibration on the most recent
$w$ observations; discounted BC ($\gamma$) down-weights older batches geometrically.
Static Bayesian projected calibration (BPC; \citealp{xie2021bayesian}) is
refit on oracle changepoint segments, BOCPD-detected segments, or a sliding
window of 40 observations. Oracle segmentation is an idealized reference
that uses the annotated boundaries.

\begin{table}[ht]
\centering
\footnotesize
\setlength{\tabcolsep}{4pt}
\caption{\protect\crnew{Windowed, discounted, and segment-based calibration on the synthetic benchmark
(mean over 25 seeds). Columns follow the metric and scenario groups in Table~\ref{tab:brpc_sources}.}}
\label{tab:cr_additional_baselines}
\begin{tabular}{@{}lrrrrrr@{}}
\toprule
& \multicolumn{3}{c}{$\theta$-RMSE $\downarrow$}
& \multicolumn{3}{c}{Response RMSE $\downarrow$} \\
\cmidrule(lr){2-4}\cmidrule(l){5-7}
Method & Drift & Sudden & Mixed & Drift & Sudden & Mixed \\
\midrule
B-BRPC & 0.016 & 0.018 & 0.021 & 0.37 & 0.71 & 0.44 \\
C-BRPC & 0.015 & 0.019 & 0.019 & 0.44 & 0.63 & 0.44 \\
BRPC (no restart) & 0.014 & 0.031 & 0.036 & 0.35 & 0.67 & 0.44 \\
\midrule
BC(40) & 0.105 & 0.118 & 0.112 & 0.45 & 0.80 & 0.53 \\
BC(80) & 0.136 & 0.178 & 0.150 & 0.62 & 1.15 & 0.71 \\
BC(160) & 0.189 & 0.264 & 0.200 & 0.96 & 1.55 & 1.01 \\
Discounted BC ($\gamma=0.90$) & 0.090 & 0.077 & 0.076 & 0.39 & 0.72 & 0.51 \\
Discounted BC ($\gamma=0.97$) & 0.116 & 0.128 & 0.109 & 0.55 & 0.95 & 0.61 \\
\midrule
BPC (oracle segments) & 0.060 & 0.033 & 0.040 & 0.60 & 0.62 & 0.49 \\
BPC (BOCPD segments) & 0.071 & 0.067 & 0.075 & 0.73 & 0.76 & 0.58 \\
BPC (sliding window, 40) & 0.060 & 0.139 & 0.108 & 0.60 & 1.14 & 0.69 \\
Joint EnKF & 0.109 & 0.139 & 0.373 & 0.56 & 1.07 & 0.85 \\
\bottomrule
\end{tabular}
\end{table}

Short windows and discounting improve BC adaptation, but their parameter errors
remain above those of recursive BRPC in this suite. Oracle-segment BPC yields
the lowest sudden-change response RMSE, while its parameter error remains
higher than that of BRPC. Thus accurate segmentation can improve prediction
without reproducing the parameter-tracking accuracy of the recursive update.

\FloatBarrier
\subsubsection{Joint Filtering with and without Restart}
\label{app:joint_enkf_sensitivity}

The joint EnKF maintains an ensemble over
$z_t=(\theta_t,\beta_t^\top)^\top$ and predicts
$y_s(X_t,\theta_t)+B_t\beta_t$, where $B_t$ contains an intercept, a linear
term, and radial-basis functions. The forecast uses
\[
\theta_t^{(i),-}=\Pi_\Theta(\theta_{t-1}^{(i)}+\xi_{\theta,t}^{(i)}),
\qquad
\beta_t^{(i),-}=\rho_\beta\beta_{t-1}^{(i)}+\xi_{\beta,t}^{(i)},
\]
with independent Gaussian innovations of standard deviations
$\sigma_\theta$ and $\sigma_\beta$. The perturbed-observation update is
\[
z_t^{(i)}=z_t^{(i),-}+C_{zy}C_{yy}^{-1}
\{Y_t+\epsilon_t^{(i)}-\widehat Y_t^{(i),-}\},
\]
where $C_{zy}$ and $C_{yy}$ are empirical forecast cross-covariance and
observation covariance, with observation noise added to $C_{yy}$.
Both calibration and bias are updated through the same innovation.

We vary $\sigma_\theta\in\{0.015,0.035,0.07\}$ and
$\sigma_\beta\in\{0.005,0.015,0.04\}$, together with the covariance
inflation settings shown in Figure~\ref{fig:jointenkf_theta_heatmap}.
The sensitivity landscapes differ across scenarios: no single allocation of
process noise dominates throughout. These results complement the joint-EnKF
row in Table~\ref{tab:cr_additional_baselines} and illustrate the tuning
needed to balance responsiveness and parameter--bias attribution.

\begin{figure}[!htbp]
\centering
\includegraphics[width=\linewidth]{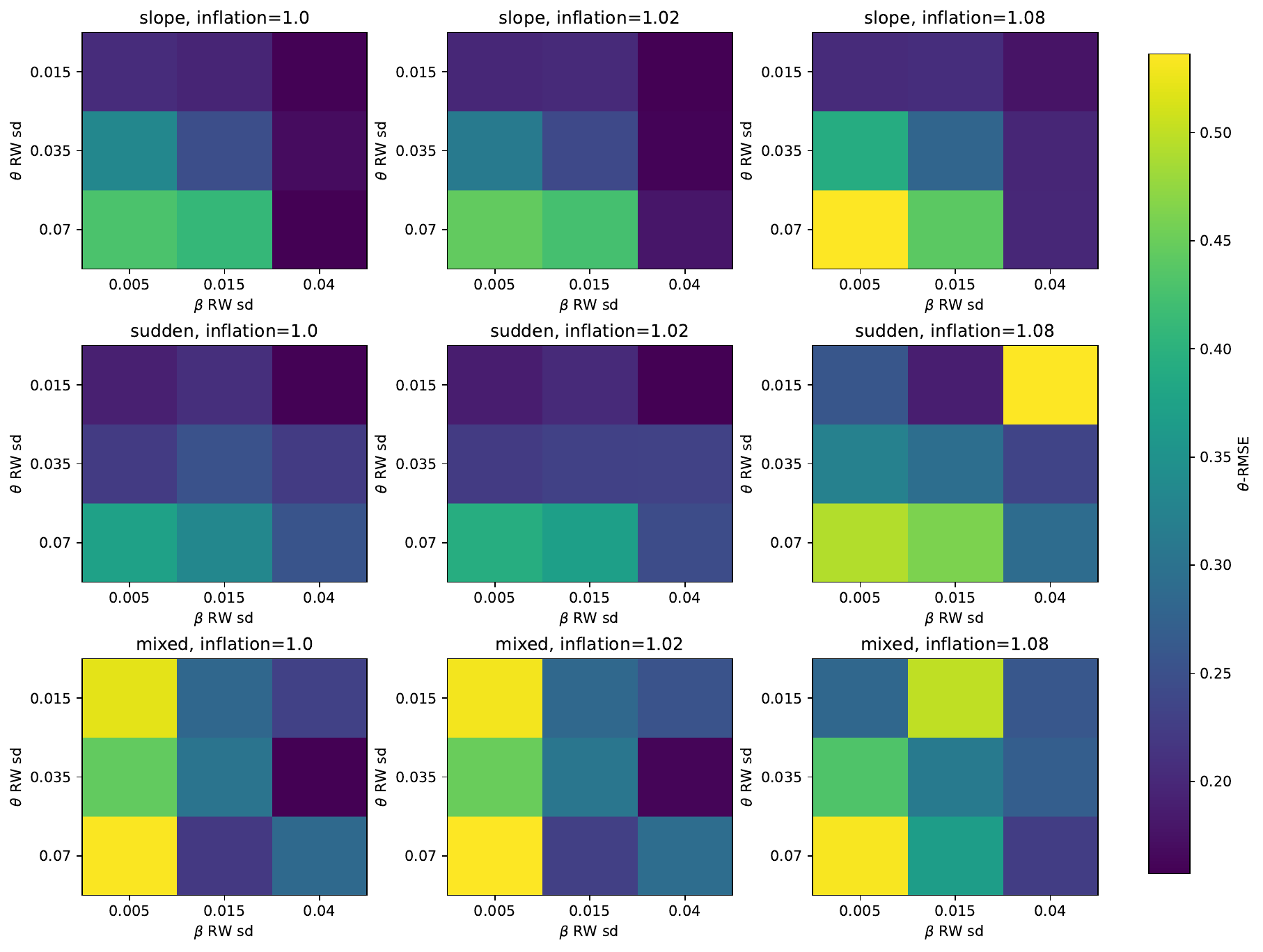}
\caption{Joint-EnKF parameter-RMSE sensitivity. Rows are synthetic scenarios,
columns are covariance inflation settings, and each heatmap spans the
$(\sigma_\theta,\sigma_\beta)$ grid. Darker cells indicate lower error
within a scenario.}
\label{fig:jointenkf_theta_heatmap}
\end{figure}

\paragraph{Adding BOCPD to a moving-particle filter.}
\label{app:bocpd_da_baseline}
BOCPD-WardPFMove couples the moving-particle filtering construction of
\citet{ward2021continuous} to BOCPD using its pre-update predictive
likelihoods. Table~\ref{tab:bocpd_da_diagnostic} compares it with B-BRPC-E
in a separate carry-over diagnostic. Parameter random walks permit local
adaptation, but the joint filtering update can attribute simulator mismatch
or abrupt change to parameter movement before restart evidence accumulates.

\begin{table}[t]
\centering
\small
\setlength{\tabcolsep}{4.5pt}
\caption{
Diagnostic comparison between B-BRPC-E and a moving-particle filtering baseline
with BOCPD restart on the synthetic benchmark. Values are mean $\pm$ standard
deviation across runs. Restart count is a restart-behavior diagnostic. Here
restart count denotes hard restarts selected by the BOCPD run-length posterior;
BOCPD updates its run-length posterior at every batch even when no hard restart
is selected.
}
\label{tab:bocpd_da_diagnostic}
\resizebox{\linewidth}{!}{%
\begin{tabular}{llrrrr}
\toprule
Scenario & Method
& $\theta$-RMSE $\downarrow$
& Response RMSE $\downarrow$
& Restarts
& Time (s) $\downarrow$ \\
\midrule

\multirow{2}{*}{Drifting}
& B-BRPC-E
& $0.0151 \pm 0.0024$
& $0.4919 \pm 0.2468$
& $12.16 \pm 1.86$
& $11.70 \pm 1.57$ \\
& BOCPD-WardPFMove
& $0.0902 \pm 0.0302$
& $1.9090 \pm 0.1164$
& $0.00 \pm 0.00$
& $5.97 \pm 0.05$ \\

\midrule
\multirow{2}{*}{Sudden}
& B-BRPC-E
& $0.0212 \pm 0.0196$
& $0.6123 \pm 0.2934$
& $10.80 \pm 4.44$
& $10.18 \pm 3.80$ \\
& BOCPD-WardPFMove
& $0.2273 \pm 0.2003$
& $1.9929 \pm 0.0720$
& $0.30 \pm 0.50$
& $5.20 \pm 1.94$ \\

\midrule
\multirow{2}{*}{Mixed}
& B-BRPC-E
& $0.0194 \pm 0.0167$
& $0.5768 \pm 0.2845$
& $11.20 \pm 3.90$
& $10.63 \pm 3.37$ \\
& BOCPD-WardPFMove
& $0.1870 \pm 0.1800$
& $1.9682 \pm 0.0948$
& $0.21 \pm 0.44$
& $5.43 \pm 1.66$ \\

\bottomrule
\end{tabular}%
}
\end{table}

The moving-particle baseline rarely selects hard restarts and has larger
parameter and response errors in all three scenarios. This is consistent with
the attribution issue in Appendix~\ref{app:joint_assimilation_identifiability},
and shows that the moving-particle baseline with BOCPD still has larger
tracking errors in this suite. The comparison supports a performance
difference while leaving its precise cause open.

\FloatBarrier
\subsubsection{Orthogonal-GP Calibration}
\label{app:cr_ogp}

Orthogonal-GP calibration \citep{plumlee2017bayesian} offers a complementary
way to reduce parameter--discrepancy confounding by constraining discrepancy
relative to the simulator tangent space. An online implementation must update
the orthogonal kernel as the parameter state evolves and approximate the
required input-space integrals. Simulator derivatives and numerical integration
introduce additional choices, especially for learned emulators and
high-dimensional inputs. Orthogonality addresses attribution; adaptation to
nonstationarity still requires an evolution or forgetting mechanism.
\FloatBarrier
\end{crblock}

\begin{crblock}
\subsection{BRPC Representation and State-Sharing Ablations}
\label{app:representation_and_sharing}

\subsubsection{Support Representation: E, P, and F}
\label{app:brpc_state_ablation}

We compare expanding support (E), proxy observations (P), and fixed support
(F), as defined in Appendix~\ref{app:brpc_state_maps}. In each controller
block, the restart rule is held fixed while the discrepancy representation
changes. Tables~\ref{tab:representation_ablation}
and~\ref{tab:plant_representation_ablation} use the carry-over configuration;
they use $\rho_\delta=1$, $q_\delta=0$ to isolate representation costs
and accuracy. The main synthetic B-BRPC and C-BRPC configuration instead
uses $\rho_\delta=1$ and $q_\delta=10^{-2}$.

\begin{table}[!htbp]
\centering
\small
\setlength{\tabcolsep}{5pt}
\caption{Synthetic representation ablation under the carry-over configuration. Each controller is compared across E, P, and F with the remaining settings fixed. Entries are means over repeated runs; runtime is per run.}
\label{tab:representation_ablation}
\label{tab:synthetic_main}
\begin{tabular}{llrrr}
\toprule
Scenario & Method & $\theta$-RMSE $\downarrow$ & Response RMSE $\downarrow$ & Time (s) $\downarrow$ \\
\midrule
\multicolumn{5}{l}{\textit{BOCPD controller}} \\
Drifting & B-BRPC-E & 0.014 & 0.484 & 39.402 \\
Drifting & B-BRPC-P & 0.014 & 0.408 & 88.171 \\
Drifting & B-BRPC-F & 0.015 & 0.484 & 37.490 \\
Sudden(3) & B-BRPC-E & 0.026 & 0.607 & 32.020 \\
Sudden(3) & B-BRPC-P & 0.028 & 0.590 & 73.600 \\
Sudden(3) & B-BRPC-F & 0.026 & 0.609 & 30.414 \\
Mixed(3) & B-BRPC-E & 0.021 & 0.464 & 32.857 \\
Mixed(3) & B-BRPC-P & 0.021 & 0.499 & 73.571 \\
Mixed(3) & B-BRPC-F & 0.022 & 0.471 & 30.606 \\
\midrule
\multicolumn{5}{l}{\textit{wCUSUM controller}} \\
Drifting & C-BRPC-E & 0.015 & 0.580 & 19.238 \\
Drifting & C-BRPC-P & 0.015 & 0.558 & 40.775 \\
Drifting & C-BRPC-F & 0.014 & 0.581 & 16.785 \\
Sudden(3) & C-BRPC-E & 0.027 & 0.671 & 16.051 \\
Sudden(3) & C-BRPC-P & 0.026 & 0.577 & 27.000 \\
Sudden(3) & C-BRPC-F & 0.027 & 0.673 & 13.858 \\
Mixed(3) & C-BRPC-E & 0.020 & 0.535 & 16.790 \\
Mixed(3) & C-BRPC-P & 0.016 & 0.494 & 25.215 \\
Mixed(3) & C-BRPC-F & 0.021 & 0.536 & 13.515 \\
\bottomrule
\end{tabular}
\end{table}

On the synthetic streams, F closely matches E under both controllers while
reducing runtime. P can improve response accuracy, particularly with wCUSUM,
but its runtime is higher. The plant comparison shows similar parameter
errors under drift, with larger differences after abrupt changes. Thus a
fixed support provides a useful cost trade-off, while the gains from proxy
observations depend on the controller and scenario.
\FloatBarrier

\begin{table}[!htbp]
\centering
\small
\setlength{\tabcolsep}{5pt}
\caption{Plant representation ablation under the carry-over configuration. Response errors are normalized by the mode-specific response scale. Entries are means over repeated runs; runtime is per run.}
\label{tab:plant_representation_ablation}
\label{tab:plant_ablation_appendix}
\begin{tabular}{llrrr}
\toprule
Scenario & Method & $\theta$-RMSE $\downarrow$ & Response RMSE $\downarrow$ & Time (s) $\downarrow$ \\
\midrule
\multicolumn{5}{l}{\textit{BOCPD controller}} \\
Drifting & B-BRPC-E & 0.800 & 0.988 & 46.714 \\
Drifting & B-BRPC-P & 0.802 & 0.960 & 113.778 \\
Drifting & B-BRPC-F & 0.801 & 0.980 & 35.584 \\
Sudden(5) & B-BRPC-E & 0.957 & 0.164 & 57.009 \\
Sudden(5) & B-BRPC-P & 1.309 & 0.163 & 118.882 \\
Sudden(5) & B-BRPC-F & 1.141 & 0.177 & 37.952 \\
Mixed & B-BRPC-E & 1.580 & 0.554 & 34.237 \\
Mixed & B-BRPC-P & 1.545 & 0.557 & 69.558 \\
Mixed & B-BRPC-F & 1.588 & 0.555 & 27.276 \\
\midrule
\multicolumn{5}{l}{\textit{wCUSUM controller}} \\
Drifting & C-BRPC-E & 0.809 & 1.001 & 24.304 \\
Drifting & C-BRPC-P & 0.813 & 0.991 & 46.636 \\
Drifting & C-BRPC-F & 0.807 & 1.020 & 17.438 \\
Sudden(5) & C-BRPC-E & 1.056 & 0.169 & 27.513 \\
Sudden(5) & C-BRPC-P & 1.024 & 0.171 & 67.723 \\
Sudden(5) & C-BRPC-F & 1.020 & 0.179 & 16.444 \\
Mixed & C-BRPC-E & 1.582 & 0.561 & 18.143 \\
Mixed & C-BRPC-P & 1.564 & 0.575 & 33.486 \\
Mixed & C-BRPC-F & 1.527 & 0.593 & 13.499 \\
\bottomrule
\end{tabular}
\end{table}

\FloatBarrier
\subsubsection{Shared versus Particle-Specific Discrepancy}
\label{app:restart_state_ablation}

Table~\ref{tab:fixed_particle_controller_ablation} uses the same carry-over
BRPC-F state to compare no restart, BOCPD, and wCUSUM for shared and
particle-specific discrepancy. The shared version updates one residual state
per expert; the particle-specific version maintains a conditional state for
each parameter particle.

\begin{table*}[t]
\centering
\small
\setlength{\tabcolsep}{4.0pt}
\caption{
Restart-mechanism and state-sharing ablation for BRPC-F on the synthetic benchmark. Values are mean $\pm$ standard deviation across runs.
Restart count is a restart-behavior diagnostic rather than a monotone
performance metric.
}
\label{tab:fixed_particle_controller_ablation}
\resizebox{\textwidth}{!}{%
\begin{tabular}{lllrrrr}
\toprule
Scenario & State & Controller
& $\theta$-RMSE $\downarrow$
& Response RMSE $\downarrow$
& Restarts
& Time (s) $\downarrow$ \\
\midrule

\multirow{6}{*}{Drifting}
& Shared & No restart
& $0.0145 \pm 0.0021$ & $1.1603 \pm 0.3362$ & $0.000 \pm 0.000$ & 15.662 \\
& Shared & BOCPD
& $0.0146 \pm 0.0021$ & $0.4839 \pm 0.2251$ & $11.758 \pm 1.755$ & 37.490 \\
& Shared & wCUSUM
& $0.0145 \pm 0.0022$ & $0.5809 \pm 0.2628$ & $4.774 \pm 0.420$ & 16.785 \\
& Particle & No restart
& $0.0145 \pm 0.0021$ & $1.1881 \pm 0.3370$ & $0.000 \pm 0.000$ & 18.100 \\
& Particle & BOCPD
& $0.0145 \pm 0.0021$ & $0.5141 \pm 0.2451$ & $11.260 \pm 1.639$ & 42.830 \\
& Particle & wCUSUM
& $0.0144 \pm 0.0022$ & $0.6040 \pm 0.2703$ & $4.774 \pm 0.420$ & 19.567 \\

\midrule
\multirow{6}{*}{Sudden}
& Shared & No restart
& $0.0401 \pm 0.0342$ & $1.1915 \pm 0.3850$ & $0.000 \pm 0.000$ & 12.847 \\
& Shared & BOCPD
& $0.0264 \pm 0.0235$ & $0.6089 \pm 0.2873$ & $10.322 \pm 4.138$ & 30.414 \\
& Shared & wCUSUM
& $0.0269 \pm 0.0245$ & $0.6727 \pm 0.3222$ & $3.746 \pm 1.243$ & 13.858 \\
& Particle & No restart
& $0.0405 \pm 0.0343$ & $1.2092 \pm 0.3909$ & $0.000 \pm 0.000$ & 15.017 \\
& Particle & BOCPD
& $0.0272 \pm 0.0224$ & $0.6504 \pm 0.2934$ & $9.310 \pm 3.972$ & 36.592 \\
& Particle & wCUSUM
& $0.0276 \pm 0.0247$ & $0.6857 \pm 0.3302$ & $3.782 \pm 1.217$ & 15.969 \\

\midrule
\multirow{6}{*}{Mixed}
& Shared & No restart
& $0.0303 \pm 0.0220$ & $1.0319 \pm 0.2009$ & $0.000 \pm 0.000$ & 12.513 \\
& Shared & BOCPD
& $0.0215 \pm 0.0170$ & $0.4714 \pm 0.0899$ & $11.867 \pm 2.647$ & 30.606 \\
& Shared & wCUSUM
& $0.0211 \pm 0.0148$ & $0.5357 \pm 0.0979$ & $3.493 \pm 1.005$ & 13.515 \\
& Particle & No restart
& $0.0307 \pm 0.0222$ & $1.0485 \pm 0.2049$ & $0.000 \pm 0.000$ & 15.146 \\
& Particle & BOCPD
& $0.0219 \pm 0.0152$ & $0.4959 \pm 0.0723$ & $9.893 \pm 2.252$ & 39.409 \\
& Particle & wCUSUM
& $0.0206 \pm 0.0134$ & $0.5584 \pm 0.1159$ & $3.440 \pm 1.056$ & 16.175 \\

\bottomrule
\end{tabular}%
}
\end{table*}

Restart reduces response error in the sudden and mixed streams for both
state-sharing choices. Under pure drift, resets reduce the effect of
outdated residual information and act as memory control. Particle-specific states achieve similar
parameter error, slightly higher response error, and greater runtime throughout
this comparison. These results support the shared state as a computational
approximation for the main experiments.
\FloatBarrier
\end{crblock}

\begin{crblock}
\subsection{Discrepancy Dynamics and Model Sensitivity}
\label{app:discrepancy_dynamics}

\subsubsection{Gaussian--Markov Propagation}
\label{app:markov_discrepancy_k_sensitivity}

On a fixed support, the Gaussian--Markov discrepancy forecast is
\[
a_t=\rho_\delta m_{t-1},\qquad
P_t=\rho_\delta^2C_{t-1}+Q_{\delta,t}.
\]
We compare carry-over ($\rho_\delta=1$, $q_\delta=0$) with random-walk
discrepancy propagation ($\rho_\delta=1$, $q_\delta=10^{-2}$), both
special cases of \eqref{eq:discrepancy_dynamics}.
Innovation allows residual structure to evolve
between batches and changes the uncertainty entering restart evidence.
Table~\ref{tab:discrepancy_propagation} reports both configurations explicitly;
the random-walk rows use expanding support and reproduce the corresponding
synthetic B-BRPC and C-BRPC results in Table~\ref{tab:brpc_main}.

\begin{table*}[!htbp]
\centering
\scriptsize
\setlength{\tabcolsep}{4pt}
\caption{\protect\crnew{Carry-over versus random-walk discrepancy propagation on synthetic streams. Both use $\rho_\delta=1$; carry-over has $q_\delta=0$, while the random-walk rows use expanding support with $q_\delta=10^{-2}$, as in Table~\ref{tab:brpc_main}. Event metrics are reported only for scenarios with annotated abrupt changes.}}
\label{tab:discrepancy_propagation}
\label{tab:appendix_old_vs_markov_corrected}
\resizebox{\textwidth}{!}{%
\begin{tabular}{lllrrrrr}
\toprule
Scenario & Method & Propagation & $\theta$-RMSE $\downarrow$ & Response RMSE $\downarrow$ & Restarts & Precision@2 $\uparrow$ & Recall@2 $\uparrow$ \\
\midrule
Drifting
& B-BRPC & Carry-over & 0.014$\pm$0.002 & 0.484$\pm$0.225 & 11.774$\pm$1.733 & -- & -- \\
& B-BRPC-E & Random-walk, $q_\delta=10^{-2}$ & 0.015$\pm$0.003 & 0.423$\pm$0.199 & 3.192$\pm$1.786 & -- & -- \\
& C-BRPC & Carry-over & 0.015$\pm$0.002 & 0.580$\pm$0.262 & 4.774$\pm$0.419 & -- & -- \\
& C-BRPC-E & Random-walk, $q_\delta=10^{-2}$ & 0.015$\pm$0.002 & 0.465$\pm$0.234 & 1.680$\pm$1.005 & -- & -- \\
\midrule
Sudden(3)
& B-BRPC & Carry-over & 0.026$\pm$0.023 & 0.607$\pm$0.285 & 10.287$\pm$4.158 & 0.314 & 0.926 \\
& B-BRPC-E & Random-walk, $q_\delta=10^{-2}$ & 0.020$\pm$0.014 & 0.592$\pm$0.281 & 2.760$\pm$1.875 & 0.657 & 0.682 \\
& C-BRPC & Carry-over & 0.027$\pm$0.024 & 0.671$\pm$0.321 & 3.706$\pm$1.209 & 0.725 & 0.834 \\
& C-BRPC-E & Random-walk, $q_\delta=10^{-2}$ & 0.027$\pm$0.024 & 0.566$\pm$0.237 & 3.187$\pm$1.392 & 0.749 & 0.713 \\
\midrule
Mixed(3)
& B-BRPC & Carry-over & 0.021$\pm$0.016 & 0.464$\pm$0.070 & 11.865$\pm$2.675 & 0.156 & 0.960 \\
& B-BRPC-E & Random-walk, $q_\delta=10^{-2}$ & 0.019$\pm$0.012 & 0.436$\pm$0.059 & 2.547$\pm$1.166 & 0.709 & 0.847 \\
& C-BRPC & Carry-over & 0.020$\pm$0.014 & 0.535$\pm$0.096 & 3.493$\pm$1.005 & 0.392 & 0.720 \\
& C-BRPC-E & Random-walk, $q_\delta=10^{-2}$ & 0.017$\pm$0.010 & 0.441$\pm$0.079 & 3.240$\pm$0.984 & 0.468 & 0.780 \\
\bottomrule
\end{tabular}%
}
\end{table*}

For B-BRPC, adding discrepancy innovation reduces the mean restart count from 11.774 to
3.192 under drift and from 11.865 to 2.547 under mixed nonstationarity, while
improving response RMSE in both cases. In the sudden scenario, precision
increases but recall decreases, indicating a more conservative restart rule.
C-BRPC also improves response prediction across all three scenarios, with a
smaller change in restart frequency. Temporal propagation therefore affects
both local adaptation and the evidence seen by the controller.
\FloatBarrier

\subsubsection{Innovation and Forgetting}

Table~\ref{tab:discrepancy_innovation} varies $q_\delta$ at fixed
$\rho_\delta=1$ for B-BRPC-E and C-BRPC-E. The $q_\delta=10^{-2}$ rows
are the corresponding synthetic results in Table~\ref{tab:brpc_main};
the other two scales assess sensitivity around this setting.
Larger $q_\delta$ makes the propagated prior more diffuse;
smaller $q_\delta$ preserves more information from the previous discrepancy state.

\begin{table*}[!htbp]
\centering
\scriptsize
\setlength{\tabcolsep}{4pt}
\caption{\protect\crnew{Innovation-scale sensitivity for B-BRPC-E and C-BRPC-E with $\rho_\delta=1$ and $Q_{\delta,t}=q_\delta K_{S_tS_t}$. The $q_\delta=10^{-2}$ rows reproduce Table~\ref{tab:brpc_main}. Entries are mean $\pm$ standard deviation over 125 drift, 300 sudden, and 75 mixed runs, with 25 seeds per configuration. Event metrics are omitted for pure drift.}}
\label{tab:discrepancy_innovation}
\label{tab:appendix_markov_k_sensitivity}
\resizebox{\textwidth}{!}{%
\begin{tabular}{lllrrrrr}
\toprule
Scenario & Method & $q_\delta$ & $\theta$-RMSE $\downarrow$ & Response RMSE $\downarrow$ & Restarts & Precision@2 $\uparrow$ & Recall@2 $\uparrow$ \\
\midrule
Drifting
& B-BRPC-E & $3\times 10^{-3}$ & 0.015$\pm$0.002 & 0.450$\pm$0.225 & 7.192$\pm$2.790 & -- & -- \\
& B-BRPC-E & $10^{-2}$        & 0.015$\pm$0.003 & 0.423$\pm$0.199 & 3.192$\pm$1.786 & -- & -- \\
& B-BRPC-E & $3\times 10^{-2}$ & 0.015$\pm$0.003 & 0.394$\pm$0.166 & 1.008$\pm$1.377 & -- & -- \\
& C-BRPC-E & $3\times 10^{-3}$ & 0.015$\pm$0.002 & 0.489$\pm$0.244 & 1.440$\pm$0.971 & -- & -- \\
& C-BRPC-E & $10^{-2}$        & 0.015$\pm$0.002 & 0.465$\pm$0.234 & 1.680$\pm$1.005 & -- & -- \\
& C-BRPC-E & $3\times 10^{-2}$ & 0.015$\pm$0.002 & 0.431$\pm$0.207 & 0.888$\pm$0.775 & -- & -- \\
\midrule
Sudden(3)
& B-BRPC-E & $3\times 10^{-3}$ & 0.021$\pm$0.015 & 0.597$\pm$0.282 & 4.783$\pm$1.957 & 0.624 & 0.896 \\
& B-BRPC-E & $10^{-2}$        & 0.020$\pm$0.014 & 0.592$\pm$0.281 & 2.760$\pm$1.875 & 0.657 & 0.682 \\
& B-BRPC-E & $3\times 10^{-2}$ & 0.022$\pm$0.016 & 0.571$\pm$0.271 & 1.947$\pm$1.757 & 0.577 & 0.474 \\
& C-BRPC-E & $3\times 10^{-3}$ & 0.026$\pm$0.024 & 0.583$\pm$0.250 & 3.190$\pm$1.417 & 0.734 & 0.709 \\
& C-BRPC-E & $10^{-2}$        & 0.027$\pm$0.024 & 0.566$\pm$0.237 & 3.187$\pm$1.392 & 0.749 & 0.713 \\
& C-BRPC-E & $3\times 10^{-2}$ & 0.026$\pm$0.024 & 0.561$\pm$0.232 & 3.057$\pm$1.435 & 0.734 & 0.691 \\
\midrule
Mixed(3)
& B-BRPC-E & $3\times 10^{-3}$ & 0.021$\pm$0.014 & 0.438$\pm$0.044 & 5.227$\pm$1.813 & 0.348 & 0.880 \\
& B-BRPC-E & $10^{-2}$        & 0.019$\pm$0.012 & 0.436$\pm$0.059 & 2.547$\pm$1.166 & 0.709 & 0.847 \\
& B-BRPC-E & $3\times 10^{-2}$ & 0.019$\pm$0.012 & 0.434$\pm$0.059 & 1.400$\pm$0.805 & 0.704 & 0.607 \\
& C-BRPC-E & $3\times 10^{-3}$ & 0.017$\pm$0.009 & 0.447$\pm$0.083 & 3.267$\pm$1.018 & 0.456 & 0.767 \\
& C-BRPC-E & $10^{-2}$        & 0.017$\pm$0.010 & 0.441$\pm$0.079 & 3.240$\pm$0.984 & 0.468 & 0.780 \\
& C-BRPC-E & $3\times 10^{-2}$ & 0.017$\pm$0.009 & 0.434$\pm$0.060 & 3.107$\pm$0.953 & 0.488 & 0.780 \\
\bottomrule
\end{tabular}%
}
\end{table*}

For B-BRPC on mixed streams, increasing $q_\delta$ from $3\times10^{-3}$ to
$3\times10^{-2}$ reduces restarts from 5.227 to 1.400 and recall from
0.880 to 0.607, while precision increases from 0.348 to 0.704.
The middle value $q_\delta=10^{-2}$ balances these effects in this suite.
C-BRPC is less sensitive: mixed-scenario response RMSE stays between 0.434
and 0.447 and recall between 0.767 and 0.780.

The forgetting diagnostic in Table~\ref{tab:brpc_robustness}(a) uses a
stationary-forgetting configuration of the same model,
\[
0<\rho_\delta<1,\qquad q_\delta=1-\rho_\delta^2,\qquad
Q_{\delta,t}=q_\delta K_{S_tS_t}.
\]
This couples mean reversion with stationary marginal covariance $K_{S_tS_t}$ on a fixed set of inputs, whereas
$\rho_\delta=1$, $Q_{\delta,t}=q_\delta K_{S_tS_t}$ is the random-walk member of the same Gaussian--Markov family, without mean
reversion. Consequently, $q_\delta$ and $1-\rho_\delta$ are not interchangeable
measures of memory. With an almost persistent state, pure-drift restarts
discard stale residual information and improve prediction; with moderate
forgetting, no-restart BRPC already adapts and additional resets slightly
increase response error. The three settings are reported in the main table
and are not repeated here.
\FloatBarrier

\subsubsection{Model Hyperparameter Sensitivity}
\label{app:cr_model_sensitivity}

An independent one-at-a-time model-hyperparameter sweep uses five seeds
on sudden and mixed synthetic streams for B-BRPC-E and C-BRPC-E. Starting
from $\eta_\theta=2$, $\sigma_\theta=0.1$, GP lengthscale $1$, and GP
variance $0.01$, we vary
\[
\begin{aligned}
\eta_\theta&\in\{0.5,1,2,4\},&
\sigma_\theta&\in\{0.05,0.1,0.2\},\\
\text{lengthscale multiplier}&\in\{0.5,1,2\},&
\sigma_\delta^2&\in\{0.01,0.1\}.
\end{aligned}
\]
Other model settings remain fixed within each sweep. The $\theta$-RMSE ranges
are $0.014$--$0.022$ for likelihood tempering, $0.014$--$0.020$ for each
kernel sweep, and $0.014$--$0.042$ for parameter-transition noise, which has
the largest tested effect. These ranges describe sensitivity over the
specified grids.
\FloatBarrier
\end{crblock}

\begin{crblock}
\subsection{Restart Diagnostics and Controller Sensitivity}
\label{app:restart_diagnostics}

These diagnostics use the carry-over discrepancy configuration. They isolate
controller behavior within that suite; the Gaussian--Markov comparison in
Appendix~\ref{app:discrepancy_dynamics} separately shows how propagation
uncertainty changes restart evidence.

\subsubsection{Detection Sensitivity to Drift and Jump Magnitude}
\label{app:restart_count_scaling}

\begin{figure}[ht]
     \centering
     \begin{subfigure}[b]{0.48\textwidth}
         \centering
         \includegraphics[width=\textwidth]{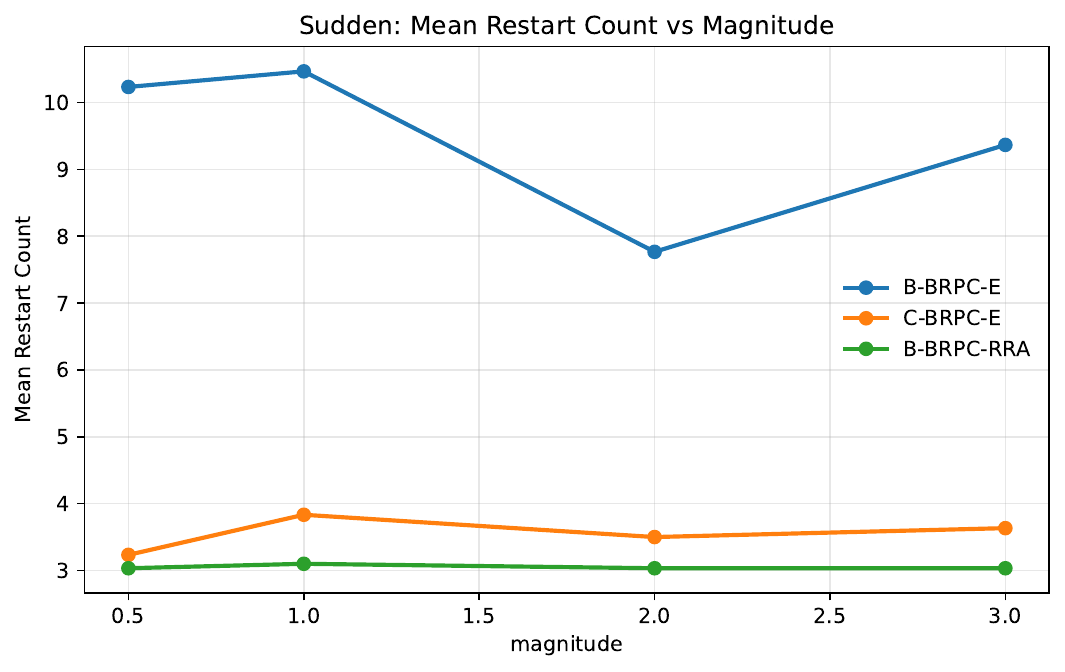}
         \caption{Restart count vs. jump magnitude}
         \label{fig:restart_vs_magnitude}
     \end{subfigure}%
     \hfill
     \begin{subfigure}[b]{0.48\textwidth}
         \centering
         \includegraphics[width=\textwidth]{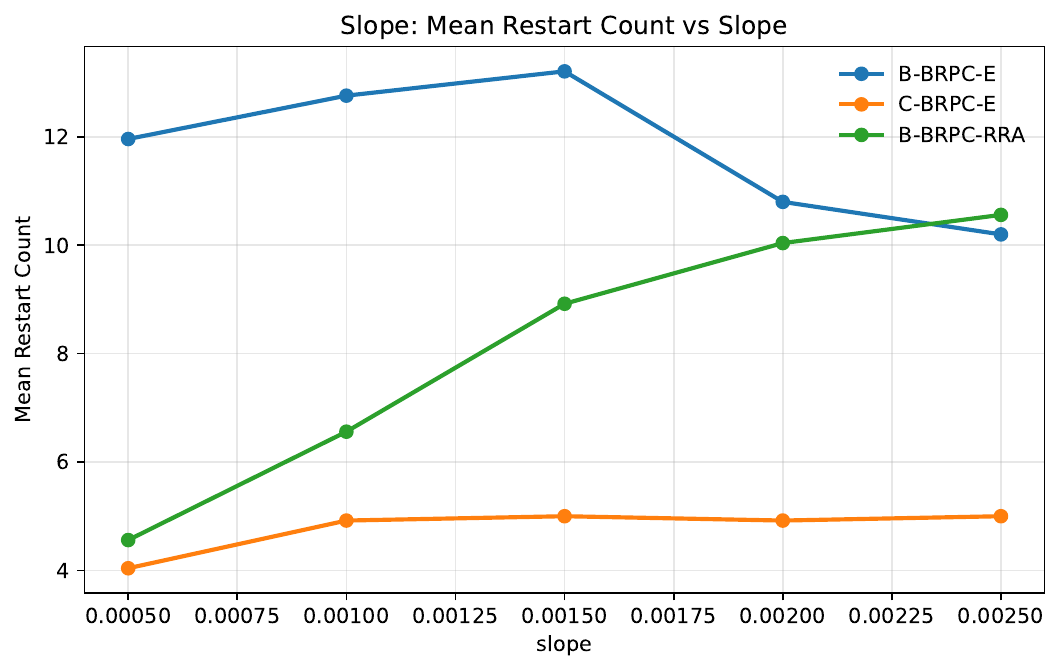}
         \caption{Restart count vs. drifting slope}
         \label{fig:restart_vs_slope}
     \end{subfigure}

     \caption{Restart count as a function of jump magnitude and drifting slope in the synthetic benchmark. Each point represents the mean restart count over repeated runs for a given scenario configuration.}
     \label{fig:three_graphs}
\end{figure}

Figure~\ref{fig:three_graphs} varies jump magnitude and drift slope.
With carry-over discrepancy, B-BRPC-E restarts more often than the three
annotated sudden changes, whereas C-BRPC-E and B-BRPC-RRA stay closer to
that count over the tested magnitudes. Under pure drift, the increasing
restart count of RRA reflects shorter segment memory as residual targets
change faster. Every such restart is a false alarm under changepoint metrics,
even when it improves adaptation; counts alone do not measure detection quality.
\FloatBarrier

\subsubsection{Restart Hyperparameter Robustness}
\label{app:hyper_sensitivity}

For B-BRPC-E, we vary hazard scale, hard-restart margin, and cooldown:
\[
\begin{aligned}
(\lambda_h,\rho_B,c_{\mathrm{cool}})\in\{&
(200,1,10),(400,2,20),\\
&(800,4,20),(1600,4,30)\}.
\end{aligned}
\]
For C-BRPC-E, we vary window size, threshold, drift allowance, and score floor:
\[
\begin{aligned}
(W,h_{\mathrm C},\kappa,\sigma_{\min})\in\{&
(4,0.25,0.25,0.25),(4,0.50,0.25,0.25),\\
&(8,0.50,0.50,0.50),(8,1.00,0.50,0.50)\}.
\end{aligned}
\]
All remaining settings are fixed. Table~\ref{tab:controller_hyper_sensitivity}
reports ranges across these configurations. Event metrics use the two-batch
post-change tolerance defined in Appendix~\ref{app:common_experimental_settings}.

\begin{table}[t]
\centering
\small
\setlength{\tabcolsep}{5pt}
\caption{
Restart-mechanism hyperparameter sensitivity for B-BRPC-E and C-BRPC-E on the
representative synthetic benchmark. Each entry reports the range across the
tested restart-rule hyperparameter grid listed in the text. Event-level metrics use
a two-batch tolerance window and do not reward early detections. Restart count is
a restart-behavior diagnostic.
}
\label{tab:controller_hyper_sensitivity}
\resizebox{\linewidth}{!}{%
\begin{tabular}{llccccc}
\toprule
Scenario & Method
& Restarts
& Precision@2 $\uparrow$
& Recall@2 $\uparrow$
& F1@2 $\uparrow$
& Delay@2 (batches) \\
\midrule
\multirow{2}{*}{Sudden}
& B-BRPC-E
& $9.6$--$9.8$
& $0.286$--$0.293$
& $0.933$
& $0.436$--$0.444$
& $0.10$--$0.17$ \\
& C-BRPC-E
& $3.0$--$3.2$
& $0.700$--$0.867$
& $0.733$--$0.867$
& $0.714$--$0.867$
& $0.00$ \\
\midrule
\multirow{2}{*}{Mixed}
& B-BRPC-E
& $12.6$--$13.0$
& $0.154$--$0.159$
& $1.000$
& $0.267$--$0.275$
& $0.20$--$0.30$ \\
& C-BRPC-E
& $3.2$--$3.4$
& $0.533$--$0.600$
& $0.900$--$1.000$
& $0.667$--$0.747$
& $0.10$--$0.30$ \\
\bottomrule
\end{tabular}%
}
\end{table}

Within this carry-over suite, B-BRPC-E retains high recall but low precision
across the tested settings. C-BRPC-E uses fewer restarts and achieves higher
F1 while retaining competitive recall. Excess restarts persist across the
tested B-BRPC-E hazard and cooldown settings, linking this behavior to the
combined propagation and controller configuration.
\FloatBarrier

\subsubsection{Representative Tracking Trajectories}
\label{app:representative_trajectories}

Figure~\ref{fig:representative_tracking} retains one trajectory from each
benchmark: synthetic drift and plant sudden changes. They illustrate local
tracking and adaptation after jumps under the carry-over configuration;
aggregate tables provide the comparisons across runs. Extra resets can
fragment otherwise accurate tracking, while segment-local residual re-anchoring
can help relocation after a jump. The aggregate tables quantify these
effects across runs.


\begin{figure}[!htbp]
\centering

\begin{subfigure}[t]{0.48\linewidth}
\centering
\includegraphics[width=\linewidth]{
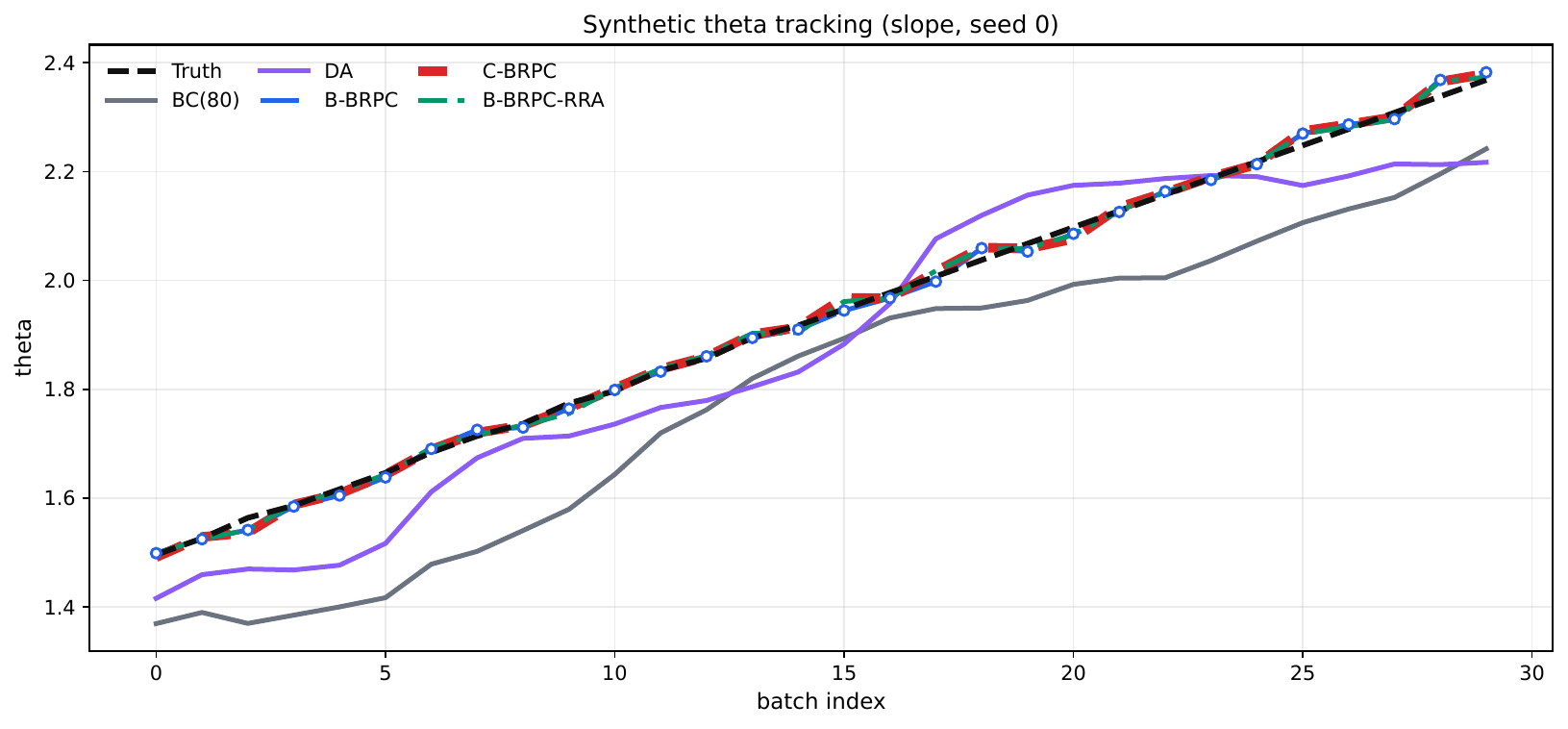}
\caption{Synthetic: gradual drift.}
\label{fig:synthetic_drift_tracking}
\end{subfigure}
\hfill
\begin{subfigure}[t]{0.48\linewidth}
\centering
\includegraphics[width=\linewidth]{
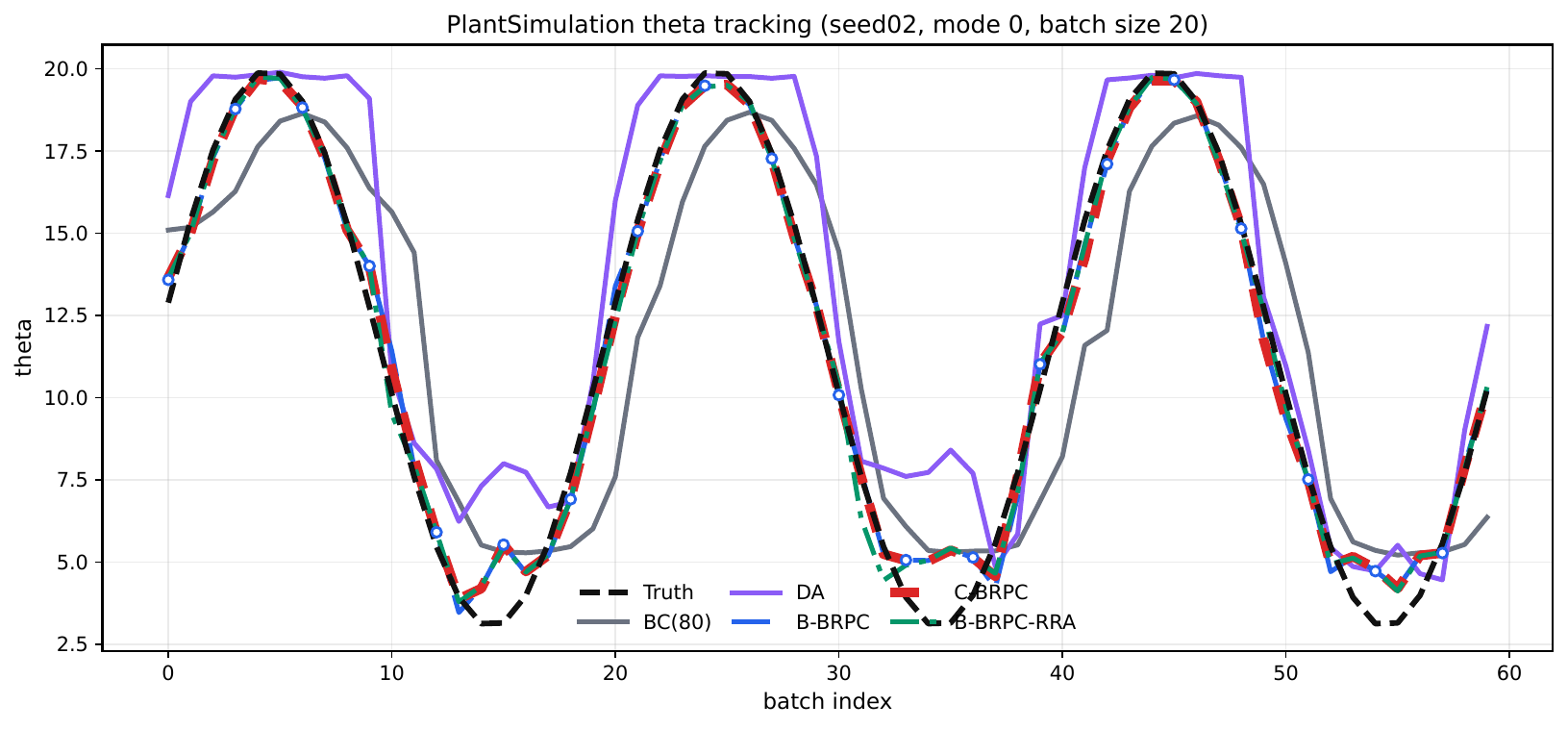}
\caption{Plant: gradual drift.}
\label{fig:plant_drift_tracking}
\end{subfigure}

\par\medskip

\begin{subfigure}[t]{0.48\linewidth}
\centering
\includegraphics[width=\linewidth]{
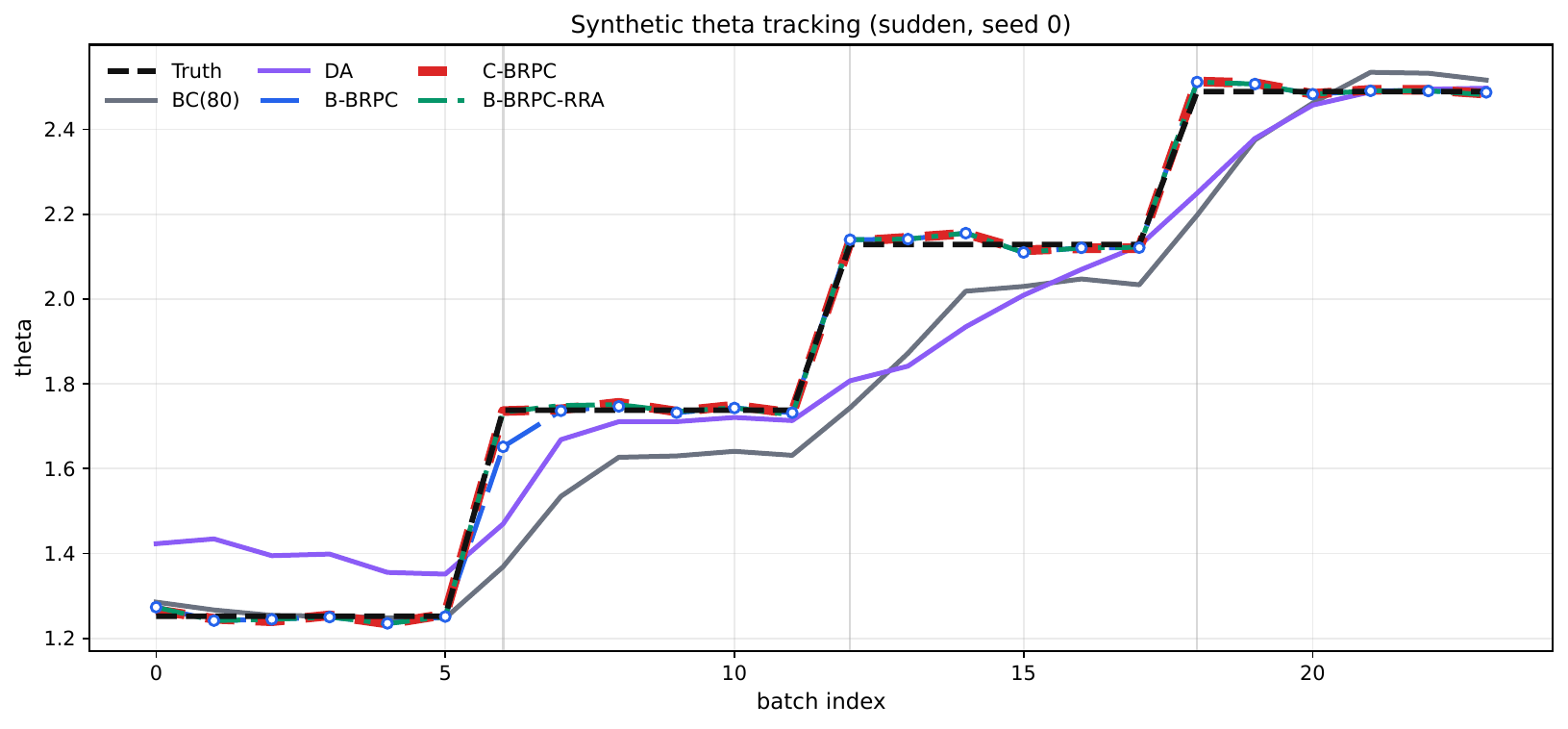}
\caption{Synthetic: sudden changes.}
\label{fig:synthetic_sudden_tracking}
\end{subfigure}
\hfill
\begin{subfigure}[t]{0.48\linewidth}
\centering
\includegraphics[width=\linewidth]{
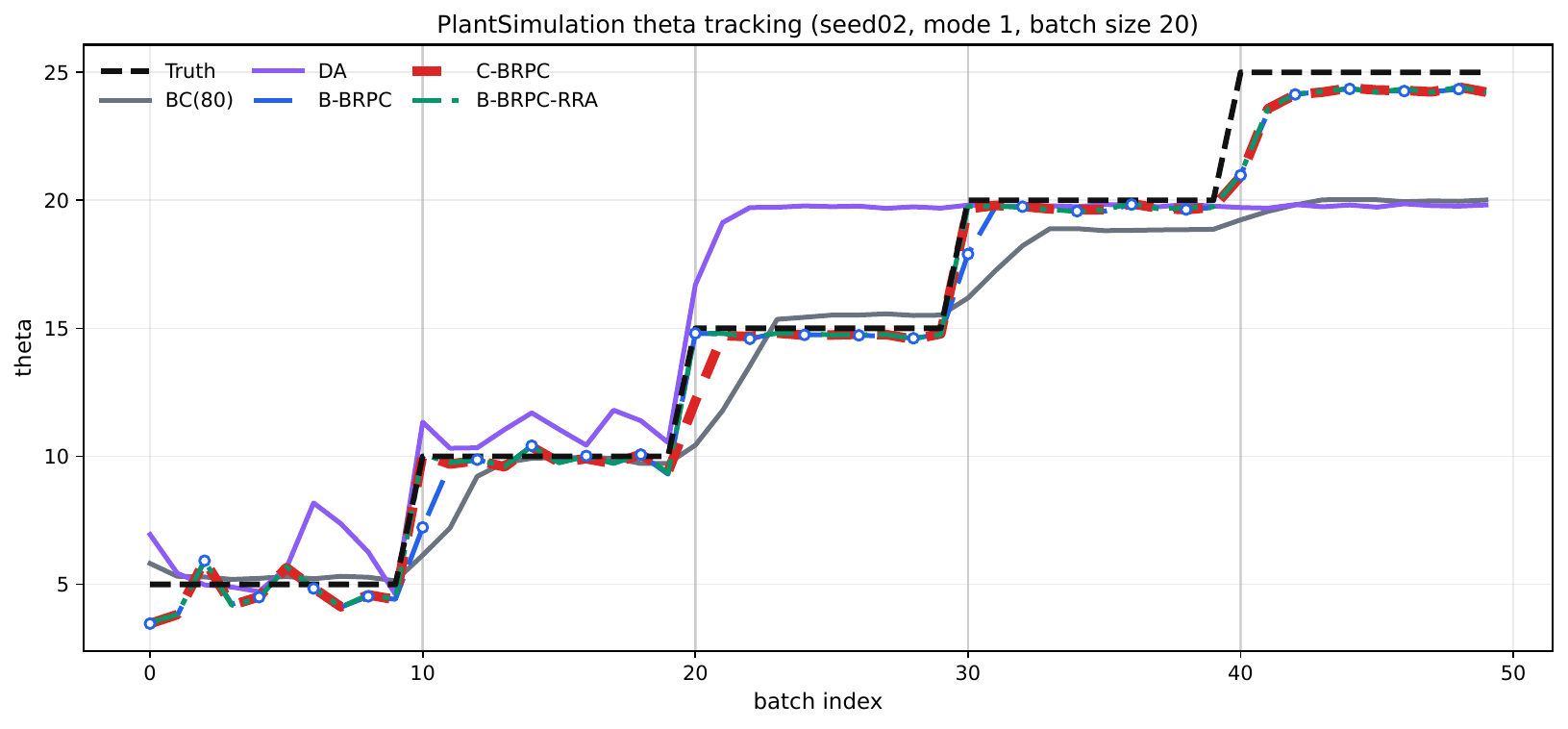}
\caption{Plant: sudden changes.}
\label{fig:plant_sudden_tracking}
\end{subfigure}

\par\medskip

\begin{subfigure}[t]{0.48\linewidth}
\centering
\includegraphics[width=\linewidth]{
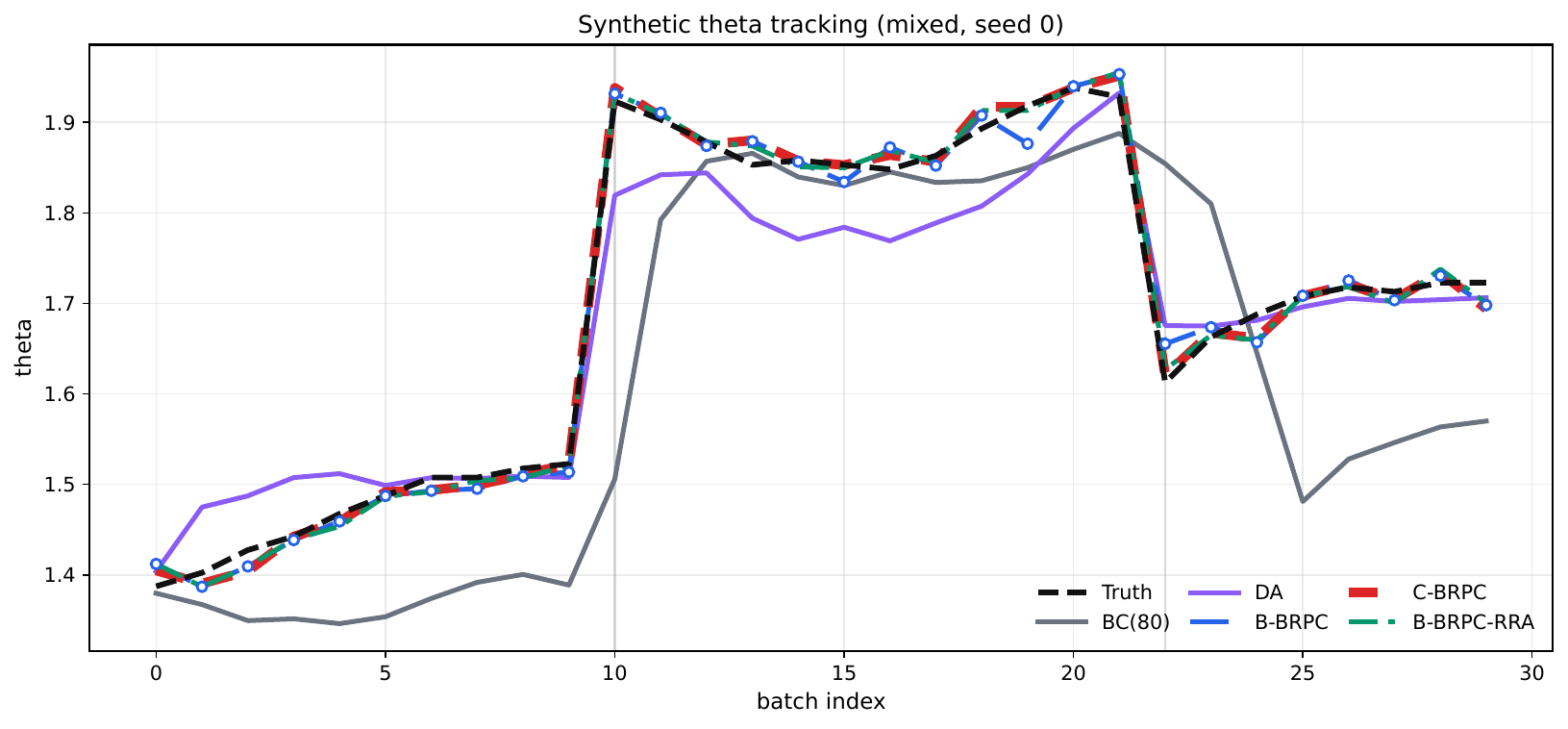}
\caption{Synthetic: mixed changes.}
\label{fig:synthetic_mixed_tracking}
\end{subfigure}
\hfill
\begin{subfigure}[t]{0.48\linewidth}
\centering
\includegraphics[width=\linewidth]{
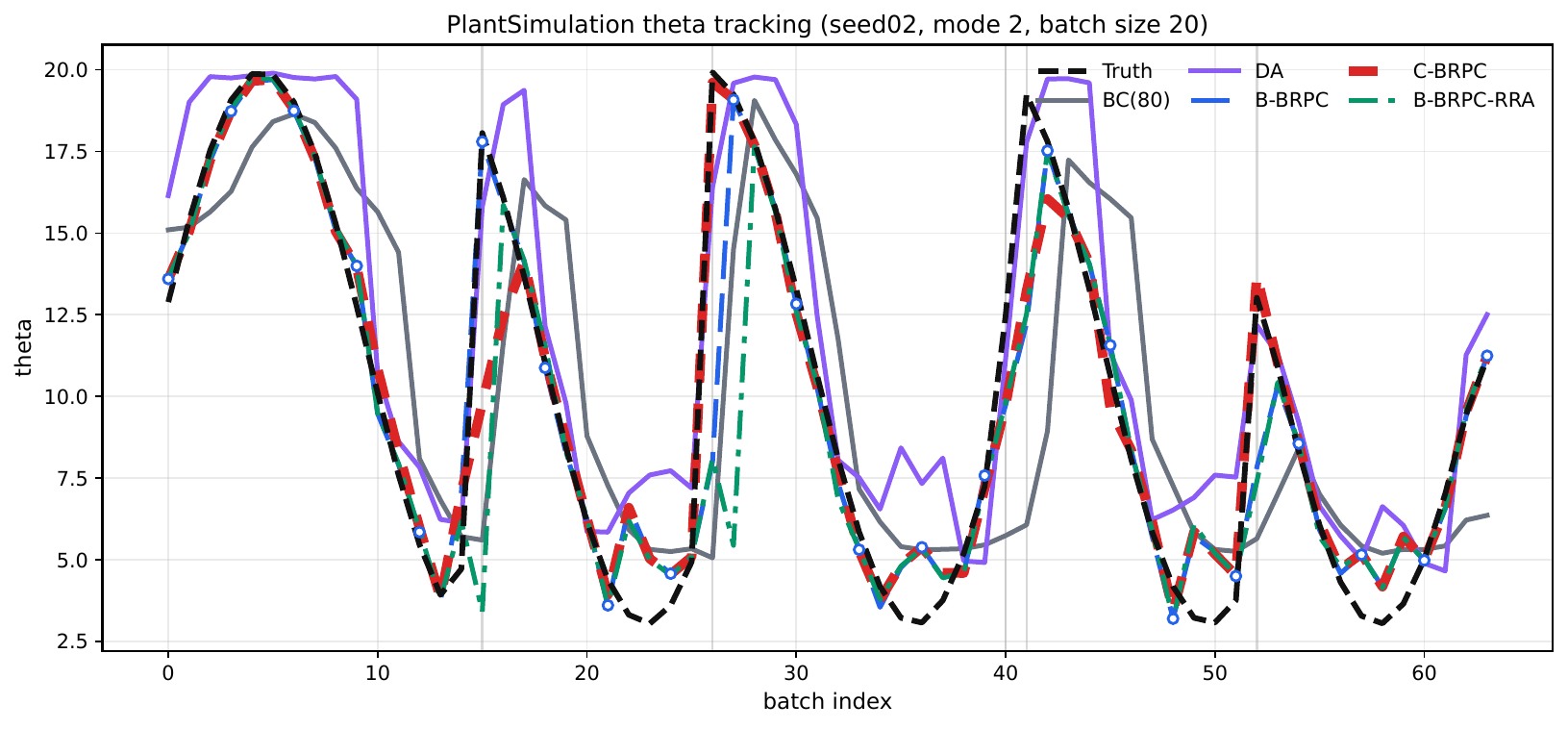}
\caption{Plant: mixed changes.}
\label{fig:plant_mixed_tracking}
\end{subfigure}

\caption{Representative parameter-tracking trajectories
across gradual drift, sudden changes, and mixed
nonstationarity on the synthetic (left) and
plant-simulation (right) benchmarks.
Solid curves show the true calibration targets
and estimated parameter trajectories; dashed
vertical lines indicate annotated changepoints
where applicable. Synthetic and plant results
use seeds 00 and 02, respectively.}
\label{fig:representative_tracking}

\end{figure}

\begin{table}[!htbp]
\centering
\caption{
Discrepancy sensitivity: (a) forgetting under drift
($q_\delta=1-\rho_\delta^2$); (b) B-BRPC-E innovation under mixed changes
($\rho_\delta=1$; main $q_\delta=10^{-2}$).
See Appendix~\ref{app:discrepancy_dynamics}.
}
\label{tab:brpc_robustness}

{
\begingroup
\footnotesize
\renewcommand{\arraystretch}{1.10}
\setlength{\tabcolsep}{3pt}

\begin{minipage}[t]{0.48\linewidth}
\vspace{0pt}
\centering

\textbf{(a) Forgetting under drift}
\vspace{3pt}

\begin{tabular}{@{}lrrr@{}}
\toprule
& \multicolumn{2}{c}{Response RMSE $\downarrow$}
& \\
\cmidrule(lr){2-3}
$1-\rho_\delta$
& No restart & Restart & \# Restarts \\
\midrule

$10^{-5}$
& 0.995 & 0.459 & 11.8 \\

$0.03$
& 0.347 & 0.359 & 1.0 \\

$0.10$
& 0.342 & 0.364 & 0.6 \\

\bottomrule
\end{tabular}

\end{minipage}
\hfill
\begin{minipage}[t]{0.48\linewidth}
\vspace{0pt}
\centering

\textbf{(b) Innovation under mixed changes}
\vspace{3pt}

\begin{tabular}{@{}lrrr@{}}
\toprule
$q_\delta$ & Restarts & Precision@2 $\uparrow$ & Recall@2 $\uparrow$ \\
\midrule

$3\times10^{-3}$
& 5.227 & 0.348 & 0.880 \\

$10^{-2}$
& 2.547 & \textbf{0.709} & 0.847 \\

$3\times10^{-2}$
& 1.400 & 0.704 & 0.607 \\

\bottomrule
\end{tabular}

\end{minipage}

\endgroup
}
\end{table}

\FloatBarrier
\end{crblock}

\begin{crblock}
\subsection{Scalability and High-Dimensional Tracking}
\label{app:scalability}

\subsubsection{Computational Complexity and Runtime}
\label{app:runtime}

The main computational costs separate into the projected parameter update and
the discrepancy update. With $N$ particles and calibration dimension
$d_\theta$, the projected parameter update requires propagating and scoring the
particles, typically $O(N)$ simulator or surrogate evaluations per batch, plus
resampling when the effective sample size is small.

The discrepancy-update cost depends on the chosen state representation. BRPC-E
uses an expanding support and is appropriate for moderate segment lengths.
BRPC-F uses a fixed support of size $M$; its cost is dominated by fixed-rank
linear algebra, such as $O(M^3)$ initialization and $O(K_tM^2)$ batch updates
under the standard fixed-support representation, where $K_t$ is the batch size.
BRPC-P has additional cost from constructing and maintaining the proxy-observation
representation. BRPC-RRA is heavier because it refits the discrepancy model on
the current segment. Let $n_e$ denote the number of residual observations retained in the current
segment for expert $e$. An exact RRA refit costs $O(n_e^3)$, while a fixed-rank
refit costs $O(n_eM^2+M^3)$.

\begin{table}[!htbp]
\centering
\small
\caption{Representative per-run runtimes for synthetic drift in the
carry-over suite. BRPC rows use shared discrepancy.}
\label{tab:runtime_comparison}
\begin{tabular}{lrrrrr}
\toprule
Method & EnKF & C-BRPC-F & C-BRPC-E & B-BRPC-E & B-BRPC-RRA \\
\midrule
Time (s) & 0.236 & 16.785 & 19.238 & 39.402 & 70.667 \\
\bottomrule
\end{tabular}
\end{table}

The runtimes show the cost of particle-based calibration and discrepancy
inference relative to ensemble filtering. Fixed-support C-BRPC-F is a practical
choice among the restarted BRPC variants in this suite. RRA can be useful on
jump-dominated streams, but its segment refits increase computation. These
timings characterize the low-dimensional benchmark. Particle coverage
becomes an additional bottleneck as the calibration dimension grows.
\FloatBarrier

\subsubsection{High-Dimensional Projected Calibration}
\label{app:highdim_physical_projected}

This section reports a high-dimensional diagnostic designed to separate three
effects: projected-target tracking, point prediction, and the effect of applying
a hard restart rule. \crnew{The physical response comes from a nonlinear function class distinct
from the simulator, and the calibration target is its \(L^2\) projection onto
the simulator family. This extends the low-dimensional benchmark to a
five-parameter simulator and a higher-dimensional input space.} The
diagnostic tests projected-target tracking under this distinct physical model.

Let \(x\in[0,1]^{20}\) and let the simulator have \(d_\theta=5\) calibration
parameters. For each time \(t\), the projected calibration target is
\[
\theta_t^\dagger
=
\arg\min_\theta
\mathbb{E}_{X}
\left[
\zeta_t(X)-y_s(X,\theta)
\right]^2,
\]
where \(\zeta_t\) is the physical response and \(y_s(\cdot,\theta)\) is the
misspecified simulator. In the experiment, \(\theta_t^\dagger\) is computed on
a large Sobol reference design, and all reported parameter errors are measured
against the projected target \(\theta_t^\dagger\).

We compare four methods. \textbf{PF-DA} is a particle-filter
generalization with particles over calibration parameters and kernel length
scale~\citep{ward2021continuous}. \textbf{BC} is a short-window sliding Bayesian calibration baseline with
a four-batch window. \textbf{BRPC-F} is the fixed-support BRPC state without
hard restart. \textbf{C-BRPC-F} uses the same fixed-support BRPC state
with a hard wCUSUM restart rule. We use \(K_t=32\), \(N=1024\) particles,
\(M=128\) fixed support points, \(\eta_\delta=1\) (unit discrepancy KL weight), and five random seeds.

\paragraph{Physical-projected data generation.}
The input is \(x\in[0,1]^{20}\), and the simulator has \(d_\theta=5\)
calibration parameters:
\[
y_s(x,\theta)
=
\sum_{j=1}^{5}\theta_j\phi_j(x),
\]
where
\[
\phi_j(x)
=
\sin(2\pi x_j)
+
0.5\cos(2\pi x_{j+5})
+
0.25\sin\{2\pi(x_j+x_{j+10})\}.
\]
The physical response is generated from a different nonlinear function family:
\[
\zeta_t(x)
=
\sum_{j=1}^{5} a_{t,j} g_j(x)
+
0.4 h(x),
\]
where
\[
g_j(x)
=
\sin(2\pi x_j+0.4x_{j+5})
+
0.4\cos(2\pi x_{j+10})
+
0.3x_{j+15}x_j .
\]

The term \(h(x)\) creates locally heterogeneous residual structure:
\[
h(x)
=
\left[
1+
1.5\exp\left(
-\frac{\|x_{1:3}-c\|^2}{2s^2}
\right)
\right]
\sum_{\ell=1}^{L}
\alpha_\ell \sqrt{2}\cos(\omega_\ell^\top x+b_\ell).
\]
Here $a_t\in\mathbb R^5$ is the time-varying physical coefficient vector,
$v_1,v_2,u_1,u_2\in\mathbb R^5$ are fixed direction vectors used to generate
nonstationarity, and $\tau_r$ denotes the $r$th jump time. The vectors
$\omega_\ell\in\mathbb R^{20}$, phases $b_\ell\in\mathbb R$, and coefficients
$\alpha_\ell\in\mathbb R$ are fixed random Fourier-feature parameters. The
constant $\lambda$ in the ridge projection is unrelated to the restart-hazard
scale $\lambda_h$ used elsewhere. Observations are generated as
\[
Y_{t,k}
=
\zeta_t(x_{t,k})
+
\epsilon_{t,k},
\qquad
\epsilon_{t,k}\sim N(0,\sigma^2).
\]

The latent coefficient \(a_t\) controls physical nonstationarity; the
calibration ground truth is its induced projected target. 
Because \(y_s\) is linear in \(\theta\), this target can be computed on a large
reference design. Let \(X_{\rm ref}=\{x_m\}_{m=1}^{M_{\rm ref}}\) be a Sobol
reference design with \(M_{\rm ref}=50{,}000\), and define
\[
\Phi_{mj}=\phi_j(x_m),
\]
and
\[
\zeta_t^{\rm ref}
=
\left(
\zeta_t(x_1),\ldots,\zeta_t(x_{M_{\rm ref}})
\right)^\top .
\]
We compute
\[
\widehat\theta_t^\dagger
=
(\Phi^\top\Phi+\lambda I)^{-1}
\Phi^\top \zeta_t^{\rm ref},
\]
with a small ridge \(\lambda=10^{-6}\). All reported
\(\theta\)-RMSE values in this diagnostic use \(\widehat\theta_t^\dagger\).

\paragraph{Nonstationary physical coefficients.}
The time variation is applied to the physical coefficient \(a_t\). In the
drifting scenario,
\[
a_t
=
a_0
+
0.8\frac{t}{T}v_1
+
0.15\sin\left(\frac{2\pi t}{T}\right)v_2 .
\]

In the sudden scenario, \(a_t\) is piecewise constant with changepoints at
\(0.25T\), \(0.50T\), and \(0.75T\), and jump magnitude \(1.2\). In the mixed
scenario,
\[
a_t
=
a_0
+
0.4\frac{t}{T}v_1
+
0.10\sin\left(\frac{2\pi t}{T}\right)v_2
+
\sum_{r=1}^{2}
u_r\mathbf{1}\{t\ge \tau_r\},
\]
with jump magnitude \(1.0\). The projected target
\(\widehat\theta_t^\dagger\) is recomputed for each batch from the corresponding
physical response \(\zeta_t\).

\begin{table}[t]
\centering
\small
\setlength{\tabcolsep}{4.5pt}
\caption{
Physical-projected high-dimensional diagnostic with \(d_x=20\) and
\(d_\theta=5\). Parameter error is measured against the L2 projected target
\(\theta_t^\dagger\). Values are mean \(\pm\) standard deviation across five
seeds. F1@2 is reported only for scenarios with annotated discrete changes.
}
\label{tab:highdim_physical_projected}
\resizebox{\linewidth}{!}{%
\begin{tabular}{llccccc}
\toprule
Scenario & Method
& $\theta$-RMSE $\downarrow$
& Response RMSE $\downarrow$
& Response CRPS $\downarrow$
& Restarts
& F1@2 \(\uparrow\) \\
\midrule

\multirow{4}{*}{Drifting}
& PF-DA
& \(1.0137 \pm 0.2174\)
& \(1.5027 \pm 0.0228\)
& \(0.8834\)
& --
& -- \\
& BC(4)
& \(0.1845 \pm 0.0168\)
& \(\mathbf{1.0828 \pm 0.0396}\)
& \(0.8066\)
& --
& -- \\
& BRPC-F
& \(\mathbf{0.1329 \pm 0.0060}\)
& \(1.1213 \pm 0.1033\)
& \(\mathbf{0.7402}\)
& --
& -- \\
& C-BRPC-F
& \(0.2038 \pm 0.0193\)
& \(1.1830 \pm 0.0264\)
& \(0.7604\)
& \(3.8\)
& -- \\

\midrule
\multirow{4}{*}{Sudden}
& PF-DA
& \(1.0506 \pm 0.1742\)
& \(1.5483 \pm 0.1604\)
& \(0.9231\)
& --
& -- \\
& BC(4)
& \(0.2073 \pm 0.0297\)
& \(\mathbf{1.1223 \pm 0.1157}\)
& \(0.8420\)
& --
& -- \\
& BRPC-F
& \(\mathbf{0.1631 \pm 0.0107}\)
& \(1.2173 \pm 0.1599\)
& \(0.7932\)
& --
& -- \\
& C-BRPC-F
& \(0.2030 \pm 0.0233\)
& \(1.1790 \pm 0.1512\)
& \(\mathbf{0.7426}\)
& \(3.4\)
& \(0.303\) \\

\midrule
\multirow{4}{*}{Mixed}
& PF-DA
& \(0.9893 \pm 0.1943\)
& \(1.4887 \pm 0.0668\)
& \(0.8781\)
& --
& -- \\
& BC(4)
& \(0.1890 \pm 0.0126\)
& \(\mathbf{1.0738 \pm 0.0439}\)
& \(0.8006\)
& --
& -- \\
& BRPC-F
& \(\mathbf{0.1555 \pm 0.0254}\)
& \(1.2031 \pm 0.0378\)
& \(0.7746\)
& --
& -- \\
& C-BRPC-F
& \(0.1951 \pm 0.0395\)
& \(1.1782 \pm 0.0454\)
& \(\mathbf{0.7551}\)
& \(3.8\)
& \(0.297\) \\

\bottomrule
\end{tabular}%
}
\end{table}

Table~\ref{tab:highdim_physical_projected} separates projected-target
tracking, point prediction, and probabilistic prediction. BRPC-F has the
lowest projected-target RMSE in all three scenarios, whereas BC(4) has the
lowest point response RMSE. BRPC-F or C-BRPC-F achieves the lowest CRPS.
PF-DA has substantially larger parameter and predictive errors in this diagnostic.

Hard wCUSUM restarts increase projected-target error in all scenarios,
including pure drift. They improve response RMSE relative to no-restart BRPC-F
after abrupt changes, with partial recovery of the annotated changes. In high-dimensional inputs, varying covariate coverage and finite-batch
score variability can resemble regime changes. With $K_t=32$, the local
update can also adapt using substantial information within a batch. Thus hard
restart trades retained information and tracking variance against
responsiveness to abrupt changes.
\FloatBarrier
\end{crblock}

\begin{crblock}
\subsection{Closed-Loop Decision Study}
\label{app:cr_closed_loop}

This study evaluates parameter synchronization when the calibrated twin also
selects operating settings in the plant environment of
Appendix~\ref{app:plant_benchmark}. Each of ten seeds spans 80 batches.
The operating settings are integer-valued: $Q,R\in\{50,\ldots,100\}$,
$W\in\{10,\ldots,15\}$, and $M_1,M_2\in\{1,2,3\}$.
Here $\theta$ is the mean customer inter-arrival time in minutes. It follows
$3\to11.5\to20\to30$ in blocks of 20 batches, with independent
$\mathcal N(0,0.5^2)$ per-batch jitter clipped to $[0.1,30]$.

For this study, an MLP with hidden widths $(128,64,32)$ is trained on
2,000 Plant Simulation runs, using standardized inputs and a standardized
signed-log revenue response. Its held-out $R^2$ is $0.963$ on 400 runs.
The digital twin uses the trained MLP surrogate, while the
physical plant is represented by the original Plant Simulation
model. 

The calibrator is B-BRPC with 512 particles, a
uniform prior on $[1,30]$. The hazard is $h(r)=1/(200+r)$ for run length $r$ measured
in observations, with at most five experts. Restart occurs when a newer
expert has more than twice the anchor's posterior mass, subject to a
10-observation cooldown.

At each batch, the synchronized twin applies the five operating settings
$X=(Q,R,W,M_1,M_2)$ with the largest UCB scores,
\[
\mu_t(X)+\beta\sigma_t(X),\qquad\beta=0.3,
\]
where $\mu_t$ and $\sigma_t$ are obtained from the pre-batch expert--particle
predictive mixture, mapped to revenue units by the inverse response transform
and the delta method, respectively. The candidate pool contains 4,096
uniformly sampled settings, up to 1,024 local perturbations of the previous
top five, and those five settings themselves. The fixed policy chooses five
settings once from a 4,096-setting pool at the nominal $\theta=11.5$ and
retains them throughout. The synchronized policy also uses $\theta=11.5$
in the first batch. Both policies are evaluated at the current true $\theta$,
and all ten observations update the calibrator before the next decision.

Let $\bar y_t^{\mathrm{sync}}$ and $\bar y_t^{\mathrm{fix}}$ be each policy's
mean revenue over its five settings. The cumulative gain within each seed is
\[
100\,\frac{\sum_{t=1}^{80}\bar y_t^{\mathrm{sync}}
-\sum_{t=1}^{80}\bar y_t^{\mathrm{fix}}}
{\left|\sum_{t=1}^{80}\bar y_t^{\mathrm{fix}}\right|}.
\]
Table~\ref{tab:closed_loop} also reports the share of batches won by the
synchronized policy and parameter MAE before and after the current batch's
update. The decision-time estimate uses only observations from earlier
batches; the post-update estimate includes the current batch.

\begin{table}[!htbp]
\centering
\small
\caption{Closed-loop decision study on the plant-simulation benchmark: 80 batches and ten seeds. Entries are
means $\pm$ sample standard deviations across seeds. Parameter MAEs average
$|\theta_t-\widehat\theta_t|$ over batches, using the pre-batch or post-update
estimate as indicated.}
\label{tab:closed_loop}
\begin{tabular}{@{}lr@{}}
\toprule
Metric & Mean $\pm$ sample s.d. \\
\midrule
Cumulative net revenue gain (\%) & $18.3\pm1.5$ \\
Batch win rate (\%) & $92.6\pm2.6$ \\
Decision-time parameter MAE (min) & $1.45\pm0.03$ \\
Post-update parameter MAE (min) & $1.07\pm0.03$ \\
\bottomrule
\end{tabular}
\end{table}
\FloatBarrier

The revenue gain is positive in every seed, with a range of
$16.2\%$--$21.4\%$. Mean gains in the four regimes are
$48.9\%$, $0.3\%$, $4.7\%$, and $6.3\%$, respectively. The largest gain
occurs at $\theta\approx3$; near the nominal $\theta\approx11.5$, the fixed
policy is already well matched. At the upper prior boundary
$\theta\approx30$, jittered inter-arrival times are clipped at 30 and the
post-update parameter MAE is $1.72\pm0.16$ minutes. Every seed records
three restarts, each 1--3 batches after a regime switch (mean delay:
1.6 batches), with no within-regime restarts. 
\FloatBarrier
\end{crblock}

\begin{crblock}
\section{Relation to Online and Continual Learning}
\label{app:online_continual_learning}

BRPC combines sequential Bayesian calibration with the use and selective
removal of past information. Its connections to online and continual learning
help explain the form of the tracking bound and the role of restarting.

\paragraph{Online learning and dynamic regret.}
Online learning evaluates sequential updates through cumulative loss
\citep{zinkevich2003online,shalev2011online}. For changing environments,
dynamic regret compares a learner with a time-varying reference
\citep{hall2013dynamical,hall2015online}. Theorem~\ref{thm:brpc_tracking}
uses this structure: its discrepancy loss is compared with a reference path,
and the bound depends on that path's variation under the state transition.
The calibration model determines the loss and the order of inference.
BRPC first anchors the simulator through projected parameter estimation and
then learns discrepancy from the resulting residuals.

\paragraph{KL regularization and streaming inference.}
The recursive updates balance the current likelihood with KL divergence from
the forecast distribution. This has the proximal structure of mirror descent
\citep{beck2003mirror,duchi2010composite} and connects to streaming Bayesian
inference \citep{broderick2013streaming}. The forecast carries past information
through the parameter transition and the Gaussian--Markov discrepancy model.
Likelihood tempering then controls the strength of the new batch's update.

\paragraph{Concept drift and change-point detection.}
Concept-drift methods study gradual and abrupt distribution changes
\citep{gama2014survey,webb2016characterizing,lu2018learning}.
CUSUM and Bayesian change-point methods use score or predictive-likelihood
contrasts to detect such changes
\citep{page1954continuous,adams2007bayesian,fearnhead2007line,
saatcci2010gaussian,xie2023window}. B-BRPC and C-BRPC use these ideas with
evidence generated by the evolving calibration state. Its discrepancy dynamics
therefore affect detection: a concentrated forecast may produce frequent
alarms, while rapid adaptation can absorb a new regime's mismatch and reduce
the predictive contrast needed for restart. Table~\ref{tab:brpc_robustness}
and Appendix~\ref{app:discrepancy_dynamics} examine this interaction.

\paragraph{Continual learning and selective forgetting.}
Continual learning emphasizes retaining useful information across tasks and
distributions \citep{kirkpatrick2017overcoming,zenke2017continual,
nguyen2018variational,schwarz2018progress,lopez2017gradient,rolnick2019experience,
aljundi2019mir,buzzega2020dark,de2021continual}. Online calibration has a
related memory trade-off. Recent batches stabilize tracking under gradual
drift. After a regime change, a reset can remove parameter and discrepancy
information that has become biased for the current system. BRPC uses
predictive evidence to choose between retaining this information and
starting a new segment, with the current projected calibration target as
its estimation objective.

\section{Limitations and Scope}\label{sec:limit}

The discrepancy-tracking bound in Theorem~\ref{thm:brpc_tracking} is
conditional on the residual sequence and uniform contraction in the
pre-update precision metric. Particle coverage, the parameter evolution
prior, and design coverage affect those residuals. The local coupled result
in Appendix~\ref{app:cr_local_coupled} additionally assumes contraction of
the projected parameter update on a fixed-support segment. It characterizes
tracking under that local stability condition; global joint consistency
remains open. Under simulator misspecification, the projected target is a
best-fitting simulator parameter and can differ from a physical parameter.

The restart bounds require concentration and persistence conditions for the
adaptive predictive scores. Running mean and variance standardization sets
their scale, while their tails depend on the response and predictive models.
Covariate shift and finite-batch variability can resemble regime changes.
The higher-dimensional diagnostic illustrates the resulting trade-off:
restarting can improve response prediction while increasing projected-target
error. Under pure drift, resets can serve as useful memory control, although
event-level detection metrics classify them as false alarms.

Computation is dominated by parameter particles, discrepancy updates, and
the active restart experts. The studied BRPC implementations are substantially
slower than ensemble Kalman filtering. Fixed support limits discrepancy-state
growth; residual re-anchoring adds repeated segment refits. As $d_\theta$
increases, particle coverage becomes a separate bottleneck. More scalable
projected-posterior approximations are needed for large calibration spaces.

Our empirical evaluation primarily relies on synthetic
streams and simulated digital-twin environments. Although
these experiments demonstrate the effectiveness of BRPC under
different forms of nonstationarity, further validation on
real-world physical systems is needed to establish its
practical robustness and generalizability. Broader comparisons
with online learning and adaptive calibration methods would
also help clarify the advantages and limitations of BRPC
in realistic applications. Given the time, cost, and
operational constraints associated with physical deployment,
we leave such evaluations for future work.

Furthermore, our closed-loop decision study employs a
relatively simple upper confidence bound (UCB) strategy
to demonstrate the potential benefits of online calibration
for downstream decision making. Future research could
integrate BRPC with more sophisticated sequential
decision-making frameworks, such as dual control,
Bayesian optimization, and reinforcement learning, to
jointly optimize system performance and information
acquisition under uncertainty.
\end{crblock}

\newpage
\section*{NeurIPS Paper Checklist}

\begin{enumerate}

\item {\bf Claims}
    \item[] Question: Do the main claims made in the abstract and introduction accurately reflect the paper's contributions and scope?
    \item[] Answer: \answerYes{}
    \item[] Justification: The abstract and introduction state the main scope: online Bayesian projected calibration under nonstationarity, the BRPC recursive update, and restart-mechanism compatibility. The claims are supported by the method, theory overview, and experiments.

\item {\bf Limitations}
    \item[] Question: Does the paper discuss the limitations of the work performed by the authors?
    \item[] Answer: \answerYes{}
    \item[] Justification: The conclusion  and appendix~\ref{sec:limit} discuss limitations including broader high-dimensional settings, additional real digital-twin datasets, computational cost of residual re-anchoring, and the need for improved false-alarm behavior\crnew{, the conditional scope of the tracking theorem, hyperparameter sensitivity, and the runtime gap to ensemble filtering}.

\item {\bf Theory assumptions and proofs}
    \item[] Question: For each theoretical result, does the paper provide the full set of assumptions and a complete (and correct) proof?
    \item[] Answer: \answerYes{}
    \item[] Justification: The main text states the assumptions for recursive tracking and the idealized score-based restart-mechanism analysis\crnew{, and states that the tracking bound is conditional on the residual sequence}. Full proofs and auxiliary results are provided in Appendices~\ref{app:proof_brpc_dynamic_regret},
\ref{app:proof_gaussian_restart_odds}, and
\ref{app:proof_wcusum_bounds}. \crnew{Appendix~\ref{app:cr_local_coupled} states all assumptions of the local coupled-tracking theorem and gives its full proof.}

    \item {\bf Experimental result reproducibility}
    \item[] Question: Does the paper fully disclose all the information needed to reproduce the main experimental results of the paper to the extent that it affects the main claims and/or conclusions of the paper (regardless of whether the code and data are provided or not)?
    \item[] Answer: \answerYes{}
    \item[] Justification: The benchmark descriptions and experimental settings in Appendices~\ref{app:benchmark_details} and \ref{app:common_experimental_settings} report the synthetic and plant-simulation configurations, seeds, batch sizes, particle counts, restart metrics, and baseline definitions.

\item {\bf Open access to data and code}
    \item[] Question: Does the paper provide open access to the data and code, with sufficient instructions to faithfully reproduce the main experimental results, as described in supplemental material?
    \item[] Answer: \crnew{\answerYes{}}
    \item[] Justification: \crnew{Code and benchmark runners are available at
    \url{https://github.com/crushonyfg/BRPC}, with installation instructions,
    experiment entrypoints, configuration examples, and data-file descriptions.
    Benchmark construction and experimental settings are given in
    Appendices~\ref{app:benchmark_details} and~\ref{app:common_experimental_settings}.}

\item {\bf Experimental setting/details}
    \item[] Question: Does the paper specify all the training and test details (e.g., data splits, hyperparameters, how they were chosen, type of optimizer) necessary to understand the results?
    \item[] Answer: \answerYes{}
    \item[] Justification: The paper specifies the benchmark construction, baseline definitions, particle counts, batch sizes, noise levels, restart-mechanism settings, and event-metric definitions in the experiment section and Appendix~\ref{app:common_experimental_settings}.

\item {\bf Experiment statistical significance}
    \item[] Question: Does the paper report error bars suitably and correctly defined or other appropriate information about the statistical significance of the experiments?
    \item[] Answer: \answerYes{}
    \item[] Justification: \crnew{Main synthetic metrics are reported as mean $\pm$ standard deviation across 125 drift, 300 sudden-change, and 75 mixed-change runs, with 25 seeds per configuration. These error bars capture variation across both configurations and seeds. Plant-simulation results are reported over 10 seeds. Additional tables report mean $\pm$ standard deviation or repeated-run summaries where appropriate.}

\item {\bf Experiments compute resources}
    \item[] Question: For each experiment, does the paper provide sufficient information on the computer resources (type of compute workers, memory, time of execution) needed to reproduce the experiments?
    \item[] Answer: \answerYes{} 
    \item[] Justification: Appendix~\ref{app:common_experimental_settings}
    reports the compute environment. Experiments were run on a local workstation
    with an NVIDIA T1000 GPU with 4GB memory available; the online calibration
    experiments are lightweight and mainly involve particle updates,
    Gaussian-process discrepancy updates, restart-mechanism bookkeeping, and
    repeated runs across random seeds.
    
\item {\bf Code of ethics}
    \item[] Question: Does the research conducted in the paper conform, in every respect, with the NeurIPS Code of Ethics \url{https://neurips.cc/public/EthicsGuidelines}?
    \item[] Answer: \answerYes{}
    \item[] Justification: The work uses synthetic and simulated plant-production data and does not involve human subjects, private personal data, or safety-critical deployment. The research is consistent with the NeurIPS Code of Ethics.

\item {\bf Broader impacts}
    \item[] Question: Does the paper discuss both potential positive societal impacts and negative societal impacts of the work performed?
    \item[] Answer: \answerYes{} 
    \item[] Justification: The work may improve adaptive digital-twin calibration and decision support under nonstationarity. Potential risks include over-reliance on automatically calibrated models in operational settings; the paper frames BRPC as a statistical calibration method whose predictions and resets should be validated before deployment.
    
\item {\bf Safeguards}
    \item[] Question: Does the paper describe safeguards that have been put in place for responsible release of data or models that have a high risk for misuse (e.g., pre-trained language models, image generators, or scraped datasets)?
    \item[] Answer: \answerNA{} 
    \item[] Justification: The paper does not release high-risk models, scraped datasets, generative models, or assets with obvious misuse potential requiring special safeguards.

\item {\bf Licenses for existing assets}
    \item[] Question: Are the creators or original owners of assets (e.g., code, data, models), used in the paper, properly credited and are the license and terms of use explicitly mentioned and properly respected?
    \item[] Answer: \answerYes{} 
    \item[] Justification: The paper cites existing methodological sources. The plant-simulation benchmark is generated by the authors using a commercial discrete-event simulation environment; no third-party dataset is redistributed.

\item {\bf New assets}
    \item[] Question: Are new assets introduced in the paper well documented and is the documentation provided alongside the assets?
    \item[] Answer: \answerYes{} 
    \item[] Justification: The paper introduces generated synthetic streams and plant-simulation benchmark outputs used for the experiments. Their construction, parameters, and evaluation protocol are documented in Appendix~\ref{app:benchmark_details} and Appendix~\ref{app:common_experimental_settings}. \crnew{The paper documents the benchmark-generation procedures; code is available at \url{https://github.com/crushonyfg/BRPC}.}

\item {\bf Crowdsourcing and research with human subjects}
    \item[] Question: For crowdsourcing experiments and research with human subjects, does the paper include the full text of instructions given to participants and screenshots, if applicable, as well as details about compensation (if any)? 
    \item[] Answer: \answerNA{} 
    \item[] Justification: The paper does not involve crowdsourcing experiments or human-subject data.

\item {\bf Institutional review board (IRB) approvals or equivalent for research with human subjects}
    \item[] Question: Does the paper describe potential risks incurred by study participants, whether such risks were disclosed to the subjects, and whether Institutional Review Board (IRB) approvals (or an equivalent approval/review based on the requirements of your country or institution) were obtained?
    \item[] Answer: \answerNA{} 
    \item[] Justification: The paper does not involve human subjects, private personal data, or crowdsourcing, so IRB approval is not applicable.

\item {\bf Declaration of LLM usage}
    \item[] Question: Does the paper describe the usage of LLMs if it is an important, original, or non-standard component of the core methods in this research? Note that if the LLM is used only for writing, editing, or formatting purposes and does \emph{not} impact the core methodology, scientific rigor, or originality of the research, declaration is not required.
    \item[] Answer: \answerNA{} 
    \item[] Justification: LLMs are not used as an important, original, or non-standard component of the proposed method, experiments, or scientific contribution. Any ordinary writing or editing assistance does not affect the core methodology.

\end{enumerate}

\end{document}